\documentclass[letterpaper,journal]{IEEEtran}
\usepackage{array}
\usepackage{subfig}
\usepackage{textcomp}
\usepackage{stfloats}
\usepackage{url}
\usepackage{verbatim}
\usepackage{graphicx}
\def\BibTeX{{\rm B\kern-.05em{\sc i\kern-.025em b}\kern-.08em
    T\kern-.1667em\lower.7ex\hbox{E}\kern-.125emX}}
\usepackage{balance}
\usepackage{amsmath}
\usepackage{amsfonts,amssymb,amsthm}
\usepackage{hyperref}
\usepackage{orcidlink}
\usepackage{graphicx}
\usepackage{cleveref}
\usepackage{authblk}
\usepackage{algorithm}
\usepackage{algpseudocode}
\usepackage{bm}
\usepackage{dsfont}
\usepackage[export]{adjustbox} 
\usepackage{cite}

\newlength\myindent
\newcommand{\Indstate}[1][1]{\State\hspace{-1.35em}\hspace{#1\myindent}}

\newtheorem{theorem}{Theorem}[section]

\newtheorem{definition}[theorem]{Definition}
\newtheorem{proposition}[theorem]{Proposition}
\newtheorem{lemma}[theorem]{Lemma}

\newcommand{\R}{\mathbb{R}}
\newcommand{\N}{\mathbb{N}}
\newcommand{\E}{\mathbb{E}}
\newcommand{\one}{\mathds{1}}
\renewcommand{\phi}{\varphi}

\newcommand{\x}{\widetilde{x}}

\newcommand{\eqd}{\overset{\text{d}}{=}}

\begin{document}
\title{Redistributor: Transforming Empirical~Data~Distributions}

\author{Pavol~Harar\,\orcidlink{0000-0001-5206-1794}, 
    Dennis~Elbr{\"a}chter,
    Monika D{\"o}rfler\,\orcidlink{0000-0001-6139-630X},
    Kory D. Johnson\,\orcidlink{0000-0002-7322-2451}

\thanks{
Pavol Harar is with University of Vienna, Austria and with Brno University of Technology, Czech Republic (email:~pavol.harar@univie.ac.at, web:~\href{https://pavol.harar.eu}{pavol.harar.eu}). Dennis~Elbr{\"a}chter is with ETH Z{\"u}rich, Switzerland. Monika D{\"o}rfler is with University of Vienna, Austria. Kory D. Johnson is currently independent (web:~\href{https://korydjohnson.com/}{korydjohnson.com}).
}}


\maketitle

\begin{abstract}
We present an~algorithm and package, Redistributor, which forces a~collection of scalar samples to follow a~desired distribution. When given independent and identically distributed samples of some random variable $S$ and the~continuous cumulative distribution function of some desired target $T$, it provably produces a~consistent estimator of the~transformation~$R$ which satisfies $R(S)=T$ in distribution. As the~distribution of $S$ or $T$ may be unknown, we also include algorithms for efficiently estimating these distributions from samples. This allows for various interesting use cases in image processing, where Redistributor serves as a~remarkably simple and easy-to-use tool that is capable of producing visually appealing results. For color correction it outperforms other model-based methods and excels in achieving photorealistic style transfer, surpassing deep learning methods in content preservation. The~package is implemented in Python and is optimized to efficiently handle large datasets, making it also suitable as a~preprocessing step in machine learning. The~source code is available at \href{https://github.com/paloha/redistributor}{https://github.com/paloha/redistributor}.
\end{abstract}

\noindent
\textbf{Keywords} Data preprocessing $\cdot$ Color correction $\cdot$ Density estimation $\cdot$ Histogram matching $\cdot$ Quantile transform

\section{Introduction}
\noindent
Redistributor provably transforms an~empirical distribution to a~known distribution or to match a~second empirical distribution, see~\Cref{fig:transformation}. Historically, similar methods have been used to transform data to be Gaussian. 
The idea of transforming data to a~known distribution has a~long history and is intuitively appealing. The~most famous method for achieving a~transformation to normality, the~Box-Cox transformation\,\cite{BoxCox64}, has been cited thousands of times in a~broad range of fields, cf.\,\cite{PeWiSh98,Os10}. The transformation is motivated by the fact that  normality is a~core assumption of many statistical tests.
The~central idea is a~simple translation of raw data to ranks or order statistics and then transforming these into corresponding normal scores, e.g., quantiles from the~standard normal distribution. If data is initially provided as ranks, this idea dates back to 1947\,\cite{Bartlett47}. Redistributor generalizes the idea to transformation into any distribution for which a~percent-point function is available or can be empirically estimated. 

\begin{figure}[!t]
    \captionsetup[subfigure]{labelformat=empty}
    \vspace{-3mm}  
    \centering
    \subfloat{{\label{fig:teaser-a}
        \includegraphics[height=0.08\paperheight]{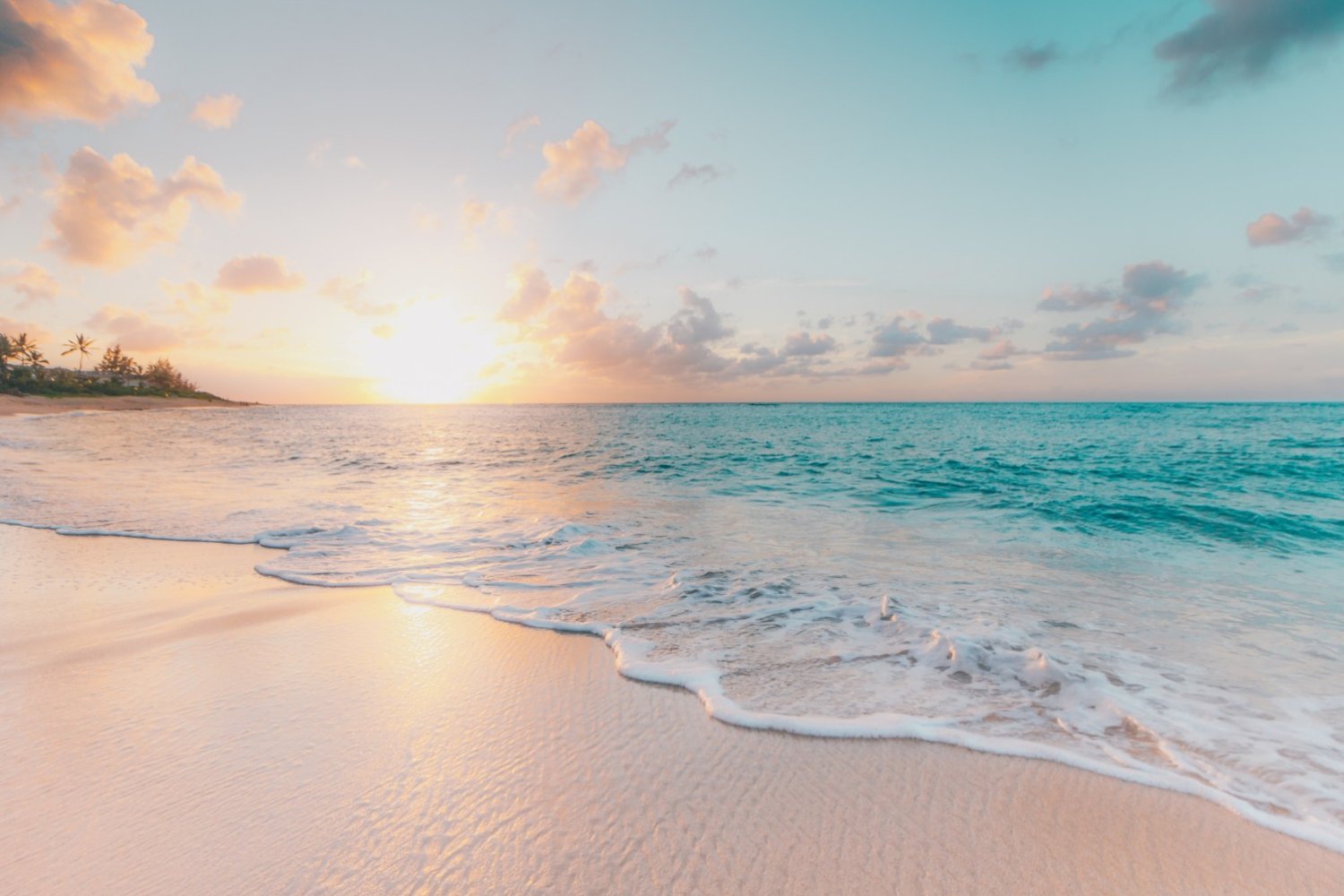} }}%
    \subfloat{{\label{fig:teaser-b}
        \includegraphics[height=0.08\paperheight]{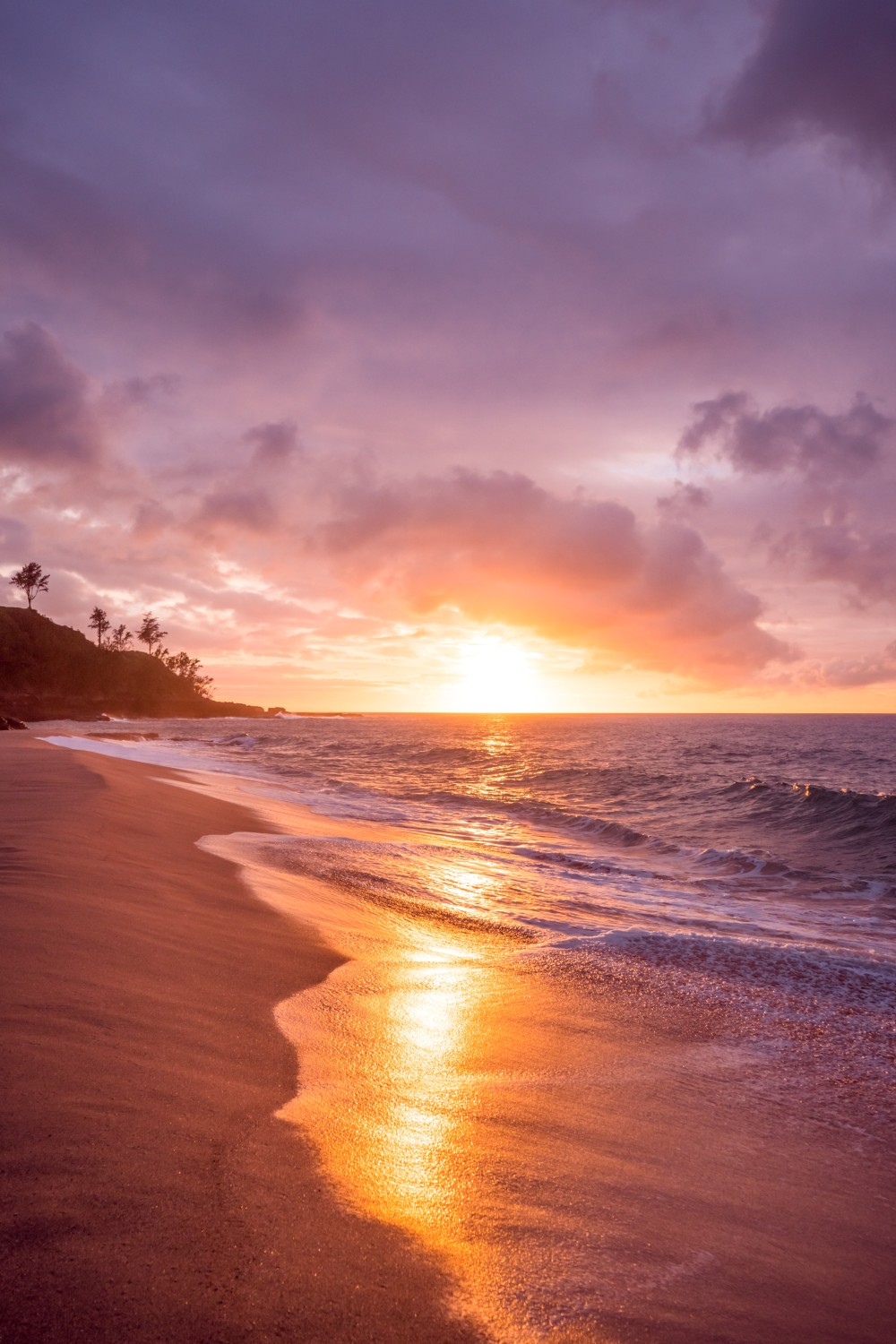} }}%
    \subfloat{{\label{fig:teaser-c}
        \includegraphics[height=0.08\paperheight]{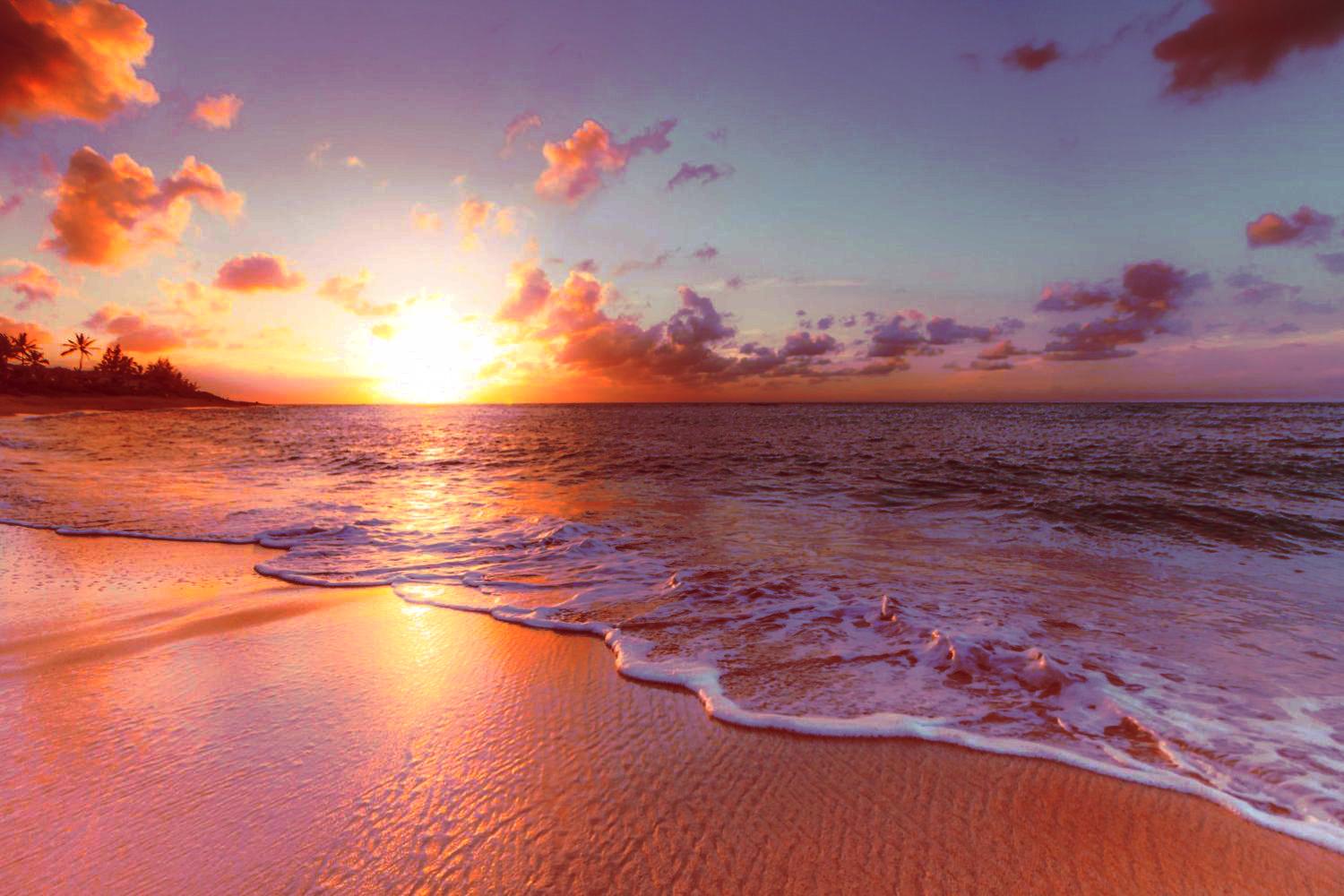} }}\vspace{-1mm}%
        
    \subfloat[\centering Source]{{\label{fig:teaser-d}
        \includegraphics[height=0.08\paperheight]{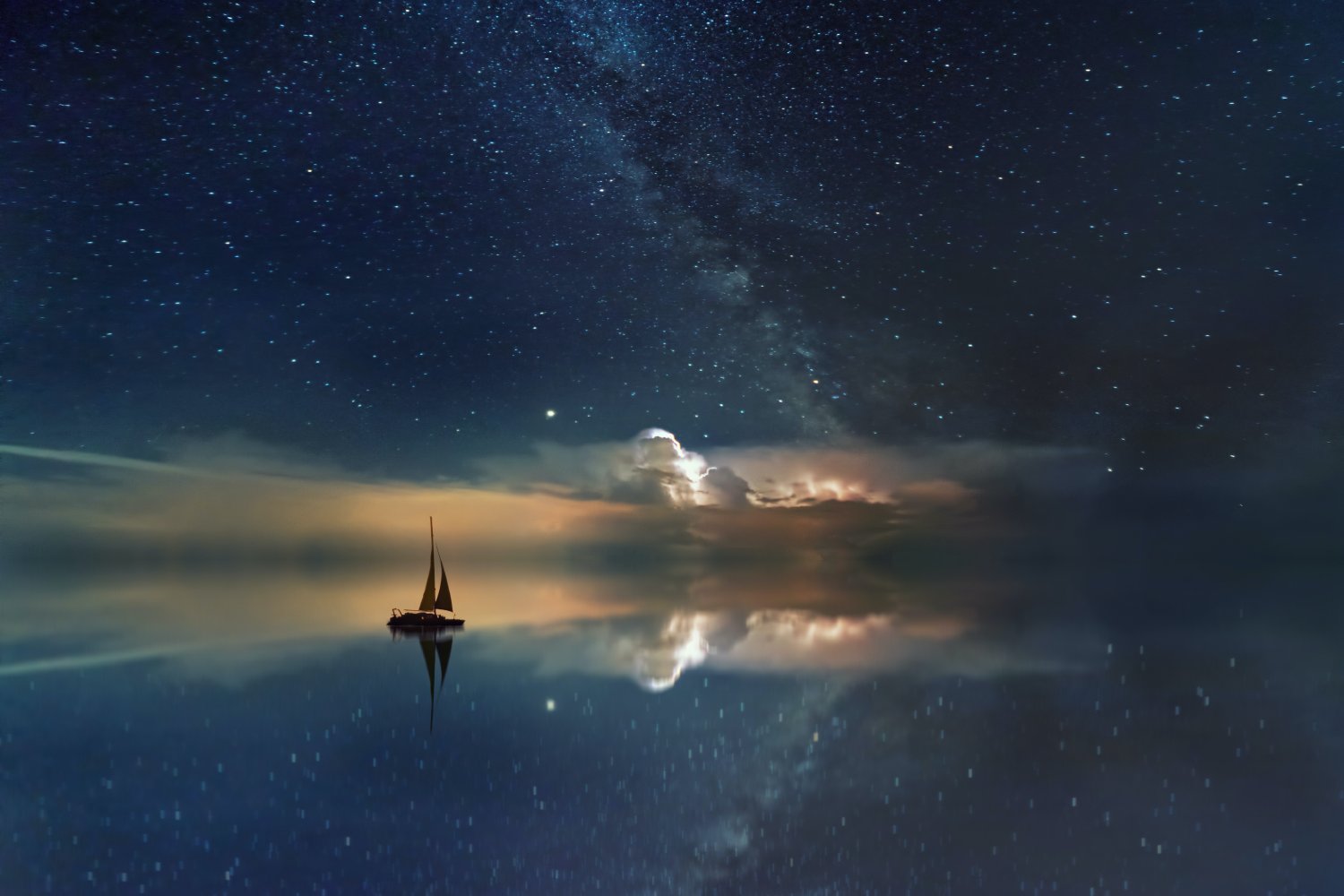} }}%
    \subfloat[\centering Target]{{\label{fig:teaser-e}
        \includegraphics[height=0.08\paperheight]{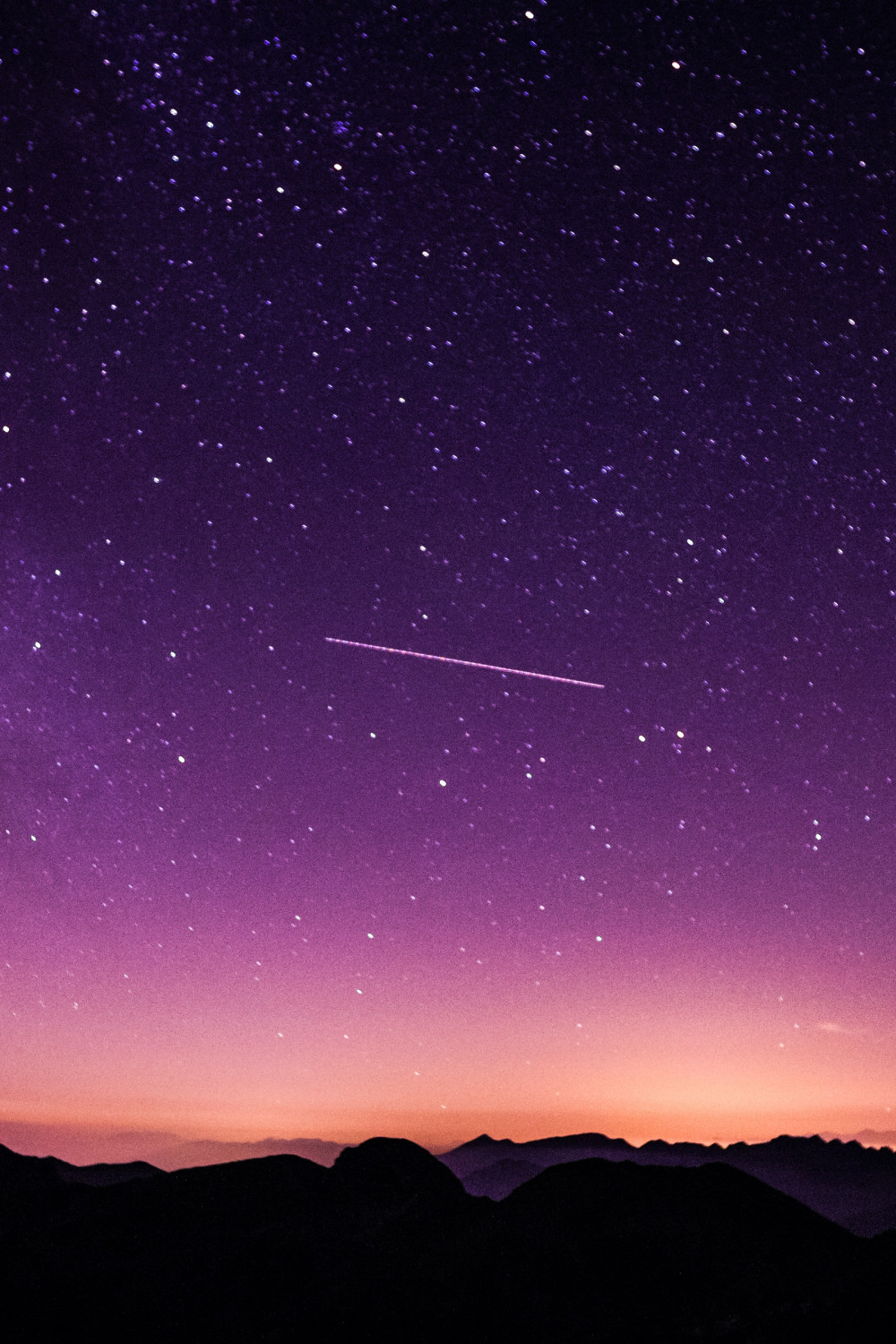} }}%
    \subfloat[\centering Output (our method)]{{\label{fig:teaser-f}
        \includegraphics[height=0.08\paperheight]{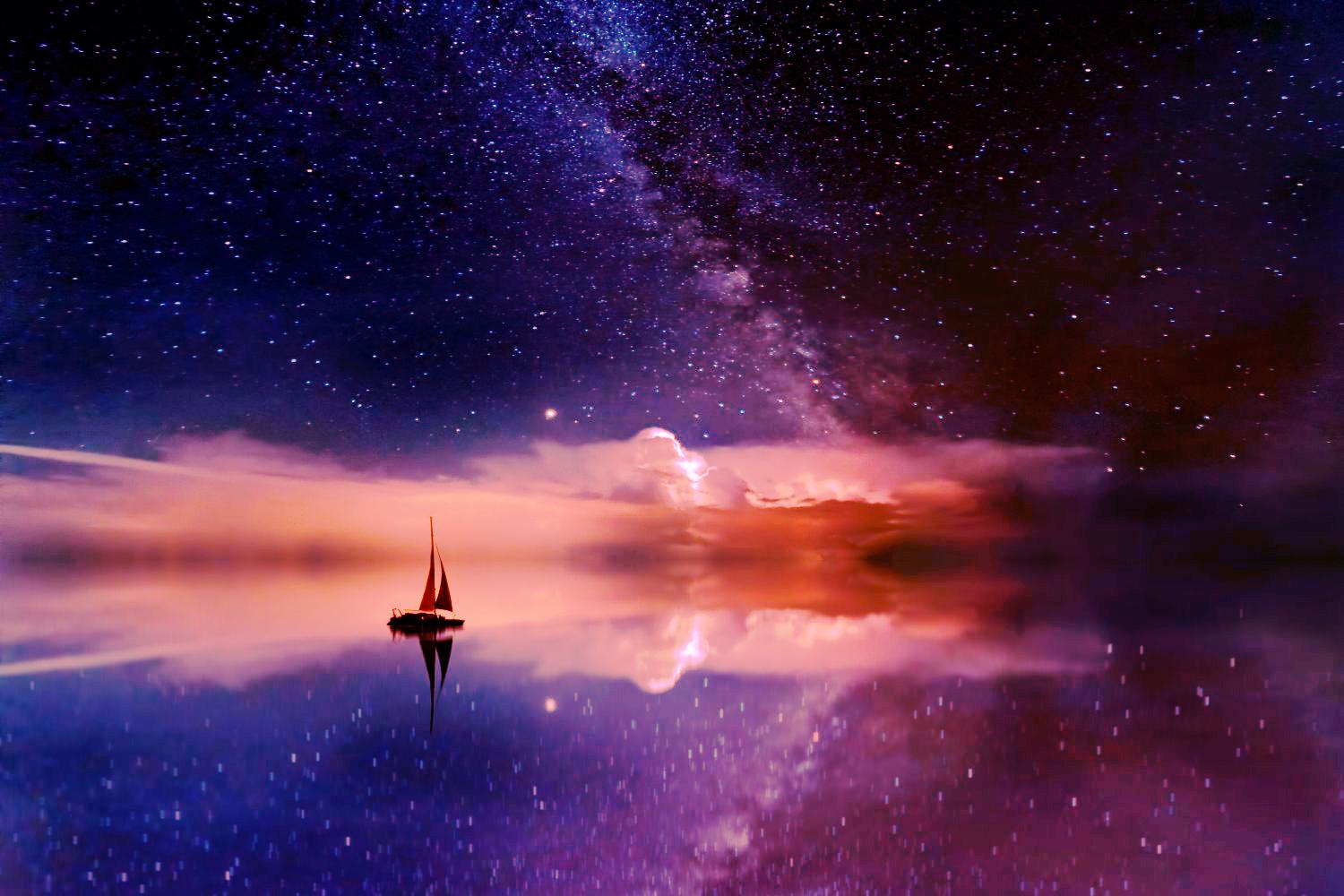} }}%
    \caption{Matching colors of a~reference image -- one of the~use cases of Redistributor from \Cref{sec:use_cases}.}%
    \label{fig:teaser}%
\end{figure}

In addition to the~package and its algorithmic details, we present theoretical results supporting the~use of the~method as well as a~broad array of examples in which transforming distributions is beneficial.
In the~context of image processing, we can correct color issues in photography, producing far better results than other automated procedures. We also consider further uses of matching the~color distribution of an~image to some reference image, e.g.\@ translating it into a~preferred color scheme for aesthetic purposes, creating photographic mosaics, and performing image data augmentation in color space. Lastly, we discuss the~use of Redistributor in signal processing and preprocessing within a~machine learning pipeline. 

\section{Related Work}
\noindent
Redistributor is implemented as a~Scikit-learn transformer that is convenient for processing large datasets and integration in machine learning pipelines. 
While the~Scikit-learn\,\cite{scikit-learn} resp.\@ Scikit-image\,\cite{scikit-image} functions \verb!quantile_transform! and \verb!match_histograms! have the~same mathematical basis, they are limited to specific use cases and do not produce a~reusable, estimated data distribution as our \verb!LearnedDistribution! and \verb!KernelDensity! classes do. They also do not feature a~robust treatment of boundary values, which may cause problems in real-world applications.

A~related R\,\cite{Rprog} package, bestNormalize\,\cite{Peterson21}, provides a~set of methods for normalizing a~distribution and functionality to pick the~best transformation from this set. We deviate from testing the~parameterization of such transformations or picking the~best one and focus instead on the~broader applicability and scalability of the~procedure. The~generality and~additional flexibility of Redistributor to estimate and map between two data distributions opens the~door for a~host of new use cases. 

One use case of Redistributor in the context of image processing is color mapping, i.e.\@ transforming the color scheme of a~given image to that of some reference image, usually for aesthetic purposes. For details on this particular application, we refer to the survey article\,\cite{ReinhardSurvey} and  references within. 
While color mapping constitutes an important use case which allows for visually appealing demonstrations, it is important to note that this is not the sole focus of this package. 
For instance, Redistributor can just as easily be applied in Fourier space to adjust the sharpness of an image to that of a~given reference image, as shown in \Cref{sec:alt_spaces}.

\begin{figure*}[!t]%
    \captionsetup[subfloat]{singlelinecheck=false}
    \centering
    \vspace{-3mm}
    \subfloat[\centering Source data sampled from $F_S$]{{
        \includegraphics[height=0.168\paperheight]{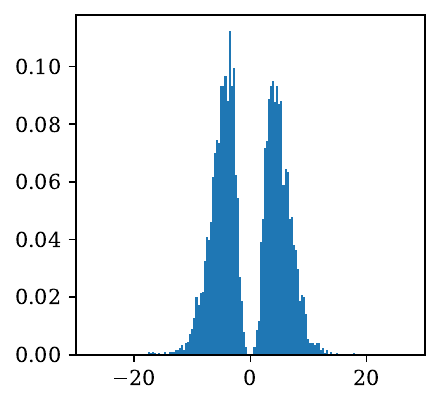} }}%
    \subfloat[\centering Transformation~$R$]{{
        \includegraphics[height=0.168\paperheight]{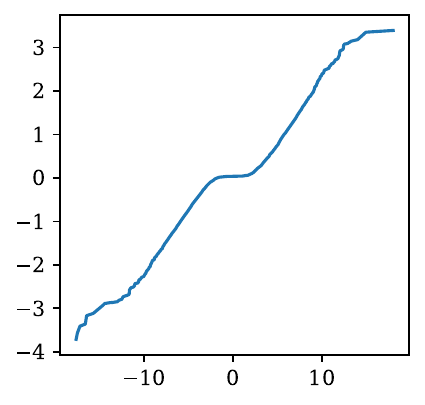} }}%
    \subfloat[\centering Source data transformed to $\hat{F}_T$]{{
        \includegraphics[height=0.168\paperheight]{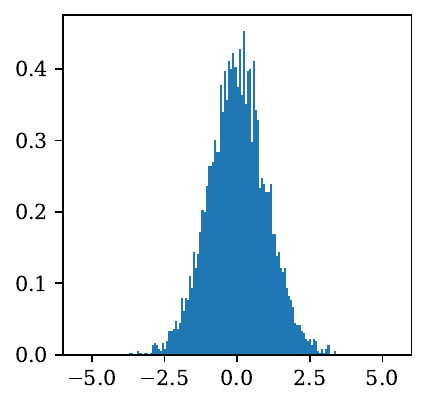} }}%
    \caption{Applying the~transformation~$R$ from $\hat{F}_S$ to $\hat{F}_T$, where $\hat{F}_S$ is an~estimate of a~Double Gamma distribution $F_S$ obtained by \Cref{alg:L} from 1000 iid samples, and $\hat{F}_T$ is an~estimate of a~Gaussian distribution $F_T$ obtained by \Cref{alg:L} from 1000 iid samples. Subfigures (a) and (c) display density histograms.
    } %
    \label{fig:transformation}%
\end{figure*}

As color mapping is also used to demonstrate the virtues of neural style transfer (NST) methods (see, e.g.,\,\cite{DPST}), it seems relevant to comment on the differences between NST and Redistributor when used for this particular task. By design, our method guarantees fidelity of content when changing style, whereas NST methods make a~trade-off between fidelity and style. In contrast to neural-network-based methods, our method is entirely explainable and therefore its use can easily be adapted based on task-specific knowledge, e.g.\@ by choosing the color space in which the image is represented. Moreover, Redistributor only needs the reference image, whereas NST requires a~corpus of images to be trained on, unless a~pretrained model is already available. Furthermore, when it comes to computational resources, Redistributor operates efficiently on a~standard CPU and can swiftly adjust images in a~matter of seconds. In contrast, NST typically requires the computational power of a~GPU, and the inference process can be significantly more resource-intensive than applying Redistributor. It is worth noting the remarkably lightweight nature of the Redistributor package, which weighs in at just a~few kilobytes. This stark contrast in package size becomes evident when compared to the bulkier requirements of running NST, for instance, PyTorch or TensorFlow, which often exceed 1 gigabyte in size. Of course, it should be noted that NST can also be used for style transfer when the underlying mathematical model is not known, e.g., when changing an image to match the painting style of a~specific artist instead of just adjusting colors. 

In~\Cref{subsec:comparison-correction}, we evaluate Redistributor on a~dataset for color correction to demonstrate that it compares favorably to existing model-based methods~\cite{Pitie,Hacohen,Faridul,Pouli,Reinhard_for_boxplot_comp}. In~\Cref{subsec:comparison-nst}, we compare Redistributor to an NST method for photorealistic style transfer, showcasing that our method is significantly better at keeping the structure of the images, while still providing a~convincing transfer of style.

\section{Our Contribution}
\noindent 
Redistributor constitutes a generalization of histogram matching, which has seen extensive use in color mapping, to a much broader field of applications. Moreover, we are providing an efficient and flexible implementation such that it can be conveniently used and adapted, based on the specifics needs of any given use case.
We establish a consistency estimate to provide a mathematically rigorous basis for the use of our method in the context of data preprocessing for machine learning tasks. 
The capabilities of Redistributor are showcased on a diverse, but not exhaustive, range of use cases. 
Lastly, we compare Redistributor to both data-based as well as model-based methods. While the primary aim of this work is to provide the community with a versatile and theoretically well-grounded tool for further exploration, we still demonstrate its ability to outperform various existing methods. In particular we show that it compares favorably to a deep learning approach, despite being fully explainable and computationally parsimonious. 

\section{Description of the~Method}\label{Desc}
\noindent
The idea behind Redistributor (see \Cref{alg:Redistributor}) is the~well-know observation that, given a~random variable $X$ with cumulative distribution function\footnote{The CDF of a~real-valued random variable $X$ is defined as
$F_X(x) = P(X \leq x), \textrm{ for all }x \in \mathbb{R}.$} (CDF) $F_X$, one has
\begin{align*}
    F_X(X)\sim \mathcal{U}[0,1]
\end{align*}
and, if $F_X$ is invertible,
\begin{align*}
    X\sim F_X^{-1}(\mathcal{U}[0,1]),
\end{align*}
where $\mathcal{U}[0,1]$ is the~uniform distribution on the~interval 
$[0,1]$. The~application of the~inverse cumulative distribution function, or percent-point function (PPF), to samples taken uniformly from $[0,1]$ in order to obtain samples from a~random variable $X$ is known as inverse transform sampling\,\cite{devroye2006nonuniform}.  

Another direct consequence of the~above is that, given a~source distribution $F_S$ and a~target distribution $F_T$ of random variables
 $S$ and $T$, respectively, the~transformation 
\begin{align}\label{eq:R_def}
    R:= F_T^{-1}\circ F_S
\end{align}
achieves 
\begin{align*}
    R(S)\eqd T,
\end{align*}
where $\eqd$ denotes equality in distribution.

While this is straightforward in principle, in practice one may not have access to the~CDFs or PPFs of $S$ and $T$, and needs to estimate them from samples. When the~number of samples in small, this can be accomplished nicely via kernel density estimation (see \Cref{alg:KDEW}). In order to produce a~continuous and invertible transform in the~case of large sample sizes, as may, for example, be encountered when using this algorithm as a~preprocessing step for machine learning, we use a~continuous piecewise linear empirical estimator (see \Cref{alg:L}) of a~cumulative distribution function $F$. Specifically, for samples $X_1,\dots,X_n$ it maps $X_k$ to $\tfrac{k}{n+1}$ and is linear in between, whereas the~behaviour outside of the~range of the~observed samples should be chosen depending on the~use case\footnote{Roughly speaking, extrapolation to inputs outside the~range of the~observed samples requires some assumption of the~probability distribution from which they are sampled. For more details on this see~\Cref{subsec:boundaries}}. The~implementation of these methods is described in more detail in the~next Section. 

In~\Cref{sec:math}, we prove consistency of our empirical estimator, i.e.\@ quantify how the~transformation obtained by replacing $F_S$ in \eqref{eq:R_def} with our empirical estimator, based on $n$ independent and identically distributed (iid) samples, approximates the~transformation using the~true CDF as $n$ increases. As demonstrated in~\Cref{sec:use_cases}, this transformation turns out to be quite useful when applied to images by treating them as a~collection of pixels, even though it is certainly not reasonable to assume that the~pixels of an~image are independent. We will discuss the~theoretic underpinning of this in \Cref{sec:multidim}.

\section{Algorithms}
\label{sec:algos}
\noindent
In order to use our Redistributor Python package to compute the~transformation~$R$, we have to instantiate the~Redistributor class\footnote{In \Cref{sec:algos}, we use terminology from object-oriented programming} by specifying the~source and target distributions. Each of the~distributions must be described by an~instance of a~class that implements at least methods for computing CDF and PPF. Assuming the~CDF (and its inverse, the~PPF) is continuous, the~transformation~$R$ and its inverse is then given by composition of those functions as described in \Cref{alg:Redistributor}.

\subsection{Estimating Distribution from Data}
\label{subsec:estimating}
\noindent
It may happen that we want to compute the~transformation~$R$ between distributions of which one or both are not known explicitly.
In that case, we could estimate the~missing distribution(s) from data, by computing the~empirical cumulative distribution function (eCDF). 
As the eCDF is not a~continuous function by definition, its inverse is not readily available. Therefore, in the \verb!LearnedDistribution! class, we use linear interpolation to obtain a~continuous and invertible estimate of the CDF. This class also provides the methods for PPF, probability density function (PDF), and a~function for random variable sampling (RVS). The~implementation is described in \Cref{alg:L}. Note that for large numbers of samples it can become quite inefficient to use all of the~samples, so a~value for $bins$ can be provided and, if the~number of samples provided is larger than that, they will be subsampled to a~set of size $bins$ before generating the~CDF and PPF. By default, the~value of $bins$ is set to the~min.\@ of $5000$ and the~number of samples. 

There exist other ways of estimating distributions which are arguably more suitable if only a~small amount of samples is available. One such technique is kernel density estimation (KDE)\,\cite{chen2017tutorial}. A~popular implementation of KDE exists as a~part of the~Scikit-learn Python package\,\cite{scikit-learn}, however, it does not implement CDF and PPF functions. For convenience, we have introduced a~\verb!KernelDensity! class that extends the~Scikit-learn KDE (with Gaussian kernel) by implementing the~methods for CDF, PPF, PDF, and RVS. This class described in \Cref{alg:KDEW} makes estimating the~source and/or target distributions of Redistributor effortless. 

\subsection{Handling Duplicate Values}
\label{subsec:duplicates}
\noindent
Redistributor assumes that both the~source and target distributions are continuous. This implies that a.s. (almost surely) no duplicate values will be present in a~random sample from these distributions. Observed data, however, may contain duplicate values that cannot be present in the~\Cref{alg:L} input array as they would render the~CDF non-invertible. This limitation would disqualify usage of \verb!LearnedDistribution! and by extension \verb!Redistributor!, on a~particularly interesting type of data -- images. For this reason, we provide a~function \verb!make_unique! that adds uniform noise of desired magnitude to duplicate values in an~array. This ensures that the~array has no repeating values a.s. so the~users can ``pretend'' the~data points come from a~continuous distribution even though the~array is for example quantized. The~function guarantees that the~minimum and maximum values are never changed, which is of technical importance in treating the~boundaries described in \Cref{subsec:boundaries}.

\begin{algorithm*}
\caption{Redistributor}\label{alg:Redistributor}
    \begin{algorithmic}
    \Indstate[0] \textbf{Instantiating}
        \Indstate[1] \textbf{Input:} source, target \Comment{Objects with \textit{ppf} and \textit{cdf} methods implemented}
        \Indstate[1] \textbf{Output:} $r$ (Redistributor instance) \Comment{Subclass of sklearn.base.TransformerMixin}
        \Indstate[1]
    
    \Indstate[0] $\bm{r}$\textbf{.transform}
        \Indstate[1] \textbf{Input:} $x$ (Float array) \Comment{Elements from source distribution}
        \Indstate[1] \textbf{Output:} $x'$ (Float array) \Comment{Elements of $x$ transformed into target distribution}
        \Indstate[1] {\footnotesize ~~1:} $x' \gets target.ppf(source.cdf(x))$
        \Indstate[1]

    \Indstate[0] $\bm{r}$\textbf{.inverse\_transform}
        \Indstate[1] \textbf{Input:} $x'$ (Float array) \Comment{Elements from target distribution}
        \Indstate[1] \textbf{Output:} $x$ (Float array) \Comment{Elements of $x'$ transformed into source distribution}
        \Indstate[1] {\footnotesize ~~1:} $x \gets source.ppf(target.cdf(x'))$
        \Indstate[1]
    \end{algorithmic}
\end{algorithm*}

\begin{algorithm*}
\caption{Learned Distribution}\label{alg:L}
    \begin{algorithmic}
        \vspace{4mm}
        \Indstate[0] \textbf{Instantiating the~class}
        
        \vspace{2mm}
        \Indstate[1] \textbf{Input:} 
            \Indstate[2] $x$ (Float array) \Comment{Samples from unknown distribution}
            \Indstate[2] $a$ (Float, optional) \Comment{Assumed left boundary of the~data distribution's support}
            \Indstate[2] $b$  (Float, optional) \Comment{Assumed right boundary of the~data distribution's support}
            \Indstate[2] $bins$  (Int, optional) \Comment{Scalar influencing the~precision of CDF approximation}
            \Indstate[2] ...  \Comment{For a~full signature consult the~source code}
        
        \vspace{2mm}
        \Indstate[1] \textbf{Outline of the~code:}
        \Indstate[2] {\footnotesize 1:} $lp \gets$ Choose the~lattice points based on $bins$ \Comment{Support of $l.ppf$}
        \Indstate[2] {\footnotesize 2:} $lv \gets$ Get the~lattice values from $x$ using a~partial-sort \Comment{Support of $l.cdf$}
        \Indstate[2] {\footnotesize 3:} Make $lv$ unique if desired \& necessary \Comment{So inversion is possible}
        \Indstate[2] {\footnotesize 5:} $l.cdf \gets$ Compute a~linear interpolant of $lp$ on $lv$
        \Indstate[2] {\footnotesize 4:} $l.ppf \gets$ Compute a~linear interpolant of $lv$ on $lp$
        \Indstate[2] {\footnotesize 6:} $l.pdf \gets$ Compute a~linear interpolant of the~derivative of $lv$ on $lp$
        \Indstate[2] {\footnotesize 7:} $l.rvs \gets$ $l.ppf$(random uniform sample from the~$ppf$ support)
        
        \vspace{2mm}
        \Indstate[1] \textbf{Output:} 
        \Indstate[2] $l$ (LearnedDistribution instance) \Comment{Object suitable as \Cref{alg:Redistributor} input}

        \vspace{2mm}
        \hspace{-11mm}\dotfill
        \vspace{2mm}
        \Indstate[0] \textbf{Available methods}
        \vspace{2mm}
    
    \Indstate[1] $\bm{l}$\textbf{.cdf} \Comment{Continuous piecewise linear approximation of CDF}
        \Indstate[2] \textbf{Input:} $q$ (Float array) \Comment{Quantile}
        \Indstate[2] \textbf{Output:} $p$ (Float array) \Comment{Lower tail probability}
        \Indstate[2] 
        
    \Indstate[1] $\bm{l}$\textbf{.ppf} \Comment{Inverse of $l.cdf$}
        \Indstate[2] \textbf{Input:} $p$ (Float array) \Comment{Lower tail probability}
        \Indstate[2] \textbf{Output:} $q$ (Float array) \Comment{Quantile}
        \Indstate[2] 
        
    \Indstate[1] $\bm{l}$\textbf{.pdf} \Comment{Numerical derivative of $l.cdf$}
        \Indstate[2] \textbf{Input:} $q$ (Float array) \Comment{Quantile}
        \Indstate[2] \textbf{Output:} $d$ (Float array) \Comment{Probability density}
        \Indstate[2] 
        
    \Indstate[1] $\bm{l}$\textbf{.rvs} \Comment{Random value generator}
        \Indstate[2] \textbf{Input:} $size$ (Int) \Comment{Desired number of elements to generate}
        \Indstate[2] \textbf{Output:} $s$ (Float array) \Comment{Random sample from the~estimated distribution}
        \Indstate[2] 
    \end{algorithmic}
\end{algorithm*}

\begin{algorithm*}
\caption{Kernel Density}\label{alg:KDEW}
    \begin{algorithmic}
    \vspace{4mm}
    \Indstate[0] \textbf{Instantiating the~class}
    
    \vspace{2mm}
    \Indstate[1] \textbf{Input:} 
        \Indstate[2] $x$ (Float array) \Comment{Elements from unknown distribution}
        \Indstate[2] $bandwidth$ (strictly positive Float, optional) \Comment{Standard deviation of the~Gaussian kernel}
        \Indstate[2] $cdf\_method$ (Str, 'precise' or 'fast', optional) \Comment{Default behaviour of $k.cdf$ function} 
        \Indstate[2] $grid\_density$  (Int, optional) \Comment{Scalar influencing the~precision of $k.cdf$/fast/ and $k.ppf$}
        \Indstate[2] ...  \Comment{For a~full signature consult the~source code}
        
        \vspace{2mm}
        \Indstate[1] \textbf{Outline of the~code:}
        \Indstate[2] {\footnotesize 1:} $kde \gets$ Define a~Gaussian mixture model using elements of $x$ and $bandwidth$
        \Indstate[2] {\footnotesize 2:} $k.cdf$/precise/ $\gets$ An~explicit evaluation of the~$cdf$ function under given $kde$
        \Indstate[2] {\footnotesize 3:} $a, b \gets$ Compute the~empirical support of $cdf$ under given $kde$
        \Indstate[2] {\footnotesize 4:} $lp \gets$ Choose the~lattice points based on $a, b$, and $grid\_density$
        \Indstate[2] {\footnotesize 5:} $lv \gets$ Get the~lattice values by evaluating $k.cdf$/precise/
        \Indstate[2] {\footnotesize 6:} $k.ppf$ $\gets$ Compute a~linear interpolant of $lv$ on $lp$
        \Indstate[2] {\footnotesize 7:} $k.cdf$/fast/ $\gets$ Compute a~linear interpolant of $lp$ on $lv$
        \Indstate[2] {\footnotesize 8:} $k.pdf$ $\gets$ An~exponential of the~log-likelihood of each sample under given $kde$

        \vspace{2mm}
        \Indstate[1] \textbf{Output:} 
        \Indstate[2] $k$ (KernelDensity instance) \Comment{Object suitable as \Cref{alg:Redistributor} input}

        \vspace{2mm}
        \hspace{-11mm}\dotfill
        \vspace{2mm}
        \Indstate[0] \textbf{Available methods}
        \vspace{2mm}
    \Indstate[1] $\bm{k}$\textbf{.cdf} \Comment{cumulative distribution function under given $kde$}
        \Indstate[2] \textbf{Input:} $q$ (Float array) \Comment{Quantile}
        \Indstate[2] \textbf{Output:} (Float array)
        \Indstate[3] \textbf{if} $cdf\_method =$ 'precise' \textbf{then} $p'$ \Comment{Explicit Gaussian mixture $cdf$}
        \Indstate[3] \textbf{if} $cdf\_method =$ 'fast' \textbf{then} $p''$  \Comment{Piecewise linear approximation of $k.cdf$/precise/}
        \Indstate[2]
        
    \Indstate[1] $\bm{k}$\textbf{.ppf} \Comment{Inverse of $k.cdf$/fast/}
        \Indstate[2] \textbf{Input:} $p$ (Float array) \Comment{Lower tail probability}
        \Indstate[2] \textbf{Output:} $q$ (Float array) \Comment{Quantile}
        \Indstate[2] 
        
    \Indstate[1] $\bm{k}$\textbf{.pdf} \Comment{Probability density function under the~given $kde$}
        \Indstate[2] \textbf{Input:} $q$ (Float array) \Comment{Quantile}
        \Indstate[2] \textbf{Output:} $d$ (Float array) \Comment{Probability density}
        \Indstate[2] 
        
    \Indstate[1] $\bm{k}$\textbf{.rvs} \Comment{Random value generator}
        \Indstate[2] \textbf{Input:} $size$ (Int) \Comment{Desired number of elements to generate}
        \Indstate[2] \textbf{Output:} $s$ (Float array) \Comment{Random sample from the~estimated distribution}
        \Indstate[2] 
    \end{algorithmic}
\end{algorithm*}

\begin{figure*}[!t]%
    \centering
    
    \subfloat[\centering Cumulative distribution function of \Cref{alg:L}]{{
        \hspace{-3mm}\includegraphics[width=0.5\textwidth]{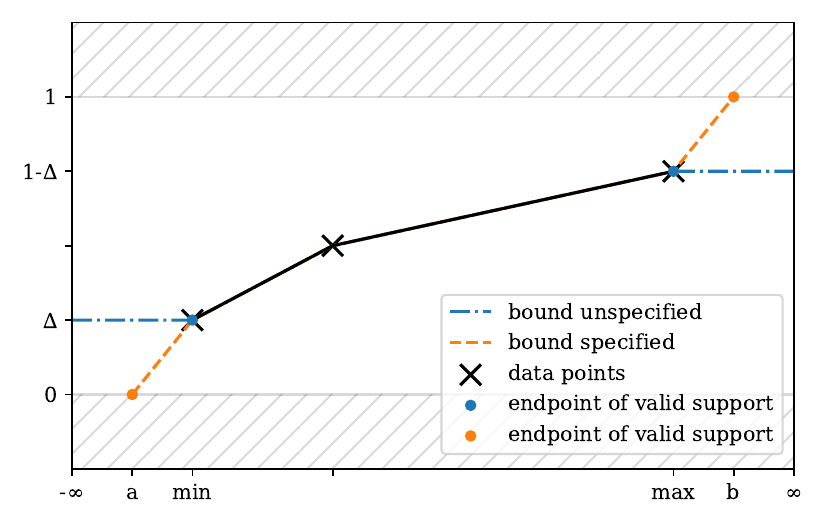} }}%
    \subfloat[\centering Percent-point function of \Cref{alg:L}]{{
        \includegraphics[width=0.5\textwidth]{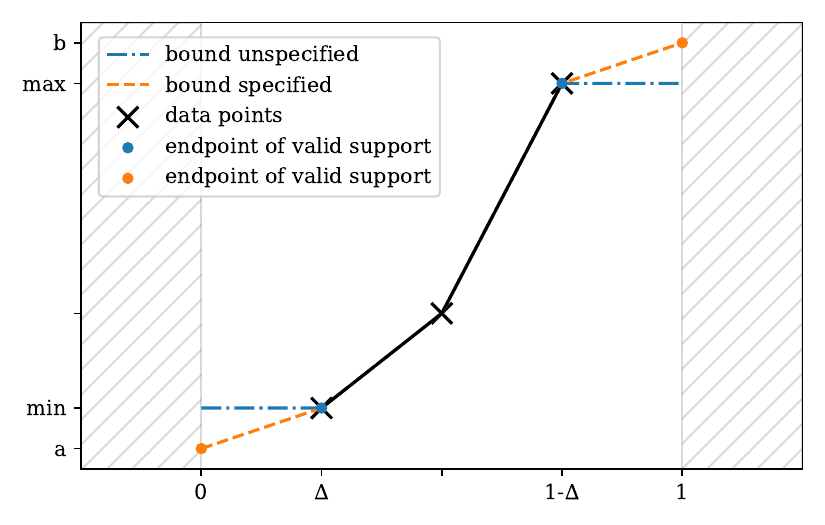} }}%
    \caption{Treating boundary values in \Cref{alg:L} -- The~simplest possible example using only 3~data points. Each subfigure shows where the~supported values of respective functions map based on whether the~boundaries are explicitly set or not. Endpoints denote the~``valid'' support, i.e.\@ the~interval where the~function is strictly increasing. E.g., if the~boundary $a$ is not set, all CDF values from the~interval [$-\infty$, $min$] map to the~constant $\Delta = 1 / (bins + 1)$. The~PPF maps values from the~interval [0, $\Delta$] to $min$, i.e.\@ the
minimum value of the~provided data. Analogously, the~same applies for the~boundary value $b$, which can be specified or not independently of $a$.}%
    \label{fig:boundaries}%
\end{figure*}

\subsection{Treating Boundary Values}
\label{subsec:boundaries}
\noindent
Estimating a~CDF outside of the~range of the~available samples is a~well-known problem as it amounts to extrapolating beyond the~observed data. In particular it is already quite difficult to estimate the~support of the~distribution from which the~samples are taken. We deal with this in the~following way.

Let $min$ and $max$ denote the~smallest resp.\@ largest value of the~$n$ samples that were used to generate the~Learned Distribution with \Cref{alg:L}, and write $\Delta = 1 / (bins + 1)$. Then, the~CDF always maps inputs in $[min,max]$ to values in $[\Delta,1-\Delta]$ and on these intervals the~CDF and PPF are inverses of each other. A~choice becomes necessary for inputs outside of these intervals. In case the~upper or lower boundary of the~support of the~distribution is known it can be specified as $a$ resp.\@ $b$, in which case the~CDF will take the~value $0$ at $a$ and $1$ at $b$ and be linear on the~intervals $[a,min]$ and $[max,b]$. In this case the~CDF and PPF will be inverses of each other on the~interval $[a,b]$ resp.\@ $[0,1]$. In this case the~CDF will not accept inputs smaller than $a$ or larger than $b$.

If one does not want to make a~choice about the~assumed support of the~distribution, the~boundaries can be left unspecified in which case the~CDF will simply map all inputs smaller than $min$ to $\Delta$ resp.\@ all inputs larger than $max$ to $1-\Delta$. Similarly, the~PPF will map all inputs in $[0,\Delta]$ to $min$ resp.\@ all inputs in $[1-\Delta,1]$ to $max$. This option is essentially saying that we do not expect to receive inputs to the~CDF, which are outside of those we have already seen, but in case we do happen to receive such inputs we do not want to extrapolate since this might lead the~Redistributor transformation to produce unreasonably large outputs, depending on the~target PPF. Doing so could cause issues, for example, in the~case of the~target being a~Gaussian distribution, where the~PPF tends to $\pm\infty$ very quickly as the~input gets close to $1$ resp.\@ $0$. It is possible to only specify either $a$ or $b$ and leave the~other unspecified. In case of \Cref{alg:KDEW}, the~situation is simpler due to the~fact that it uses the~Gaussian kernel to estimate the~density, which has a~natural tail behaviour. 

\subsection{Time \& Space Complexity}
\label{subsec:complexities}

\begin{figure*}[!t]%
    \centering
    \vspace{-3mm}
    \subfloat[\centering Timing~\Cref{alg:L} -- LearnedDistribution]{{
        \hspace{-3mm}\includegraphics[width=0.5\textwidth]{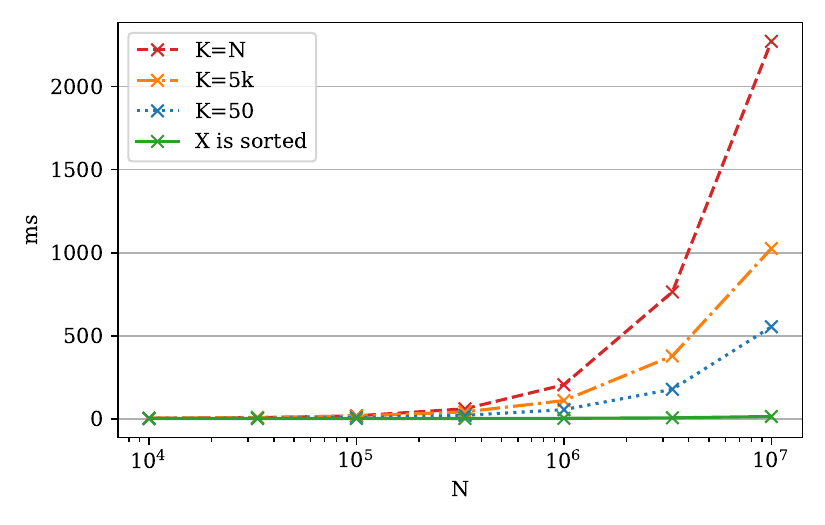} }}%
    \subfloat[\centering Timing~\Cref{alg:KDEW} -- KernelDensity]{{
        \includegraphics[width=0.5\textwidth]{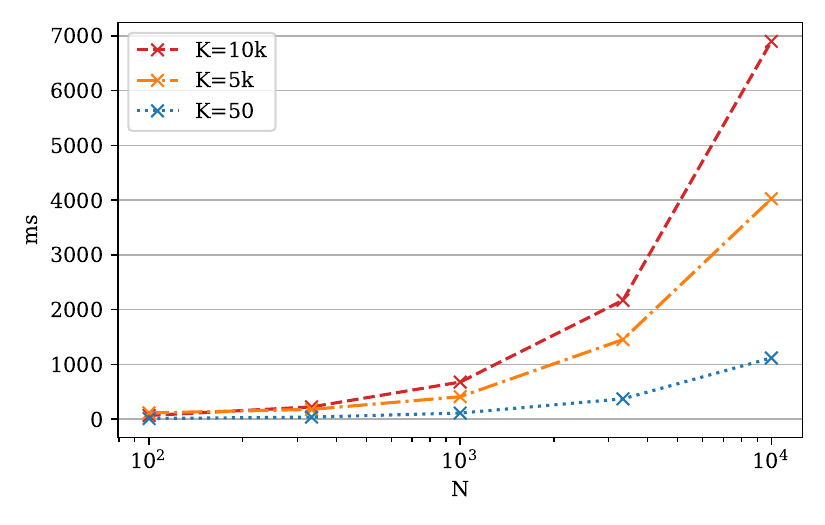} }}%
    \caption{Timing~\Cref{alg:L,alg:KDEW} on a consumer grade CPU -- Intel\,\textregistered\, Core\texttrademark~i7-8565U 1.80GHz. In both subfigures, N~denotes the~number of input data points. In (a), K~denotes the~number of bins and in (b), K~denotes the~grid density. Note that in comparison to \Cref{alg:KDEW}, \Cref{alg:L} can handle approx. 4~orders of magnitude more data points in the~same amount of time making it applicable in larger data-processing pipelines.} %
    \label{fig:timing}%
\end{figure*}

\noindent
The time complexity of~\Cref{alg:L} in the~worst case is $\mathcal{O}(N \log K)$, where $N$ is the~number of data points and $K$ is the~number of bins. The~bottleneck is partial sorting the~data with introspective sort\,\cite{musser1997introspective} to get the~lattice values. In the~best case, the~input array is already sorted and the~time complexity is $\mathcal{O}(N)$. In the~worst case $K = N$ and we need to do a~full sort instead of a~partial sort. In case the~lattice values contain repeated values, a~call to \verb!make_unique! will add $\mathcal{O}(K \log(K))$. Computation time as a~function of $N$ on a~specific hardware is shown in~\Cref{fig:timing}(a). 

The space complexity of \Cref{alg:L} is $\mathcal{O}(N)$. Auxiliary space complexity is $\mathcal{O}(N)$ if $keep\_x\_unchanged=False$, and $\mathcal{O}(K)$ otherwise. In practical terms, when we use a~1\,GB input array, the~peak memory will be approx. 2\,GB if we want to keep the~order of elements in the~input array unchanged. If changing the~order is not a~problem, the~peak memory will be approx. 1\,GB (not counting the~negligible overhead costs of the~interpreter, loaded modules, and K-sized arrays stored in the~interpolants), i.e.\@ almost no additional memory is used. Being aware of space complexity of this algorithm is important, as its speed allows us to process reasonably large amounts of data points for it to be relevant.

The time complexity of~\Cref{alg:KDEW} is $\mathcal{O}(NK)$, where $N$ is the~number of data points and $K$ is the~grid density. Most of the~time is spent on $N$ evaluations of the~Gaussian CDF on $K$-many points; however, evaluating the~$cdf$ on $K$~points is vectorized so the~speedup depends on particular hardware. Evaluating $ppf$ and $cdf$/fast/ has a~time complexity of linear interpolation. Computation time as a~function of $N$ on a~specific hardware is shown in~\Cref{fig:timing}(b). 

Space complexity of \Cref{alg:KDEW} is $\mathcal{O}(N)$. Auxiliary space complexity is $\mathcal{O}(K)$. That being said, the~speed of this algorithm is a~limiting factor on the~amount of input data points that can be reasonably processed. Broadly speaking, this algorithm can process approx.\,4 orders of magnitude fewer data points than \Cref{alg:L}. With this amount of data, the~space complexity is almost always irrelevant.

\section{Mathematical Considerations}
\label{sec:math} 
\noindent
In this section, we will establish a~notion of consistency of the~Redistributor transformation generated from an~empirical estimator for the~CDF of the~source distribution, i.e.\@ describe the~asymptotic behaviour of $F_T^{-1}\circ F^h_n$ when compared to $F_T^{-1}\circ F$, where $F^h_n$ is our empirical estimator generated from $n$ iid samples of $F$. 
This will be accomplished by viewing the~transformation as a~statistical functional on a~space containing cumulative distribution functions, and employing what is commonly referred to as the~functional delta method. We  derive a~statement of the~form
\begin{align*}
    \sqrt{n}(R_{g,x}(F_n)-R_{g,x}(F))\rightsquigarrow \mathcal{N}(0,\sigma_{g,F,x}),
\end{align*}
where $\mathcal{N}(0,\sigma_{g,F,x})$ denotes the~normal distribution with variance $\sigma_{g,F,x}$ and $R_{g,x}(F):=(g\circ F)(x)$ describes the~output of the~transformation for the~input $x$. Here $g$ is some function, which in application, would be the~PPF $F_T^{-1}$ of some target distribution. The~statement and the~required conditions on $g$, $F$, and $x$ will be made precise in Proposition~\ref{lem:final}. To do so we require a~certain amount of technical apparatus (see, e.g.,\,\cite{vaart_1998}), which we will introduce now.

\begin{definition}
Let $(X,\|\cdot\|_X)$, $(Y,\|\cdot\|_Y)$ be Banach spaces and $D\subseteq X$. A~map $\phi\colon D\to Y$ is called \textbf{Hadamard differentiable} at $F\in D$ if there exists a~continuous linear map {$\phi'(F)\colon X\to Y$} such that for every $h\in X$, $(h_n)_{n\in\N}\subseteq X$, $(t_n)_{n\in\N}\subseteq \R$ with $\lim_{n\to\infty}\|h_n-h\|_X=0$, $\lim_{n\to\infty}|t_n|= 0$, and $x+t_nh_n\in D$, it holds that
\begin{align*}
    \lim_{n\to\infty}\|\tfrac{\phi(F+t_n h_n)-\phi(F)}{t_n}-[\phi'(F)](h)\|_Y=0.
\end{align*}
\end{definition}

With this notion of differentiability we can make use of the~following theorem.

\begin{theorem}[\cite{vaart_1998}, Thm.20.8]\label{thm:func_delt}
Let $(X,\|\cdot\|_X)$, $(Y,\|\cdot\|_Y)$ be Banach spaces, $D\subseteq X$ and let $\varphi\colon D\to Y$ be Hadamard differentiable at $F\in X$. Let $G$ and $F_n$, $n\in\N$, be $D$-valued random variables\footnote{We do not explicitly discuss the~underlying probability spaces, but refer the~interested reader to\,\cite{vaart_1998}.} and $r_n\in\R$, $n\in\N$, with $\lim_{n\to\infty}r_n =\infty$. Assume that $r_n(F_n-F)\rightsquigarrow G$. Then we have
\begin{align*}
    r_n(\varphi(F_n)-\varphi(F))\rightsquigarrow[\varphi'(F)](G).
\end{align*}
If $\varphi'[F]$ is defined and continuous on the~entirety of $X$, we also have\footnote{For a~definition of the~probabilistic Landau notation $\mathrm{\textbf{o}}_P$ see\,\cite{vaart_1998}.}
\begin{align*}
    r_n(\varphi(F_n)-\varphi(F)) = [\varphi'[G]](r_n(F_n-F))+ \mathrm{\textbf{o}}_P(1).
\end{align*}

\end{theorem}

As we are primarily interested in cases where the~cumulative distribution function $F$ is continuous and we are generating a~continuous empirical cumulative distribution function, a~natural Banach space to consider would be $(C_b(\R,\R),\|\cdot\|)$, i.e.\@ the~space of bounded continuous functions from $\R$ to $\R$ equipped with the~uniform norm
\begin{align*}
    \|f\|_\infty:=\sup_{t\in\R}|f(t)|.
\end{align*}
However in order to derive the~desired convergence results, we will need to deal with usual empirical cumulative distribution functions which are discontinuous. Therefore we also introduce the~larger space $S=(S,\|\cdot\|_\infty)$, where 
\begin{align*}
S:=\{f\colon\R\to\R\colon \forall x\in\R\colon (\exists\lim_{t\nearrow x}f(t) \wedge \lim_{t\searrow x}f(t)=x\} 
\end{align*}
is the~set of functions which are right continuous and have left limits. 
We can now consider, for $g\in C((0,1),\R)$ and $x\in\R$, the~real-valued statistical functional 
\begin{align*}
    R_{g,x}\colon S_x
    &\to \R\\
    F&\mapsto g(F(x)),
\end{align*}
where
\begin{align*}
     S_x:=\{F\in S\colon F(x)\in(0,1)\}.
 \end{align*}
The functional $R_{g,x}$ maps a~cumulative distribution function $F\in S_x$ to the~value at $x$ of the~Redistributor transformation generated from $g$ and $F$. We consider the~inverse cumulative distribution function $g$ as a~function on $(0,1)$ and restrict the~domain of $R_{g,x}$ to $S_x$ in order to avoid the~possibility of $g(F(x))\in\{-\infty,\infty\}$. As we have $\mathbb{P}_{x\sim F}[F(x)\in(0,1)]=1$ anyway, this is not a~practically relevant restriction. Note that $F\in S_x$ if $x$ is in the~interior of the~suppport of the~density function corresponding to $F$. We can now show the~following
\begin{lemma}\label{lem:Rdiv}
Let $g\in C^1((0,1),\R)$, $x\in \R$, and $F\in S_x$. Then the~Hadamard derivative of $R_{g,x}$ at $F$ is continuous and given by $[R'_{g,x}(F)](h)=g'(F(x))h(x)$ for $h\in S$.
\end{lemma}

\begin{proof}
Let $h\in S$, $(h_n)_{n\in\N}\subseteq S$, $(t_n)_{n\in\N}\subseteq \R$ with $\lim_{n\to\infty}\|h_n-h\|_\infty=0$, $\lim_{n\to\infty}|t_n|= 0$.\\
Then
\begin{align*}
    \left(\tfrac{R_{g,x}(F+t_n h_n)-R_{g,x}(F)}{t_n}\right)(x)=\tfrac{g(F(x)+t_n h_n(x))-g(F(x))}{t_n}.
\end{align*}
The definition of $S_x$ ensures $F(x)\in(0,1)$ and therefore continuity of $g$ guarantees the~existence of a~neighborhood around $F(x)$ on which $g$ is finite. As a~direct consequence of the~assumptions we have $\lim_{n\to\infty}|t_n h_n(x)|=0$ and consequently there exists $n_0\in\N$ such that $F(x)+t_n h_n(x)$ is in this neighborhood for all $n\geq n_0$. We will assume in the~following that $n\geq n_0$.\\
By considering the~Taylor expansion of $g$ around $F(x)$ evaluated at $F(x)+t_n h_n(x)$ we get that
\begin{align*}
    g(F(x)+t_n h_n(x))&=g(F(x)) + g'(F(x))t_n h_n(x)\\
    &\quad+ r(F(x)+t_n h_n(x))t_n h_n(x) ,
\end{align*}
where 
\begin{align}\label{eq:remainder_term}
 \lim_{n\to\infty} r(F(x)+t_n h_n(x))=0.
\end{align}
Consequently we get 
\small
\begin{align*}
 &|\tfrac{g(F(x)+t_n h_n(x))-g(F(x))}{t_n} - g'(F(x))h(x)|\\
 &\quad= |g'(F(x))h_n(x)+r(F(x)+t_n h_n(x))h_n - g'(F(x))h(x)|\\
 &\quad\leq |g'(F(x))(h_n(x)-h(x))| +  |r(F(x)+t_n h_n(x))h_n(x)|.
\end{align*}
\normalsize
Due to \eqref{eq:remainder_term} and $\lim_{n\to\infty}\|h_n-h\|_\infty=0$ this implies
\begin{align*}
    \lim_{n\to\infty}\left|\tfrac{R_{g,x}(F+t_n h_n)-R_{g,x}(F)}{t_n}-g'(F(x))h(x)\right|=0,
\end{align*}
which proves that $[R'_{g,x}(F)](h)=g'(F(x))h(x)$.\\
Moreover, $R'_{g,x}(F)$ is continuous as, for all $h_1,h_2\in S$, it holds that 
\begin{align*}
    |R(h_1)-R(h_2)|&=|g'(F(x))h_1(x)-g'(F(x))h_2(x)|\\
    &\leq |g'(F(x))|\ \|h_1-h_2\|_\infty.
\end{align*}
\end{proof}

Note that each empirical cumulative distribution function $F_n\in S$ is derived from a~real valued random variable with cumulative distribution function $F\in C_b(\R,\R)$ by application of a~deterministic function $h^*_n\colon T^n\to S$ to $(X_1,\dots,X_n)$, where $X_i\sim F$, $i\in\{1,\dots,n\}$, iid and\footnote{For simplicity of exposition we define $h^*_n$ not on $\R^n$ but assume the~inputs to be ordered and distinct. As the~input is intended to be a~vector with iid entries the~order should not matter, so we may, w.l.o.g.\@, assume it to be ordered. Moreover, as we assume $F$ to be continuous, the~input will have distinct entries with probability $1$, so the~values that $h^*_n$ takes for inputs, where the~entries are not distinct, does not affect its convergence (in distribution).}
\begin{align*}
    T^n:=\{x\in\R^n\colon x_1<\dots<x_n\}.
\end{align*}
Specifically we have 
\begin{align*}
    F_n(t)=(h^*_n(X_{\pi(1)},\dots,X_{\pi(n)})(t):=\tfrac{1}{n}\sum_{k=1}^n \one_{[X_k\leq t]},
\end{align*}
where $(X_{\pi(1)},\dots,X_{\pi(n)})$ is an~increasing rearrangement of $(X_1,\dots,X_n)$.\\
Colloquially speaking, the~sequence $(h^*_n)_{n\in\N}$ describes the~deterministic method to generate the~usual empirical cumulative distribution function out of $n$ of real-valued samples from $F$. 
In this functional setting the~convergence of the~empirical cumulative distributions functions can be described by the~following Theorem from\,\cite{vaart_1998}. 
\begin{theorem}\label{thm:Donsker}
Let $F\in S$ be the~cumulative distribution function of a~real-valued random variable and $(F_n)_{n\in\N}$ the~corresponding sequence of empirical cumulative distribution functions. Then we have
\begin{align}\label{Donsker}
    \sqrt{n}(F_n-F)\rightsquigarrow \mathbb{G}_F,
\end{align}
i.e.\@ convergence in distribution\footnote{A~sequence $(X_n)_{n\in\N}$ of random variables in some metric space $(M,\rho)$ converges in distribution to $X\in M$, written as $X_n\rightsquigarrow X$, if it holds for every bounded, continuous functional $f\colon M\to \R$ that $\lim_{n\to\infty}|\E[f(X_n)]-\E[f(X)]|=0$. Note that some rather subtle issues of measurabilty arise here, when considering the~underlying probability spaces, for a~treatment of which we refer to\,\cite{vaart_1998}.} in the~space $(S,\|\cdot\|_\infty)$, where $\mathbb{G}_F$ is a~zero-mean Gaussian process with covariance
\begin{align*}
\mathrm{cov}[\mathbb{G}_F(s),\mathbb{G}_F(t)]=\min\{F(s),F(t)\}-F(s)F(t).    
\end{align*}
\end{theorem}

Our algorithm, however, generates a~continuous empirical distribution function\footnote{With some variations at the~boundaries depending on the~use case.} out of a~vector of samples $(X_1,\dots,X_n)$.  We therefore need to consider what conditions on a~sequence $h=(h_n)_{n\in\N}$ of functions $h_n\colon T^n\to S$ are required to ensure that \eqref{Donsker} remains valid. More precisely, we write
\begin{align*}
 F^h_n:=h_n(X_{\pi(1)},\dots,X_{\pi(n)})\in S   
\end{align*}
for the~function-valued random variable derived by the~deterministic function $h_n$ from the~real-valued cumulative distribution function $F\in C_b(\R,\R)$. This is done  via taking $X_i\sim F$, $i\in\{1,\dots,n\}$, iid, where, again, $(X_{\pi(1)},\dots,X_{\pi(n)})$ is an~increasing rearrangement of $(X_1,\dots,X_n)$. We them need to show that with this construction, we still obtain
\begin{align*}
    \sqrt{n}(F^h_n-F)\rightsquigarrow \mathbb{G}_F, 
\end{align*}
To accomplish this we will make use of a~continuous mapping theorem for Banach-space-valued random variables.
\begin{theorem}[\cite{vaart_1998}, Thm.18.11
]\label{thm:contmap}
Let $(D,\|\cdot\|_D)$, $(E,\|\cdot\|_E)$ be Banach spaces.
Assume that $D_n\subseteq D$ and $g_n\colon D_n\to E$, $n\in\N\cup\{\infty\}$, satisfy for every sequence $(x_n)_{n\in\N}$ with $x_n\in D_n$, that for every subsequence $(x_{n_k})_{k\in\N}\subseteq\N$ we have
\begin{align*}
    &\Big(\exists s\in D_{\infty} \colon \lim_{k\to\infty}\|x_{n_k}-s\|_D=0\Big)\\
    &\quad\implies \lim_{k\to\infty}\|g_{n_k}(x_{n_k})-g_\infty(s))\|_E =0.
\end{align*}
Then, for any sequence of random variables $X_n\in D_n$, $n\in\N$, and random variable $X\in D_\infty$ with $g_\infty(X)\in E$, it holds that
\begin{align*}
    X_n\rightsquigarrow X \implies g_n(X_n)\rightsquigarrow g_\infty(X).
\end{align*}
\end{theorem}

We obtain the~following Lemma, which essentially states that, if a~deterministic construction method $(h_n)_{n\in\N}$ is sufficiently close to the~one of the~usual empirical cumulative distribution function, i.e.\@ $(h^*_n)_{n\in\N}$, then the~corresponding sequences of function-valued random variables $(F^h_n)_{n\in\N}$ and $(F_n)_{n\in\N}$ enjoy the~same kind of convergence.  

\begin{lemma}\label{lem:3_7}
    Let $F\in C_b(\R,\R)$ be the~cumulative distribution function of a~real-valued random variable and $h_n\colon$~{$T^n\to$~$S$,~$n\in\N$}, a~sequence of functions. Assume that there exist constants $C>0$ and $c>\tfrac{1}{2}$ such that for all $n\in\N$, $x\in\R^n$ it holds that
    \begin{align}\label{eq:conv_req}
       \|h_n(x)-h^*_n(x)\|_\infty \leq Cn^{-c}.
    \end{align}
    Then we have
    \begin{align*}
        \sqrt{n}(F^h_n-F)\rightsquigarrow \mathbb{G}_F,
    \end{align*}
    where $\mathbb{G}_F$ is a~zero-mean Gaussian process with covariance
    \begin{align*}
        \text{cov}[\mathbb{G}_F(s),\mathbb{G}_F(t)]=\min\{F(s),F(t)\}-F(s)F(t).    
    \end{align*}
\end{lemma}
\begin{proof}
This proof is effected by a~fairly straightforward application of Theorem~\ref{thm:contmap}. To this end, let $D=D_\infty=E=(S,\|\cdot\|_\infty)$ and, for $n\in\N$,
\begin{align*}
    D_n:=\{\x\in S\colon x\in T^n\},
\end{align*}
where 
\begin{align*}
    \x(t)\!:=\!\sqrt{n}\big(F(t)-\tfrac{1}{n}\sum_{k=1}^n \one_{[x_k\leq t]}\big)\!=\!\sqrt{n}\big(F(t)-(h^*_n(x))(t)\big),
\end{align*}
$x\in T^n$, $t\in\R$.
In particular we have that $\sqrt{n}(F_n-F)$ is a~random variable in the~function space $D_n\subseteq S$. Moreover, the~map $T^n\ni x\mapsto \x\in D_n$ is bijective since $F$ is  continuous, which means that the~points of discontinuity of $\x$ uniquely determine $x$. Thus there are well-defined maps {$g_n\colon D_n\to S$}, $n\in\N$, which satisfy for all $n\in\N$, $\x\in D_n$, that
\begin{align*}
    g_n(\x)=\sqrt{n}(F-h_n(x)).
\end{align*}
We will show that the~requirement in Theorem~\ref{thm:contmap} is satisfied with $g_\infty\colon S\to S$ being the~identity. To this end consider $(\x_n)_{n\in\N}$ with $\x_n\in D_n$ and assume there exists a~subsequence $(\x_{n_k})_{k\in\N}$
and $s\in S$ such that 
\begin{align}\label{eq:conThmAss}
    \lim_{k\to\infty}\|\x_{n_k}-s\|_\infty=0.
\end{align}
We have
\begin{align*}
    &\|g_{n_k}(\x_{n_k})-g_\infty(s)\|_\infty\\
    &\quad= \|\sqrt{n_k}(F-h_{n_k}(x_{n_k}))-s\|_\infty\\
    &\quad= \|\sqrt{n_k}(F-h_{n_k}(x_{n_k})+h^*_{n_k}(x_{n_k})-h^*_{n_k}(x_{n_k}))-s\|_\infty\\
    &\quad\leq \sqrt{n_k}\|h_{n_k}(x_{n_k})-h^*_{n_k}(x_{n_k})\|_\infty\\
    &\quad\quad+\|\sqrt{n_k}(F-h^*_{n_k}(x_{n_k}))-s\|_\infty.
\end{align*}
The first term vanishes for $k\to\infty$ as we assumed $c>\tfrac{1}{2}$ and
\begin{align*}
    \sqrt{n_k}\|h_{n_k}(x_{n_k})-h^*_{n_k}(x_{n_k})\|_\infty\leq Cn_k^{\frac{1}{2}-c},
\end{align*}
whereas the~second term vanishes for $k\to\infty$ due to assumption \eqref{eq:conThmAss}. 
Consequently we have 
\begin{align*}
    \lim_{k\to\infty}\|g_{n_k}(\x)-g_\infty(s)\|_\infty=0
\end{align*}
as claimed. Note that  
\begin{align*}
    g_n(\sqrt{n}(F_n-F))&= g_n(\sqrt{n}(F- h^*_n(X_{\pi(1)},\dots,X_{\pi(n)})))\\
    &=\sqrt{n}(F- F^h_n(X_{\pi(1)},\dots,X_{\pi(n)}))\\
    &=\sqrt{n}(F^h_n-F).
\end{align*}
Since $\sqrt{n}(F_n-F)\rightsquigarrow \mathbb{G}_F$ by Theorem~\ref{thm:Donsker}, application of Theorem~\ref{thm:contmap} yields $\sqrt{n}(F^h_n-F)\rightsquigarrow \mathbb{G}_F$, which completes the~proof.
\end{proof}
We next argue that  
the constructions $(h_n)_{n\in\N}$ we use to obtain the~ continuous empirical distribution function fulfill  condition \eqref{eq:conv_req}.  We first observe that for  $n\in\N$ and   $x\in T^n$, the~functions $(h_n)_{n\in\N}$  satisfy 
\begin{align*}
    &(h_n(x))(t)\in[0,\tfrac{1}{n+1}] &\mathrm{for}\ t\leq x_1,\\
    &(h_n(x))(t)\in[\tfrac{n}{n+1},1] &\mathrm{for}\ t\geq x_n,\\
    &(h_n(x))(x_k)=\tfrac{k}{n+1}     &\mathrm{for}\ k\in[n],
\end{align*}
and are linear on each interval $[x_k,x_{k+1}]$, {$k\in\{1,\dots,n-1\}$}. We therefore have for $n\in\N$, $x\in T^n$, $k\in\{1,\dots,n-1\}$, $t\in[x_k,x_{k+1})$ that
\begin{align*}
     |(h_n(x))(t)-(h^*_n(x))(t)|&\leq |(h_n(x))(t)-(h_n(x))(x_k)|\\
     &\quad+|(h_n(x))(x_k)-(h^*_n(x))(t)|\\
     &\leq |(h_n(x))(x_{k+1})-(h_n(x))(x_k)|\\
     &\quad+|(h_n(x))(x_k)-(h^*_n(x))(x_k)|\\
     &\leq \tfrac{1}{n+1}+(\tfrac{k}{n}-\tfrac{k}{n+1})\leq 2n^{-1}.
\end{align*}
Moreover, we have for $t<x_1$ that
\begin{align*}
    |(h_n(x))(t)-(h^*_n(x))(t)|&=|(h_n(x))(t)|\leq|(h_n(x))(x_1)|\\
    &=\tfrac{1}{n+1}
\end{align*}
and for $t\geq x_n$ that 
\begin{align*}
|(h_n(x))(t)-(h^*_n(x))(t)|&=|1-(h_n(x))(t)|\\
&\leq|1-(h_n(x))(x_1)|\\
&=\tfrac{1}{n+1}.    
\end{align*}
As such, our construction satisfies \eqref{eq:conv_req} and is covered by the~following proposition.
\begin{proposition}\label{lem:final}
    Let $g\in C^1((0,1),\R)$, $x\in \R$, $F\in S_x$, and $h=(h_n)_{n\in\N}$ such that the~assumptions of Lemma~\ref{lem:3_7} are satisfied. Then
\begin{align*}
    \sqrt{n}(R_{g,x}(F^h_n)-R_{g,x}(F))\rightsquigarrow g'(F(x))\mathbb{G}_F(x)
\end{align*}
and
\begin{align*}
    g'(F(x))\mathbb{G}_F(x)= \mathcal{N}\big(0,(g'(F(x)))^2(F(x)-F(x)^2)\big).
\end{align*}
\end{proposition}
\begin{proof}
Lemma~\ref{lem:3_7} ensures $\sqrt{n}(F^h_n-F)\rightsquigarrow \mathbb{G}_F$ and thus Theorem~\ref{thm:func_delt} yields
\begin{align*}
    \sqrt{n}(R_{g,x}(F^h_n)-R_{g,x}(F))\rightsquigarrow [R'_{g,x}(F)](\mathbb{G}_F(x)).
\end{align*}
Due to Lemma~\ref{lem:Rdiv} we further have
\begin{align*}
    [R'_{g,x}(F)](\mathbb{G}_F(x))&=g'(F(x))\mathbb{G}_F(x)\\
    &=g'(F(x))\mathcal{N}\big(0,F(x)-F(x)^2\big)\\
    &=\mathcal{N}\big(0,(g'(F(x)))^2(F(x)-F(x)^2)\big).
\end{align*}

\end{proof}
    
\subsection{Componentwise Application to Multi-dimensional Random Variables}
\label{sec:multidim}
\noindent
So far we have analyzed the~behaviour of the~Redistributor transformation using an~empirical estimator generated with iid samples from some real-valued random variable. 
We could however also consider a~vector-valued\footnote{The same observations, of course, hold when considering a~random variable whose values are matrices or, more generally, arrays of any dimension.} random variable \mbox{$X\in\R^d$}. Given samples $X^i\sim X$, $i\in[n]$, we~can consider the~collection of scalar values \mbox{$(X^i_j)_{i\in[n],j\in[d]}\subseteq\R$} and take\footnote{To simplify notation we write $[m]:=\{1,\dots,m\}$ for $m\in\N$.} the~empirical cumulative distribution function $h^*_{nd}((X^i_j)_{i\in[n],j\in[d]})$ as if it was a~vector of samples from a~real-valued random variable. We note that
\begin{align*}
    (h^*_{nd}((X^i_j)_{i\in[n],j\in[d]}))(t)&=\tfrac{1}{nd}\sum_{j\in [d]}\sum_{i\in [n]} \one_{[X^i_j\leq t]}\\
    &=\tfrac{1}{d}\sum_{j\in [d]}h^*_n((X^i_j)_{i\in[n]}).
\end{align*}
Which means this would be the~empirical cumulative distribution function of a~scalar-valued random $\widetilde{X}$ variable with cumulative distribution function $\widetilde{F}=\tfrac{1}{d}\sum_{j\in [d]} F_j$ with $F_j$ being the~cumulative distribution function of $X_j\in\R$. Similarly taking our empirical estimator to get $\widetilde{F}^h_n=h_n((X^i_j)_{i\in[n],j\in[d]}))$ would yield an~approximation, in the~sense of Proposition~\ref{lem:final}, of the~transformation $\widetilde{R}$ which satisfies
\begin{align*}
    \widetilde{R}(\widetilde{X}) \eqd Y,
\end{align*}
where $Y$ is a~random variable with cumulative distribution function $g$. While the~statistical implications of applying a~thusly created transformation component-wise to a~vector $x\in\R^d$, is quite unclear, this type of use of the~Redistributor transformation to, e.g.\@ images, yields some appealing results as can be seen in the~next section. 

\section{Use Cases}
\label{sec:use_cases}
\noindent
We demonstrate a~range of use cases in the context of image processing and discuss potential further applications.

\subsection{Color Correction}
\label{subsec:correction_usecase}
\noindent
Poorly lit or over/under-exposed photos are quite common, leading photo editing software to incorporate automated tools for their correction. Typically, these reference-less tools adjust the entire image's exposure or handle light and dark regions separately. However, Redistributor takes a~different approach. It modifies the pixel values to match their distribution with that of a~reference image. This reference image represents what the photo should look like had it been properly exposed. Notably, the process extends beyond simple exposure correction; it also addresses other color issues. In the~\Cref{fig:badlight} example, each RGB channel of the source image is redistributed individually to match the distribution of the target image outperforming a~standard auto-color tool available in Photopea.com software.

In~\Cref{sec:alt_spaces} we describe more advanced approaches to redistribution that are also capable of correcting illumination, color temperature, and other visual aspects of the images.

In \Cref{subsec:comparison-correction} we evaluate Redistributor on a~dataset with available ground truth in order to objectively compare it's performance to other color correction methods that use a~reference image.

\begin{figure}
    \centering
    \vspace{-3mm} 
    \subfloat[\centering Underexposed (source)]{{
        \includegraphics[width=0.225\textwidth]{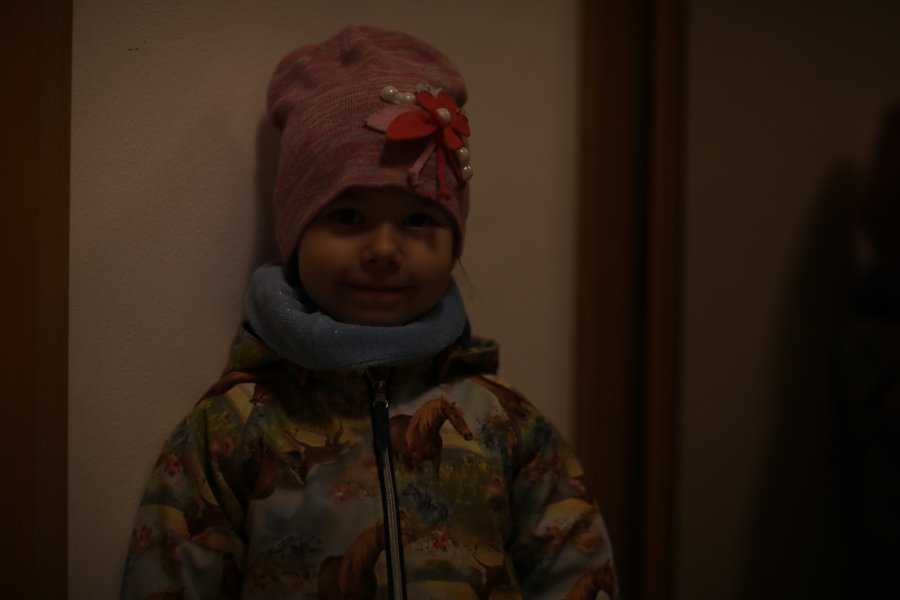} }}%
    \hspace{2mm}\subfloat[\centering Well-exposed (target)]{{
        \includegraphics[width=0.225\textwidth]{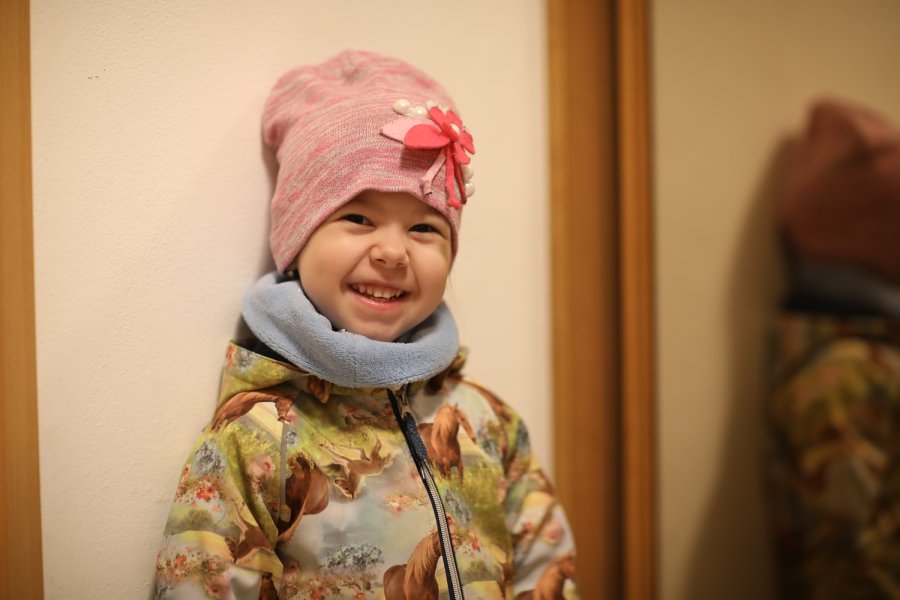} }}\vspace{-2mm}
        
    \subfloat[\centering Corrected (our~method)]{{
        \includegraphics[width=0.225\textwidth]{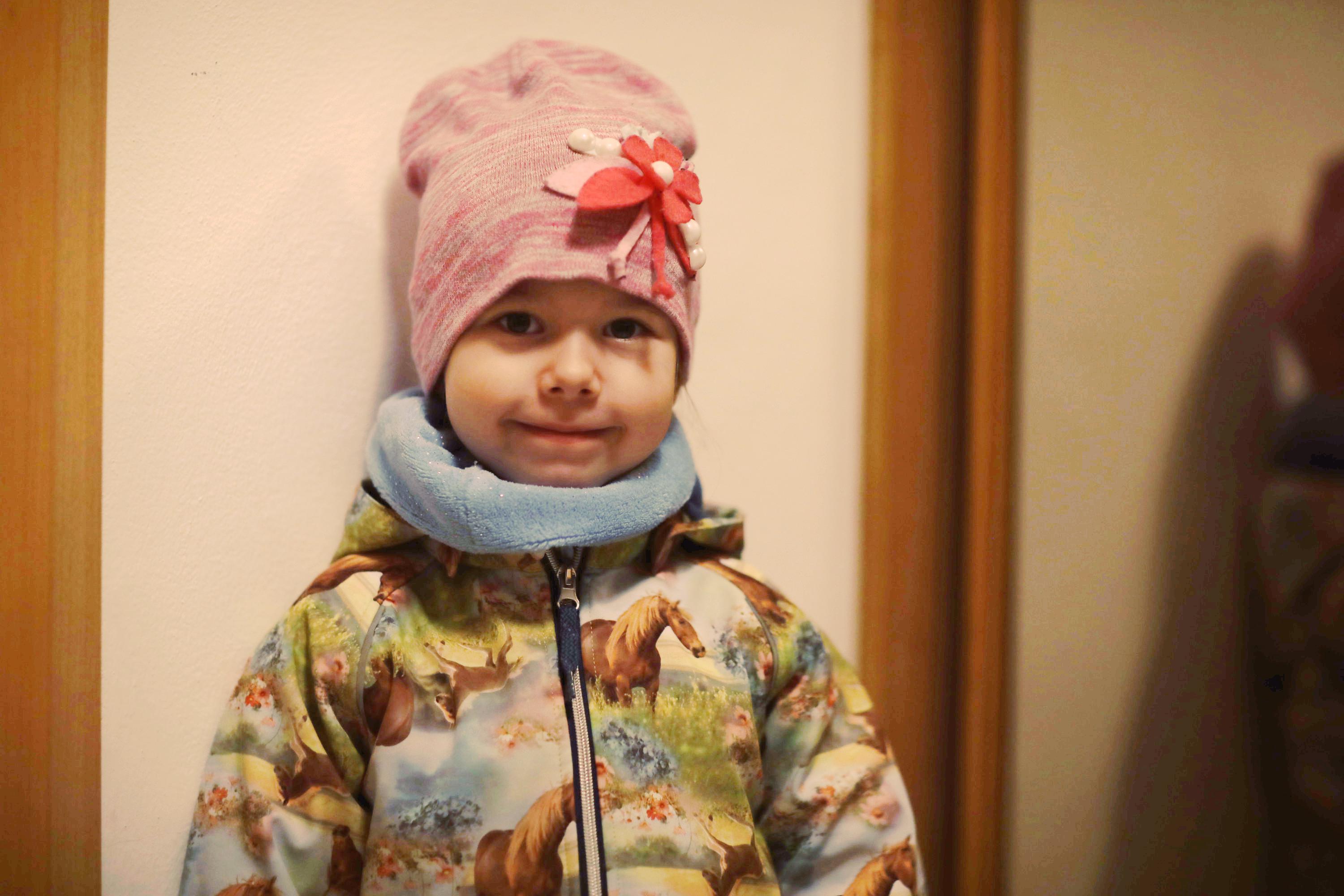} }}%
    \hspace{2mm}\subfloat[\centering Corrected (auto-color)]{{
        \includegraphics[width=0.225\textwidth]{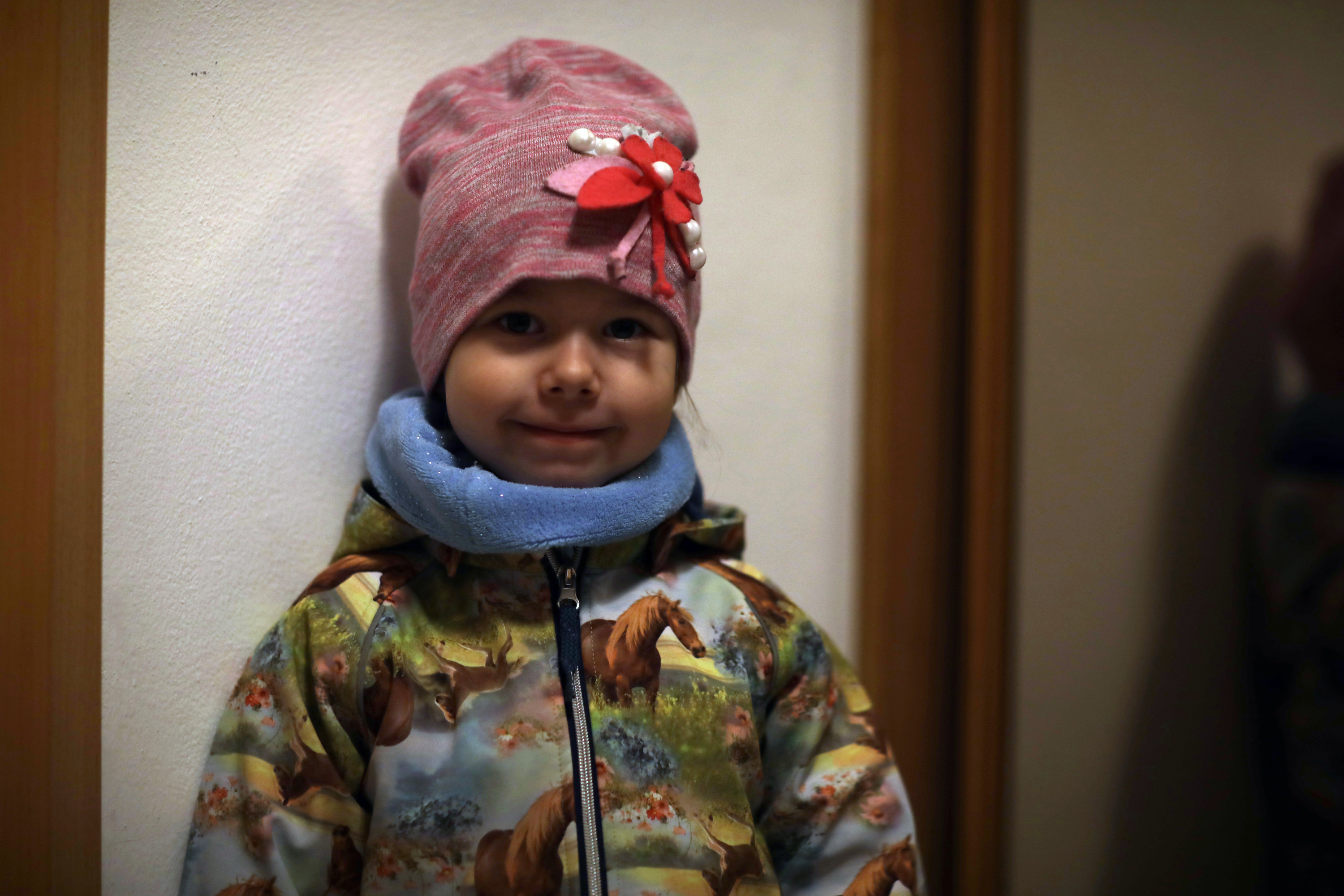} }}%
    \caption{Correcting exposure using a~reference image.
    }%
    \label{fig:badlight}%
\end{figure}

\begin{figure}
\captionsetup[subfigure]{labelformat=empty}
    \centering
    \vspace{-3mm}
    \subfloat{{\label{fig:matchcolor-a}
        \includegraphics[height=0.08\paperheight]{images/MatchColor/beach.jpg} }}\vspace{-1mm}%
    \subfloat{{\label{fig:matchcolor-b}
        \includegraphics[height=0.08\paperheight]{images/MatchColor/beach2.jpg} }}\vspace{-1mm}%
    \subfloat{{\label{fig:matchcolor-c}
        \includegraphics[height=0.08\paperheight]{images/MatchColor/beach-to-beach2.jpg} }}%
        
    \subfloat{{\label{fig:matchcolor-d}
        \includegraphics[height=0.08\paperheight]{images/MatchColor/lakenight.jpg} }}\vspace{-1mm}%
    \subfloat{{\label{fig:matchcolor-e}
        \includegraphics[height=0.08\paperheight]{images/MatchColor/nightsky2.jpg} }}\vspace{-1mm}%
    \subfloat{{\label{fig:matchcolor-f}
        \includegraphics[height=0.08\paperheight]{images/MatchColor/lakenight-to-nightsky2.jpg} }}%
        
    \subfloat{{\label{fig:matchcolor-g}
        \includegraphics[height=0.08\paperheight]{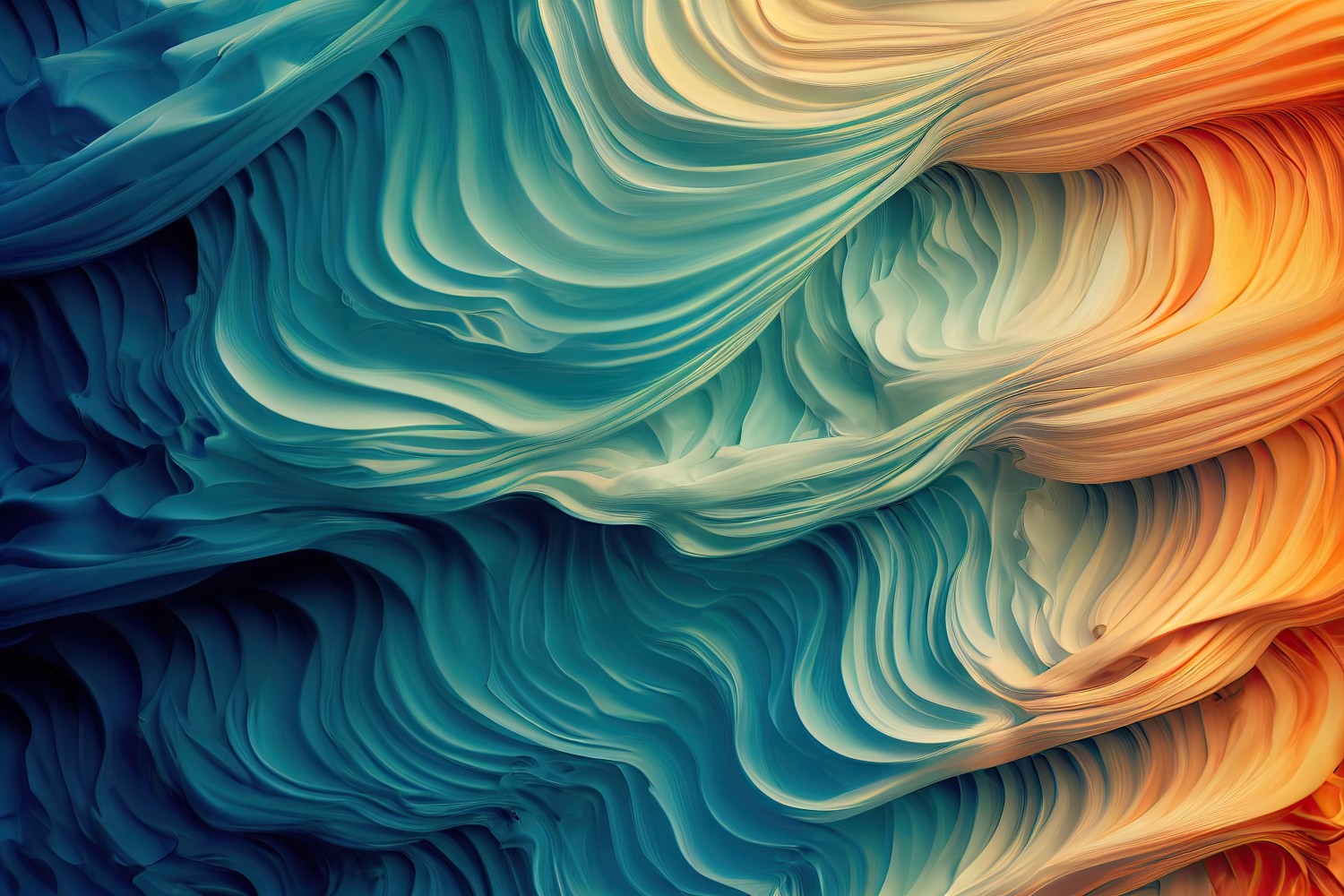} }}\vspace{-1mm}%
    \subfloat{{\label{fig:matchcolor-h}
        \includegraphics[height=0.08\paperheight]{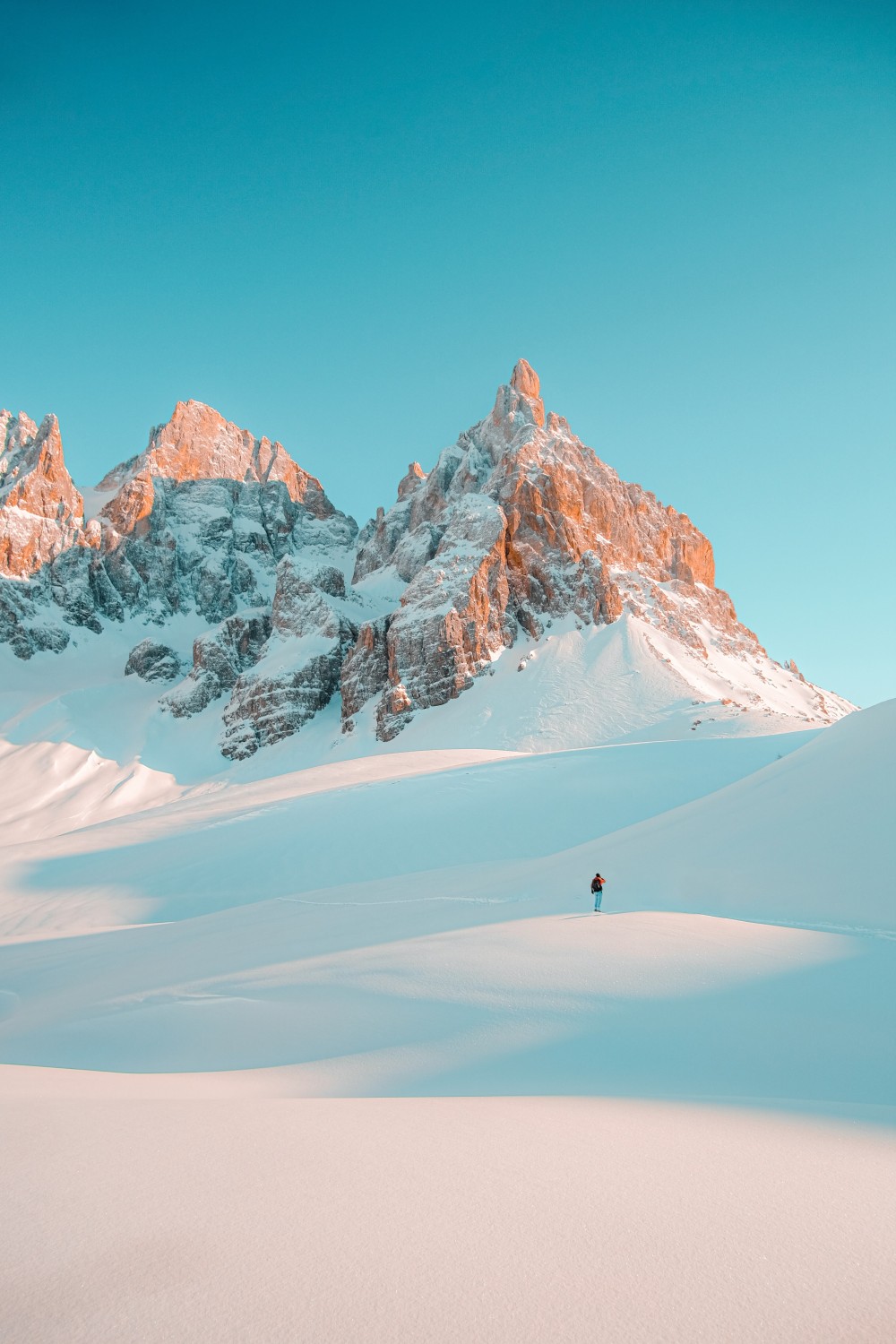} }}\vspace{-1mm}%
    \subfloat{{\label{fig:matchcolor-i}
        \includegraphics[height=0.08\paperheight]{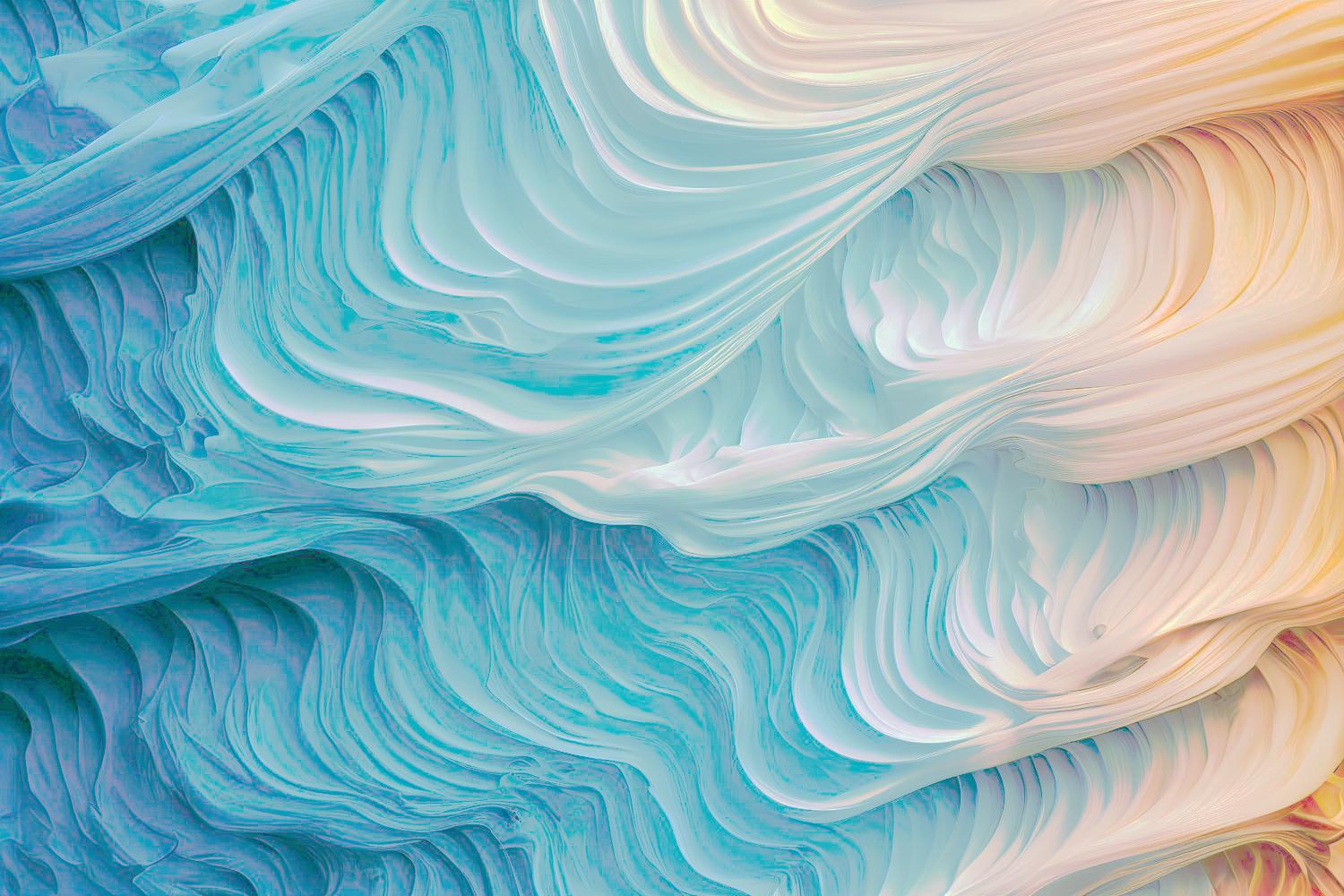} }}%
        
    \subfloat[\centering Source]{{\label{fig:matchcolor-j}
        \includegraphics[height=0.08\paperheight]{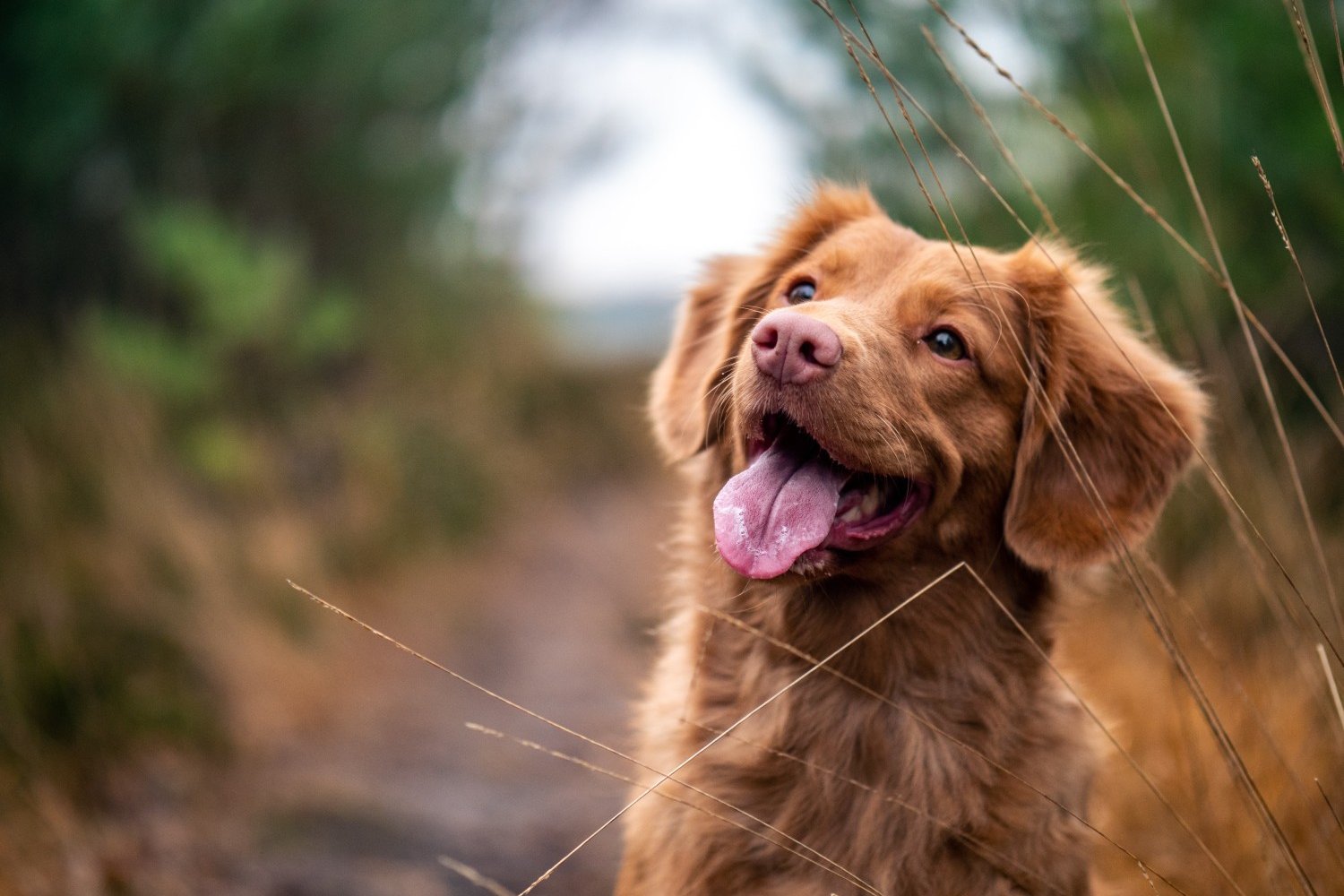} }}%
    \subfloat[\centering Target]{{\label{fig:matchcolor-k}
        \includegraphics[height=0.08\paperheight]{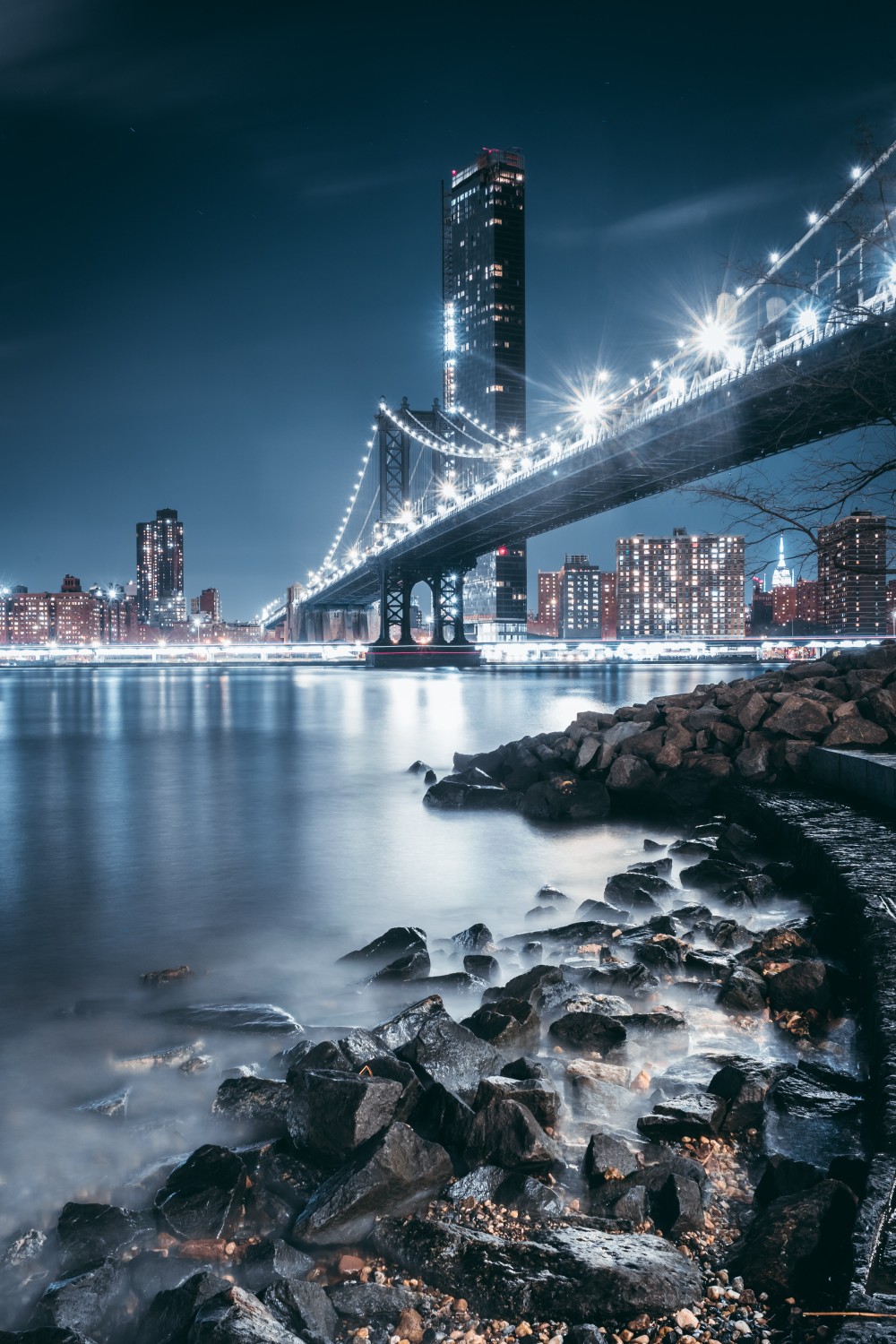} }}%
    \subfloat[\centering Output (our method)]{{\label{fig:matchcolor-l}
        \includegraphics[height=0.08\paperheight]{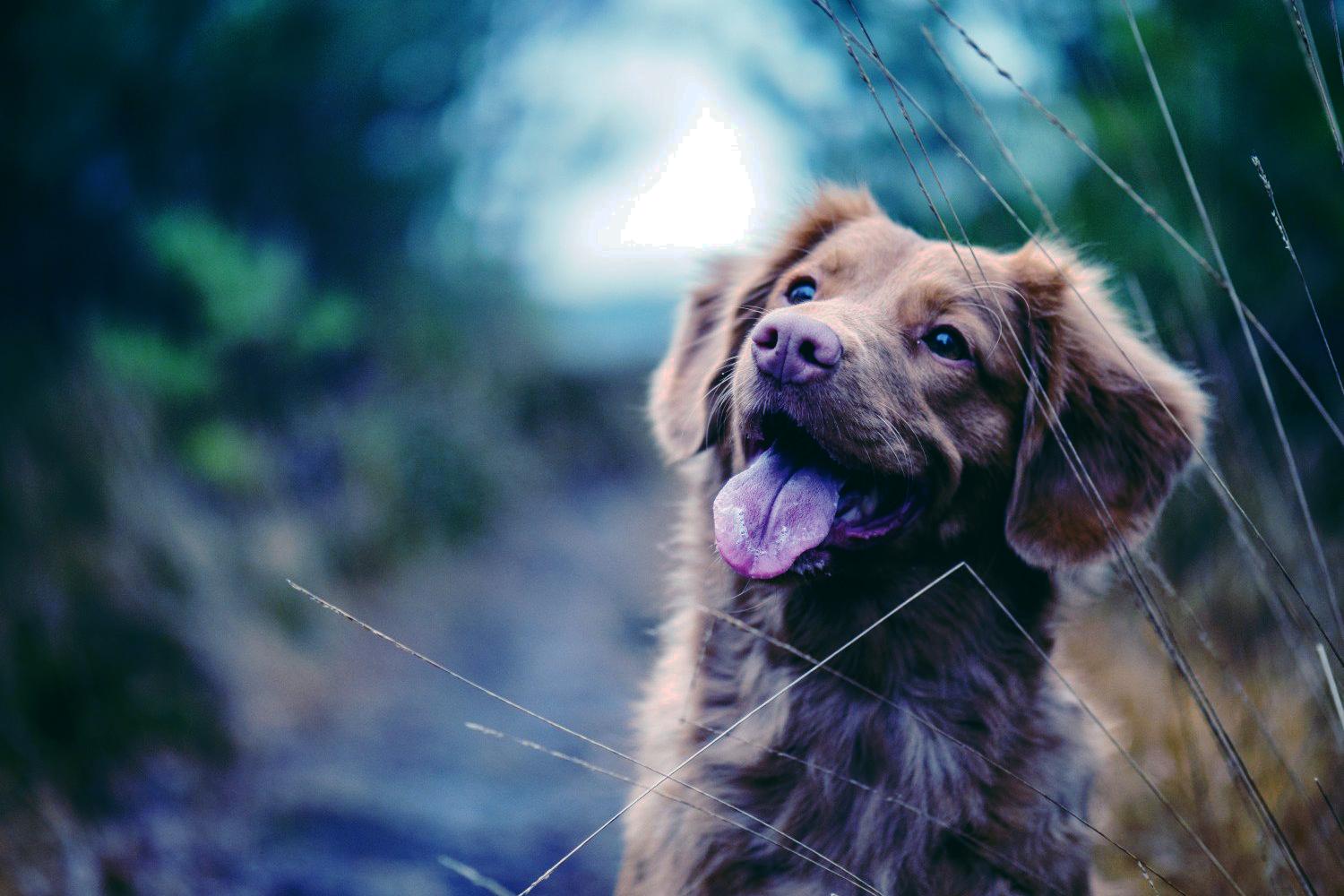} }}%
    \caption{Matching colors of a~reference image.
    }%
    \label{fig:matchcolor}%
\end{figure}

\subsection{Photorealistic Style Transfer}
\label{subsec:nst_usecase}
\noindent
Moving in the~direction of AI-enhanced art, we explore usage of Redistributor for transforming images into new color schemes, demonstrated in~\Cref{fig:matchcolor}. By adjusting RGB colors using a~reference image, often of a~very different scene, we can achieve impressive effects. This approach is versatile, from creating surreal images to maintaining consistent color themes for design or decor. Yet, not all source-target pairs yield pleasing results. In such cases, advanced transformations from~\Cref{sec:alt_spaces} are more suitable. In~\Cref{subsec:comparison-nst}, we compare Redistributor to an NST method for photorealistic style transfer and highlight Redistributor's advantages.

\begin{figure*}
    \centering
    \subfloat[\centering Mosaic 1 -- $50\times50$ mosaic made of 1000\,px$^{2}$ tiles redistributed to a~Gaussian with $\mu=$ desired pixel value and $\sigma=0.05$. The~original image is overlaid over the~mosaic with $\alpha=0.25$ to regain high-frequency content.
    ]{{
        \includegraphics[width=0.82\textwidth]{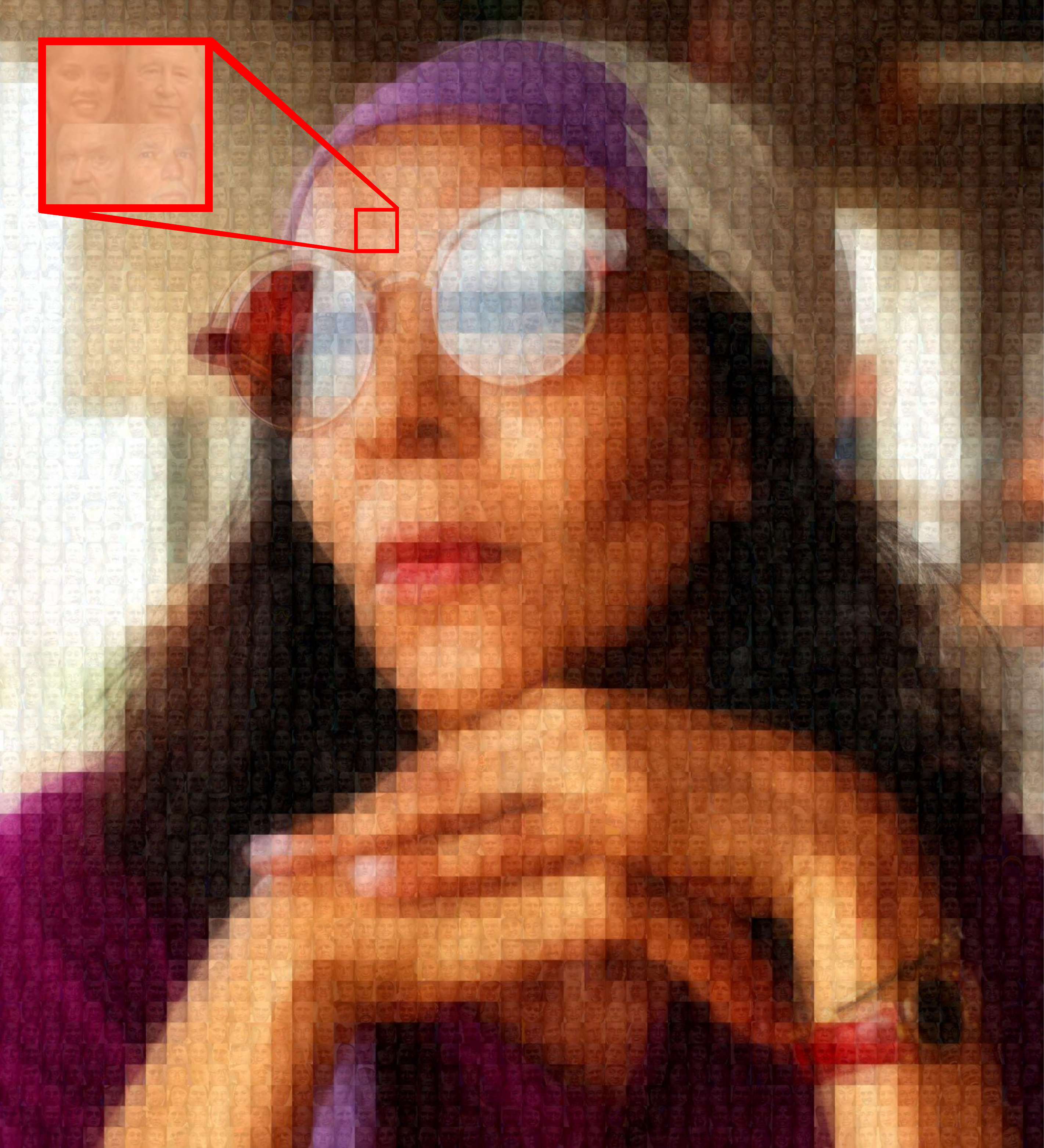} }}%
        
    \subfloat[\centering $\alpha=0.00$, $\sigma=0.05$]{{
        \includegraphics[width=0.263\textwidth]{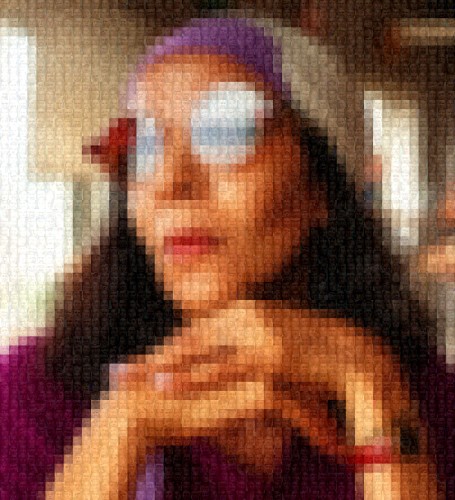} }}%
    \subfloat[\centering $\alpha=0.25$, $\sigma=0.20$]{{
        \includegraphics[width=0.263\textwidth]{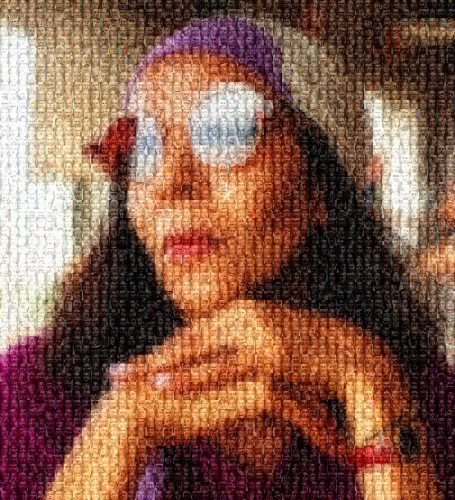} }}%
    \subfloat[\centering $\alpha=0.25$, $\sigma=0.40$]{{
        \includegraphics[width=0.263\textwidth]{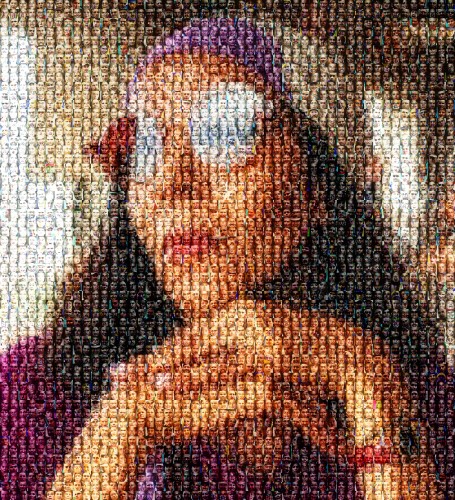} }}%

    \caption{Mosaic effect achieved by tiling redistributed images from the~LFW dataset\,\cite{LFWTech}.
    }%
    \label{fig:mosaic}%
\end{figure*}

\subsection{Making Photomosaics}
\label{subsec:photomosaics}
\noindent
As a~further use case in image modification, Redistributor can easily be used to make photographic mosaics (photomosaics). These are images that have been divided into smaller, tiled regions, each of which is replaced by an~image with similar mean color value as the~original tile. Usually, mosaics are created by selecting the~smaller images with fitting mean from a~large collection. The~resulting granularity is a~function of both the~size of the~collection as well as the~size of the~tiles. Redistributor can instead transform any provided image to a~Gaussian distribution with mean $\mu$ matching the~tile to replace. By doing so, one could even make a~mosaic by repeatedly using only a~single image. Furthermore, the~granularity of the~mosaic can be easily modified using a~parameter $\sigma$ of the~target Gaussian distribution, either keeping or compressing the~variability in an~individual sub-image. Moreover, the~original full-resolution image can be overlaid over the~mosaic to regain high-frequency content, adding another parameter $\alpha$ influencing the~ratio between the~mosaic and the~original image it depicts. Examples of mosaics created with varying parameters are shown in \Cref{fig:mosaic}.

\begin{figure*}
    \centering
    \vspace{-3mm}
    \hspace{-1.5mm}\subfloat[\centering Input image]{{
        \includegraphics[width=0.16\textwidth]{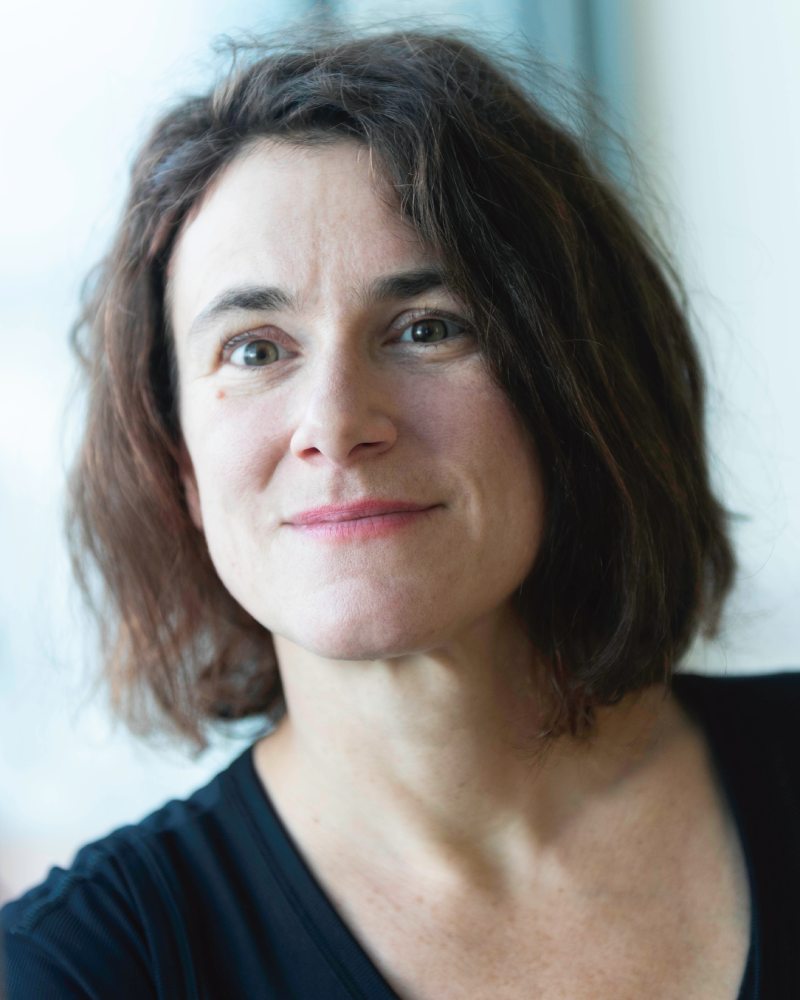}}}%
    \subfloat[\centering Input image]{{
        \includegraphics[width=0.16\textwidth]{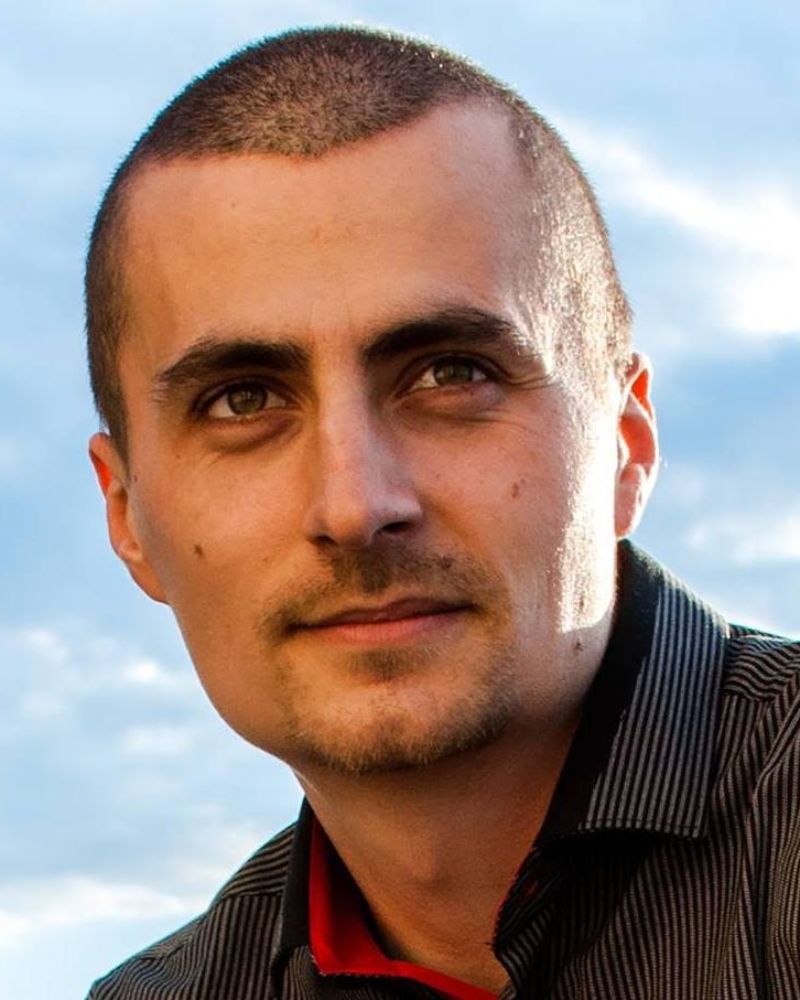}}}%
    \subfloat[\centering $(a)\rightarrow(b)$~Combined]{{
        \includegraphics[width=0.16\textwidth]{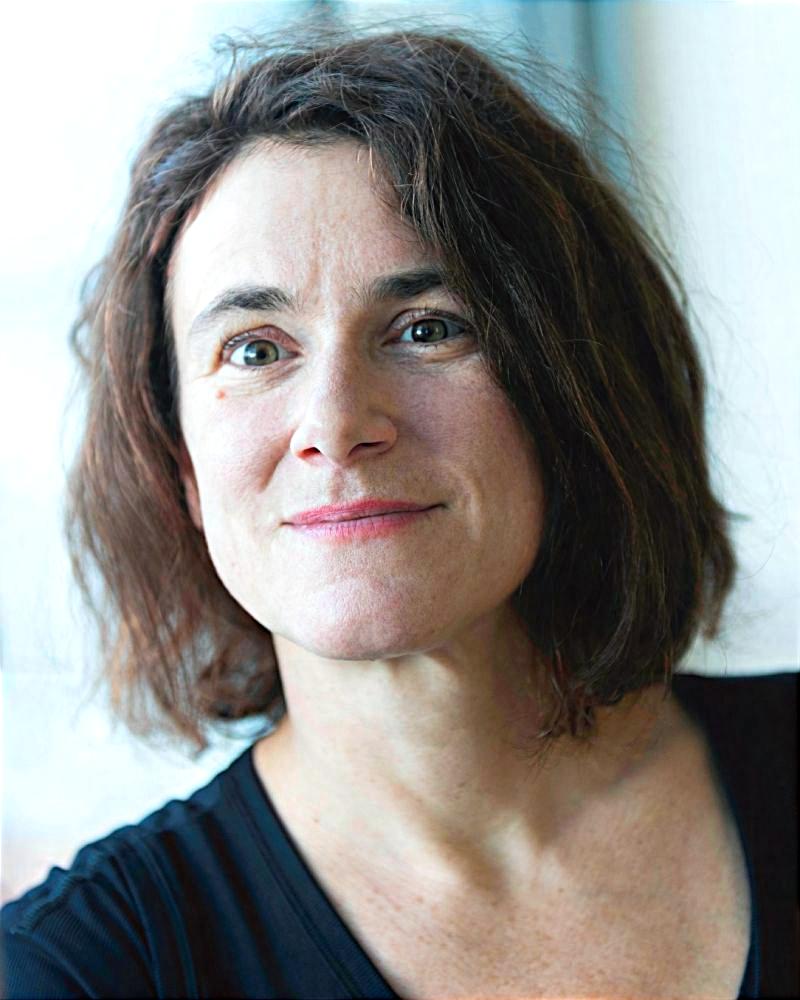}}}%
    \subfloat[\centering $(a)\rightarrow(b)$~Separate]{{
        \includegraphics[width=0.16\textwidth]{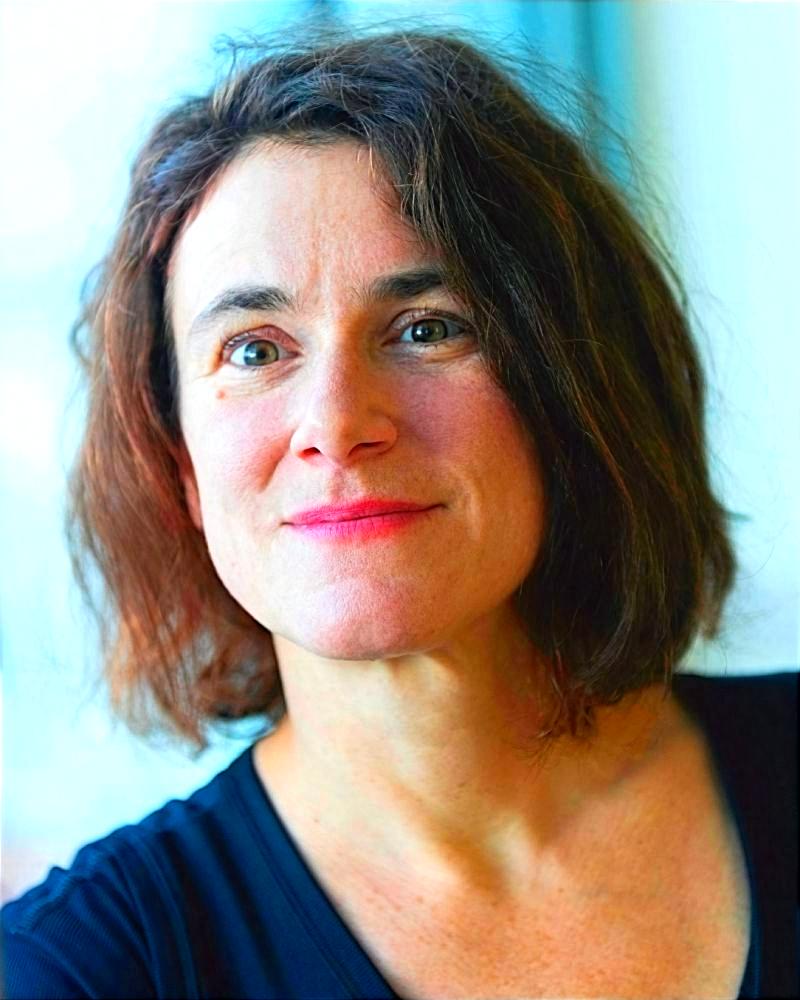}}}%
    \subfloat[\centering $(b)\rightarrow(a)$~Combined]{{
        \includegraphics[width=0.16\textwidth]{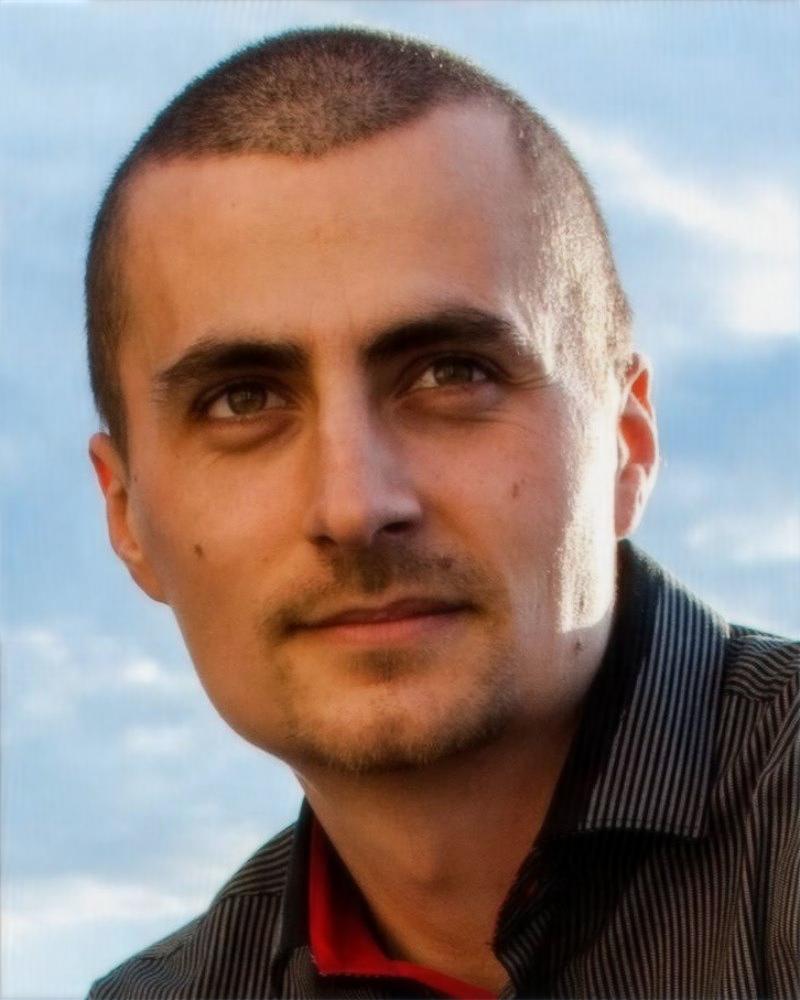}}}%
    \subfloat[\centering $(b)\rightarrow(a)$~Separate]{{
        \includegraphics[width=0.16\textwidth]{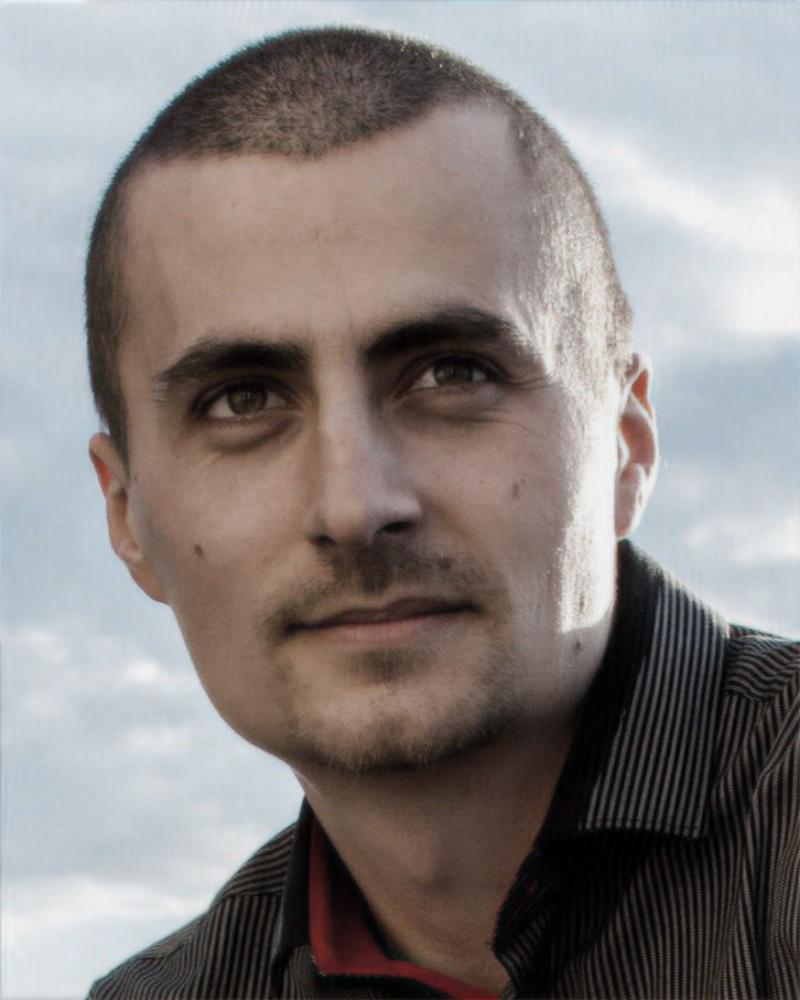}}}%
 
    \caption{Redistributing the amplitudes in Fourier space. Combined -- treating all 3 channels as one flattened array. Separate -- treating each channel separately. In all outputs, the sharpness is matched (zooming in might be necessary to see the effect). In Combined, the Output images tend to retain the colors of Source while in Separate, also the colors are adjusted according to Target (however, not in a~very predictable manner).
    }%
    \label{fig:fourier}%
\end{figure*}

\subsection{Applying the Transformation in Alternative Spaces}
\label{sec:alt_spaces}\noindent
As opposed to the previous use cases, this subsection describes an approach where processing happens in spaces other than the RGB color space. The idea of using different color representation spaces for color mapping is well established, cf.,~e.g.,~Section 2.2.2 in~\cite{ReinhardSurvey}, where transformations in LMS cones space or (CIE)LAB space, among others, are reviewed. We present two useful extensions. 

On the one hand, we consider application of Redistributor in Fourier space. More precisely, the amplitudes of an image's channels' Fourier coefficients are redistributed, while original phases are kept. Results are shown in \Cref{fig:fourier}, where both combined or separate channel processing is demonstrated. Both cases adjust image sharpness.

On the other hand, we devise and examine advanced sequential transformations that let us consecutively match histograms of chosen channels within different color spaces. This approach produces promising results, especially when combined with masking and weighting of each distinct transformation for better and more intuitive control of the final result. 
To be more precise, by masking we mean applying Redistributor between pre-specified regions of source and target images. Masking can be used to enforce mapping of the colors between semantically corresponding regions while weighting certain transformations helps to reduce undesired artifacts caused by heavily-skewed distributions of certain channels. See~\Cref{fig:weighting} for a~visual aid explaining what we mean by weighting in this context. To produce the results presented in~\Cref{sec:comparisons} and shown in~\Cref{fig:colormapping,fig:nst,fig:colormapping-perf}, we iteratively applied a~sequence of changes over the Saturation (in HSV color space), Lightness (in LAB), Red, Green, Blue (in RGB), and again Lightness (in LAB) channels. This approach is explainable, as each of the channels bears a~specific meaning, and also traceable, as one can visualize the intermediate result after each transformation.

\begin{figure}
    \centering
    \vspace{-3mm}
    \hspace{-1.5mm}\subfloat[\centering Target]{{
        \includegraphics[width=0.092\textwidth]{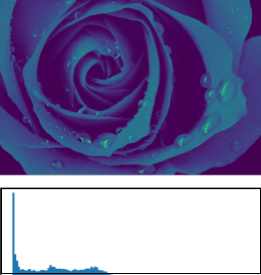}}}%
    \subfloat[\centering Source]{{
        \includegraphics[width=0.092\textwidth]{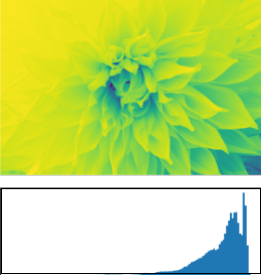}}}%
    \subfloat[\centering $w=0.33$]{{
        \includegraphics[width=0.092\textwidth]{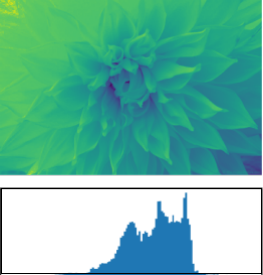}}}%
    \subfloat[\centering $w=0.67$]{{
        \includegraphics[width=0.092\textwidth]{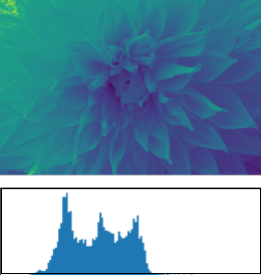}}}%
    \subfloat[\centering $w=1.00$]{{
        \includegraphics[width=0.092\textwidth]{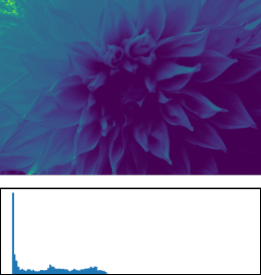}}}%

    \caption{Averaging the source\,(b) and transformed\,(e) values to achieve partial redistribution according to weight $w$.}
    \label{fig:weighting}%
\end{figure}


\begin{figure}
    \centering
    \includegraphics[width=0.48\textwidth]{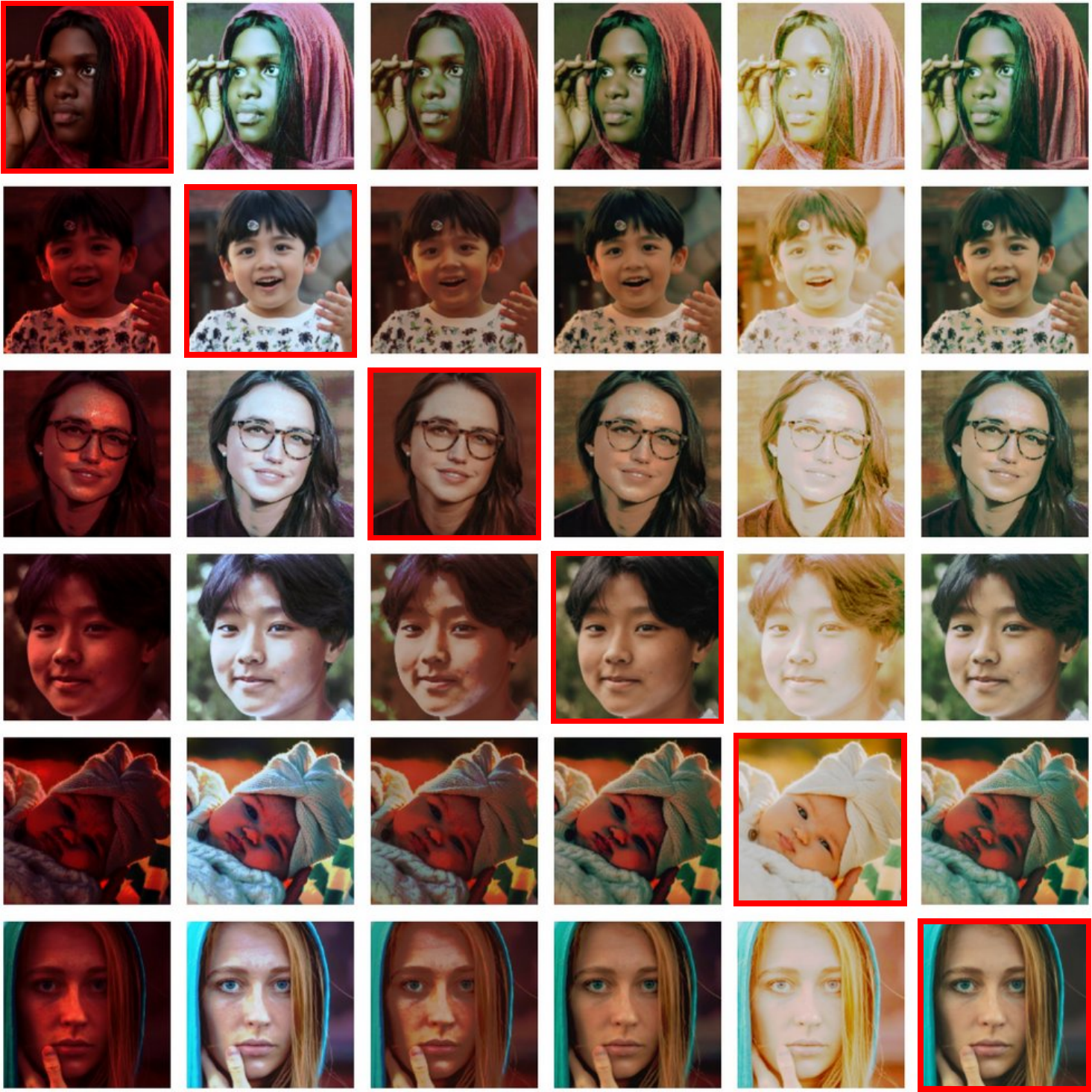}
    \caption{Augmenting image data in color space using distributions of other images within the~same batch. The~original images are highlighted using a~red border. All images within one column are redistributed according to the~image in the~same column that lies on the~diagonal; e.g.,\@ the image in row\,3, column\,2 is redistributed according to the~image in row\,2, column\,2.
    }%
    \label{fig:augmentation}%
\end{figure}

\subsection{Image Data Augmentation}
\label{subsec:augmentation}
\noindent
Redistributor may also be used for Image data augmentation in color space, as described in\,\cite{shorten2019survey}, i.e.\@ taking an~image and producing a~similar one by modifying the~histograms of the~RGB color channels.  
\Cref{fig:augmentation} illustrates the~example of a~batch of $6$ images, where out of each image $5$ new ones have been created by channel-wise redistribution to match each of the~other images. One can also create new versions of an~image by redistributing them to some manually chosen distribution of interest, depending on the~specifics of the~intended application. Naturally, Redistributor could also be applied to any other type of data as long as it consists of some collection of scalar values. Of course, it is not always clear whether these augmentations are necessarily sensible for the~intended task. While the~benefits and drawbacks need to be carefully considered based on the~goals of the~intended application, Redistributor always provides an~efficient and flexible tool for systematic data augmentation.

\subsection{Preprocessing for Machine Learning Tasks}
\label{subsec:mlpreprocessing}
\noindent
Another potential application of this algorithm would be as a~preprocessing step for machine learning tasks. It is designed to efficiently handle a~large number of samples with a~low memory footprint, and as it is implemented as a~Scikit-learn transformer, it can be conveniently inserted in common machine learning pipelines.  
It always produces an~invertible transformation, i.e.\@ does not, in principle, reduce the~input-output relationships which may be learned. In addition, \Cref{sec:math} shows that given $n$ iid samples of a~CDF $F_S$ and some specified target CDF $F_T$, it is a~consistent estimator of $F_T^{-1}\circ F_S$, i.e.\@ it can be expected to behave predictably on further samples of $F_S$. Particularly when applied component-wise to high-dimensional data, e.g.\@ images, it seems unlikely that it would improve things from a~statistical perspective. However, there might be other benefits. For example, when the~source distribution is very highly concentrated around some small value, it could serve as a~kind of normalization and lead to improved numerical performance of an~algorithm.

Anecdotally, we have observed various cases of improved performance on a~number machine learning tasks using neural networks. The~improvement is, however, not very consistent and often rather mild. Although, given the~simple and inexpensive nature of this preprocessing step it would be quite unreasonable to expect a~consistent significant improvement anyway. We abstain from presenting numerical examples for this use case, since this is not the~focus of the~paper and the~black-box nature of neural networks means that a~brief demonstration would require significant cherry picking.

\begin{figure*}[t]
    \captionsetup[subfigure]{labelformat=empty}
    \centering
    \vspace{-3mm}
\hspace{-0.4mm}\subfloat{{\includegraphics[width=0.193\textwidth,valign=c]{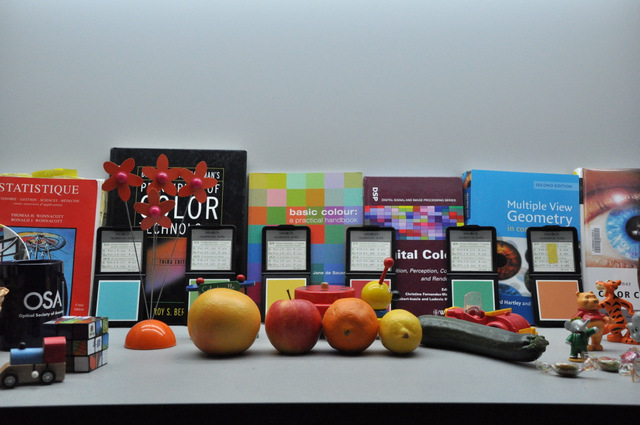}}}
\hspace{1.5mm}\subfloat{{\includegraphics[width=0.193\textwidth,valign=c]{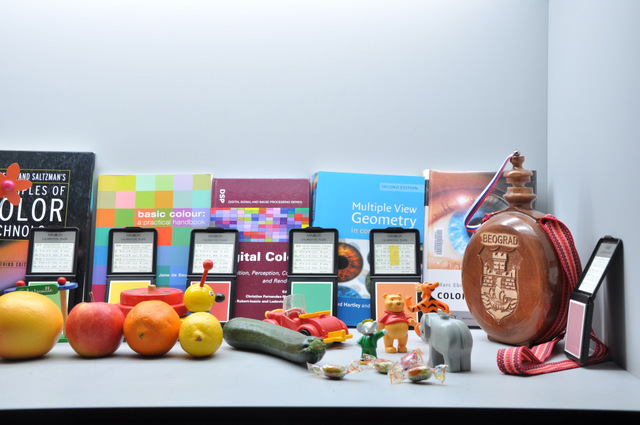}}}
\hspace{1.5mm}\subfloat{{\includegraphics[width=0.193\textwidth,valign=c]{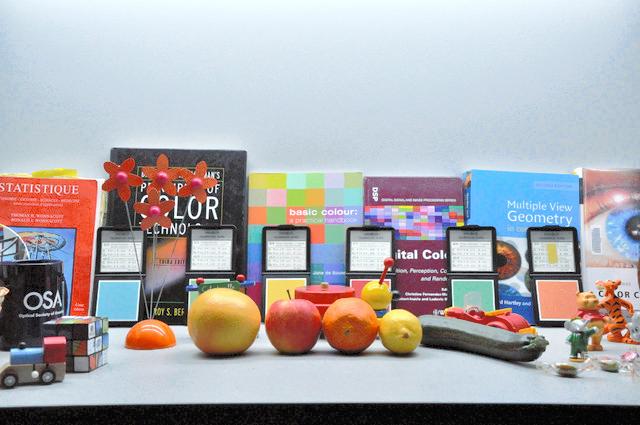}}}
\hspace{1.5mm}\subfloat{{\includegraphics[width=0.193\textwidth,valign=c]{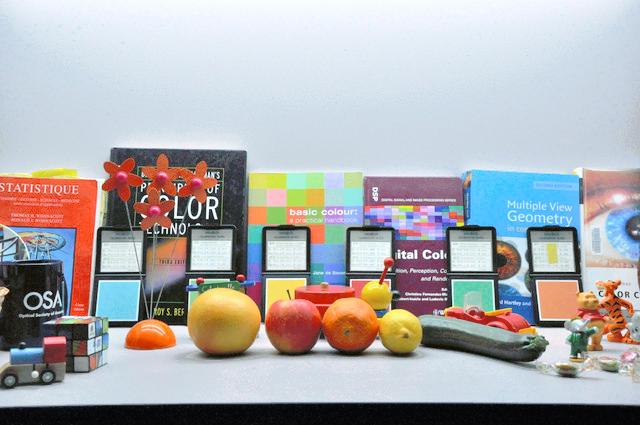}}}
\hspace{1.5mm}\subfloat{{\includegraphics[width=0.193\textwidth,valign=c]{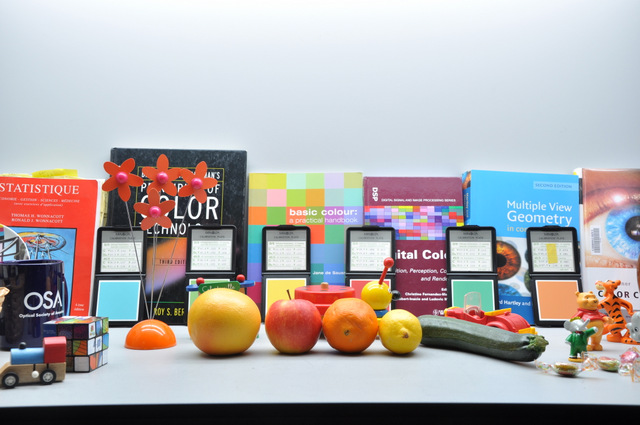}}}\vspace{-2mm}

\hspace{-0.4mm}\subfloat{{\includegraphics[width=0.193\textwidth,valign=c]{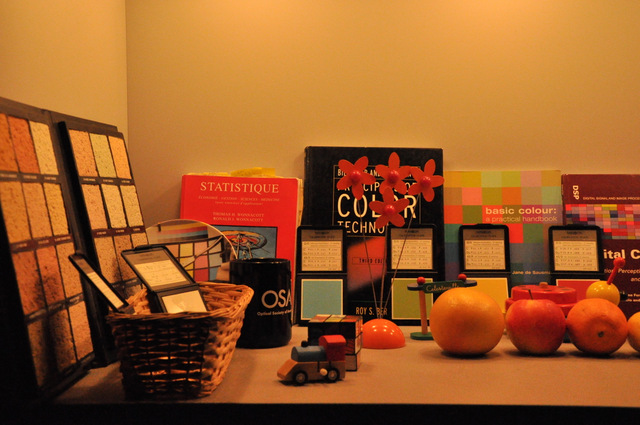}}}
\hspace{1.5mm}\subfloat{{\includegraphics[width=0.193\textwidth,valign=c]{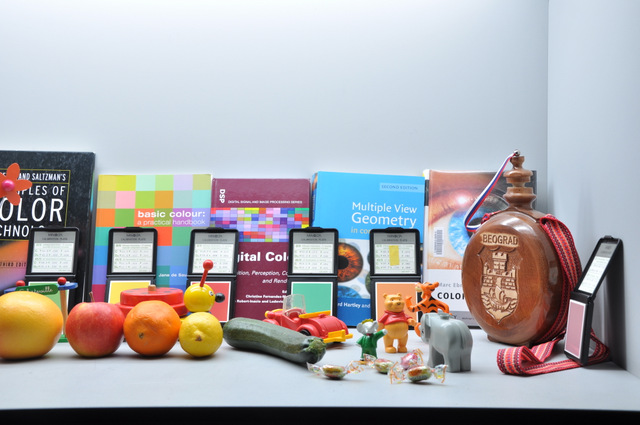}}}
\hspace{1.5mm}\subfloat{{\includegraphics[width=0.193\textwidth,valign=c]{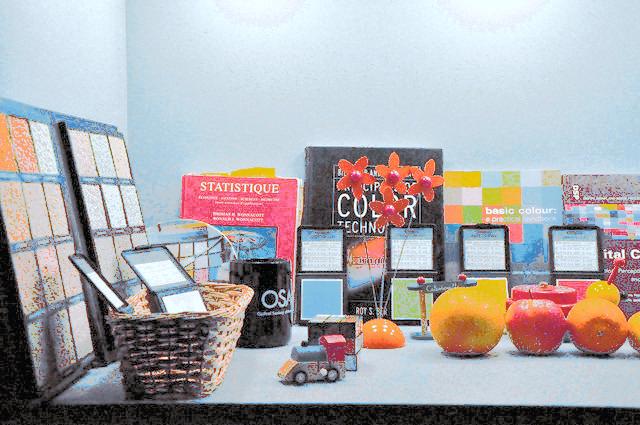}}}
\hspace{1.5mm}\subfloat{{\includegraphics[width=0.193\textwidth,valign=c]{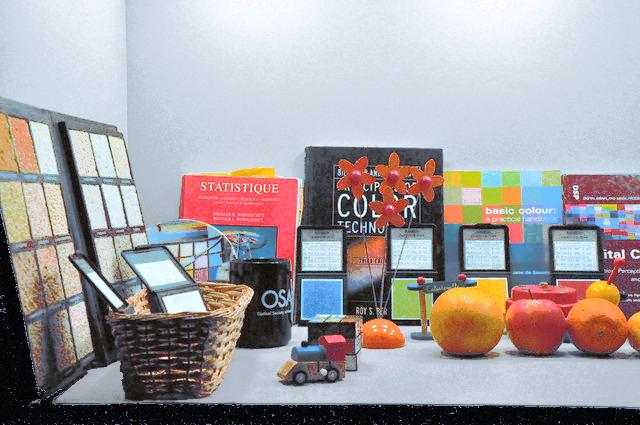}}}
\hspace{1.5mm}\subfloat{{\includegraphics[width=0.193\textwidth,valign=c]{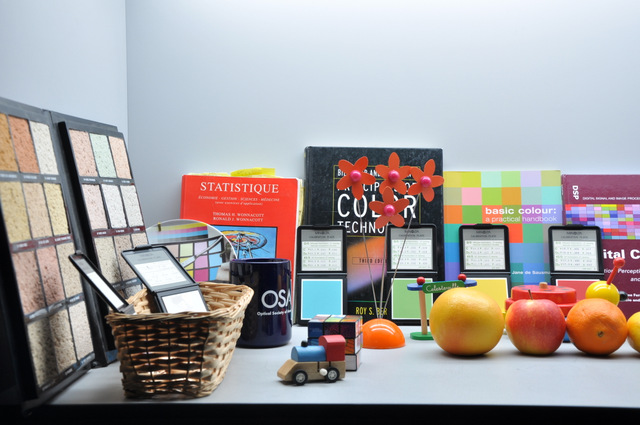}}}\vspace{-2mm}

\hspace{-0.4mm}\subfloat[\centering Input]{{\includegraphics[width=0.193\textwidth,valign=c]{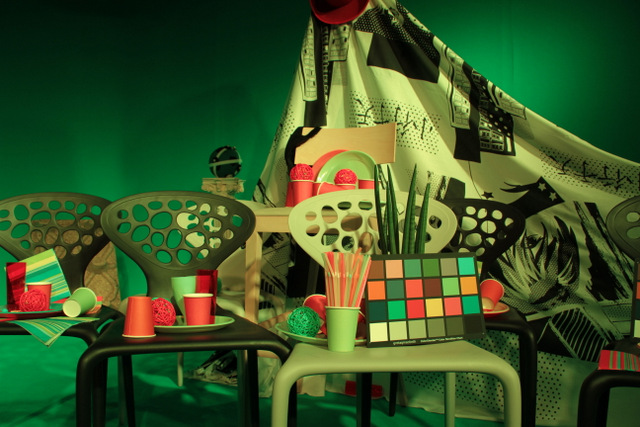}}}
\hspace{1.5mm}\subfloat[\centering Reference]{{\includegraphics[width=0.193\textwidth,valign=c]{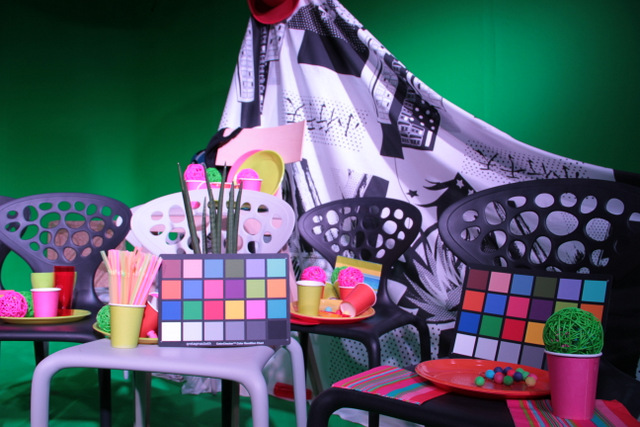}}}
\hspace{1.5mm}\subfloat[\centering Redistributor]{{\includegraphics[width=0.193\textwidth,valign=c]{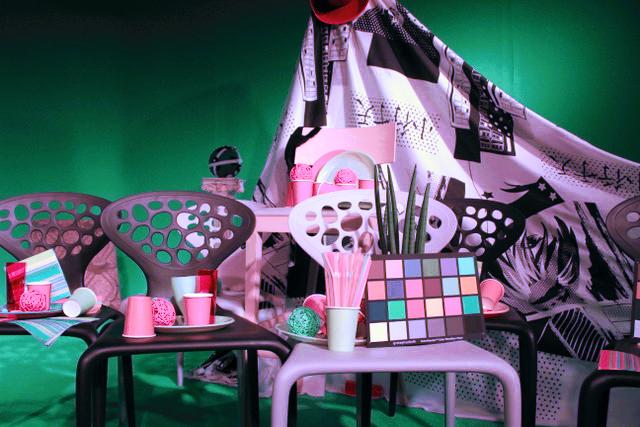}}}
\hspace{1.5mm}\subfloat[\centering Redistributor (masked)]{{\includegraphics[width=0.193\textwidth,valign=c]{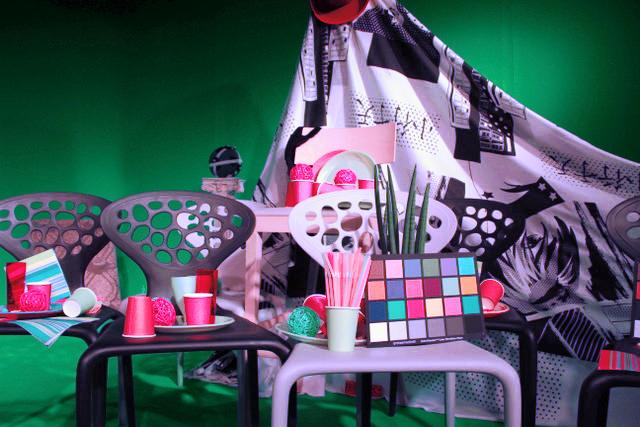}}}
\hspace{1.5mm}\subfloat[\centering Ground truth]{{\includegraphics[width=0.193\textwidth,valign=c]{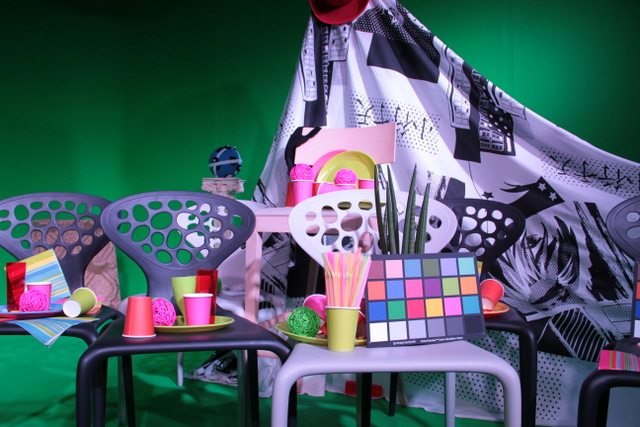}}}
 
    \caption{Performance of Redistributor on the Color Mapping Dataset. Each of the three rows contains one example image, namely 4, 39, and 113, from the three image sets. Input images shown here represent the ''middle difficulty`` from each of the sets. The whole dataset is displayed in Figure\,13 of~\cite{ReinhardSurvey}.
    }%
    \label{fig:colormapping}%
\end{figure*}

\begin{figure*}[!t]%
    \centering
    \vspace{-2.5mm}
    \subfloat{{
        \hspace{-3mm}\includegraphics[width=1.015\textwidth]{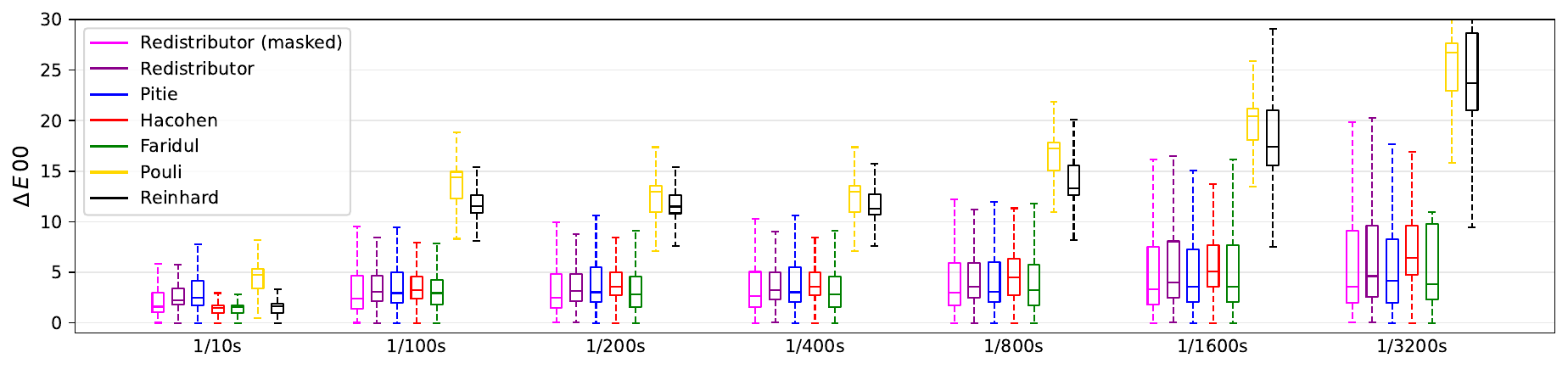} }}\vspace{-5.4mm}

    \subfloat{{
        \hspace{-3mm}\includegraphics[width=1.015\textwidth]{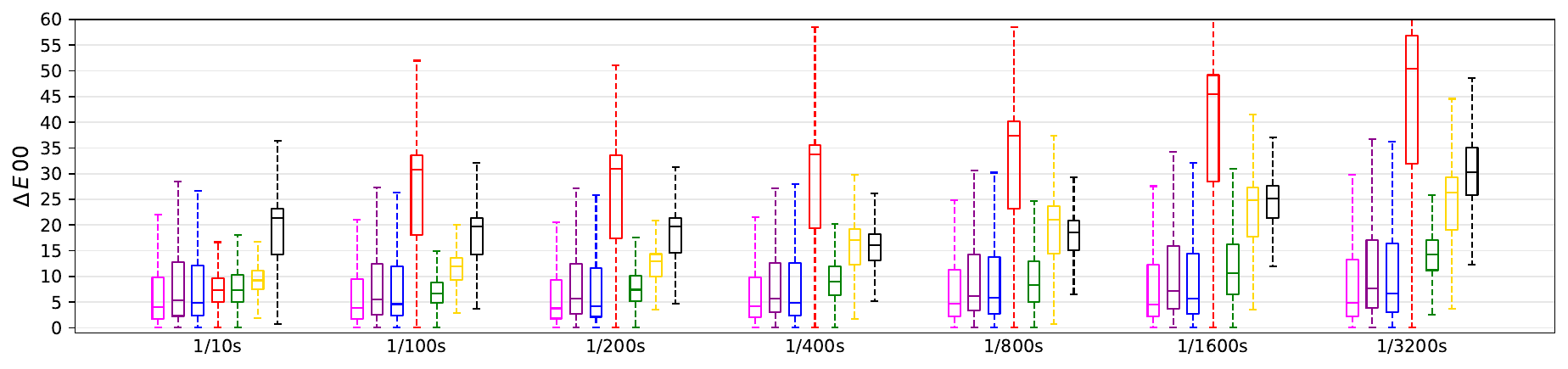} }}\vspace{-5.4mm}

    \subfloat{{
        \hspace{-3mm}\includegraphics[width=1.015\textwidth]{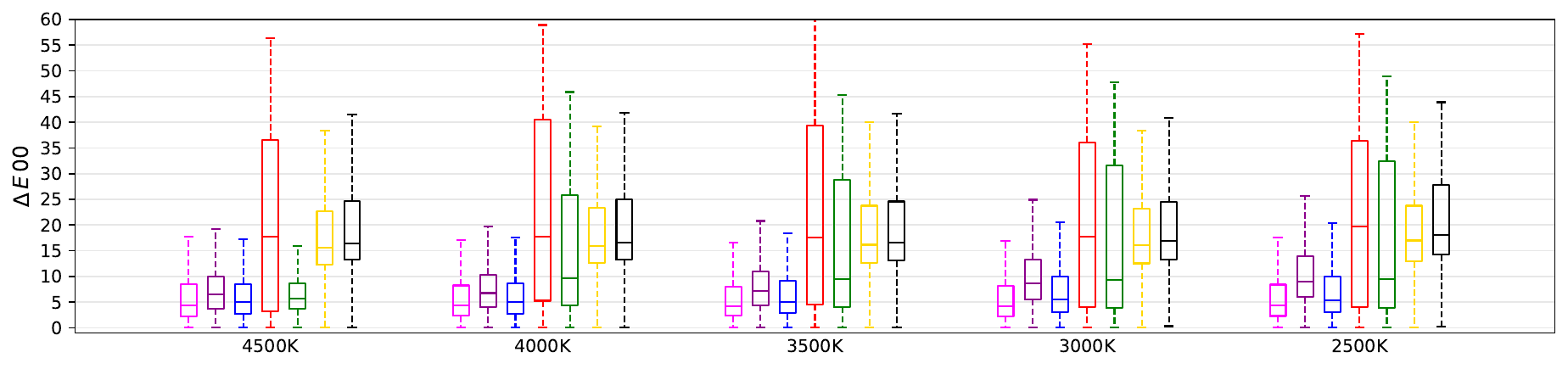} }}%
    \caption{Comparing Redistributor on the Color Mapping Dataset to Pitie~\cite{Pitie}, Hacohen~\cite{Hacohen}, Faridul~\cite{Faridul}, Pouli~\cite{Pouli}, and Reinhard~\cite{Reinhard_for_boxplot_comp}.
    The rows correspond to image sets\,1--3. The x-axis in the first two rows represents the shutter speed, in the last row, the color temperature. 
    Each box shows the median CIE $\Delta E00$, and the $25^{th}$ and $75^{th}$ percentiles. The whiskers extend from the box by $1.5\times IQR$.} %
    \label{fig:colormapping-perf}%
\end{figure*}

\subsection{Connection to Transport Based Signal Processing}
\label{subsec:transport}
\noindent
Thus far we have been interested in the~output we can obtain by applying the~Redistributor transformation to some input of interest. To round out this section we will briefly mention a~rather different usage of the~basic idea behind the~Redistributor, which has been put forward in\,\cite{CDT}. We give a~somewhat simplified and slightly paraphrased description. Let $f\colon\R\to\R$ be a~continuous function which satisfies $\|f\|_1:=\int_{\R}f(x)\mathrm{d}x=1$, i.e.\@ such that we can interpret it as a~density function, and let $F_T=\int_{-\infty}^x f(x)\mathrm{d}x$ denote the~corresponding CDF.
Moreover, we fix some reference probability measure with continuous density function $f_S$ and CDF $F_S$. 
We can now consider the~so-called Cumulative Distribution Transform $\hat{f}\colon\R\to\R$ of $f$, which can be written~as
\begin{align*}
    \hat{f}(x)=(R(x)-x)\sqrt{f_s(x)},
\end{align*}
where $R=F_T^{-1}\circ F_S$ is simply the~Redistributor transformation. In this setting, one considers data $(f_i)_{i\in\N}$ where each $f_i$ can be represented suitably as a~function. In\,\cite{CDT}, a~number of advantages of applying the~CDT to each of the~data points and working with the~set $(\hat{f}_i)_{i\in\N}$ are shown theoretically and demonstrated empirically. This has been developed into a~number of interesting methods for signals processing, see~\cite{PyTransKit}.

\begin{figure}[tp]
    \captionsetup[subfigure]{labelformat=empty}
    \centering
    \vspace{-3mm}
\subfloat{{\includegraphics[width=0.10\textwidth,valign=c]{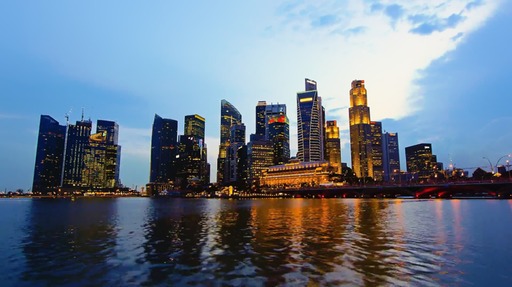}}}
\hfill\subfloat{{\includegraphics[width=0.10\textwidth,valign=c]{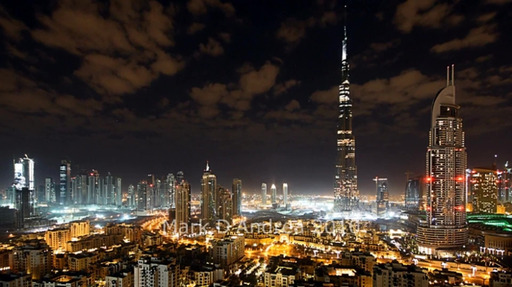}}}
\hfill\subfloat{{\includegraphics[width=0.10\textwidth,valign=c]{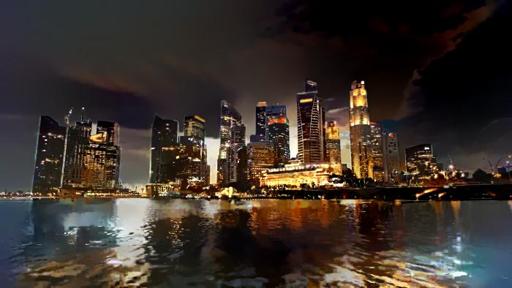}}}
\hfill\subfloat{{\includegraphics[width=0.10\textwidth,valign=c]{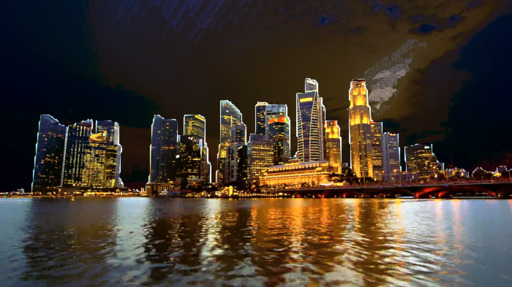}}}\vspace{-2.5mm}

\subfloat{{\includegraphics[width=0.10\textwidth,valign=c]{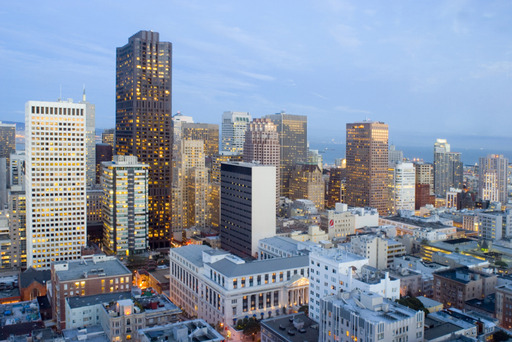}}}
\hfill\subfloat{{\includegraphics[width=0.10\textwidth,valign=c]{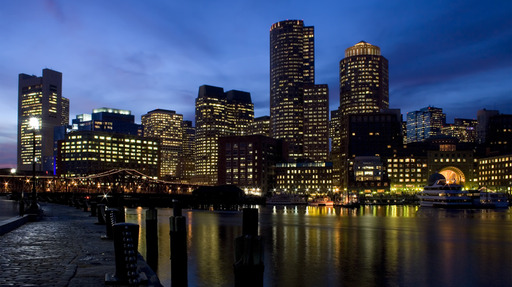}}}
\hfill\subfloat{{\includegraphics[width=0.10\textwidth,valign=c]{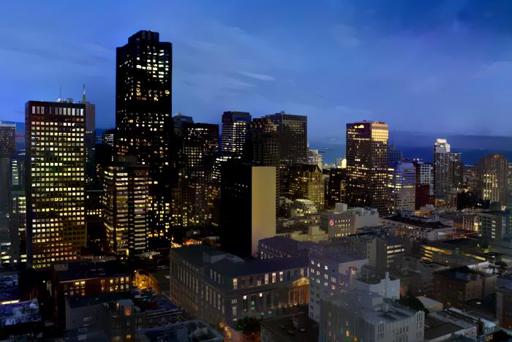}}}
\hfill\subfloat{{\includegraphics[width=0.10\textwidth,valign=c]{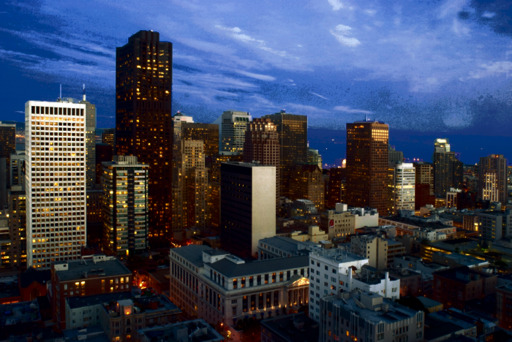}}}\vspace{-2.5mm}

\subfloat{{\includegraphics[width=0.10\textwidth,valign=c]{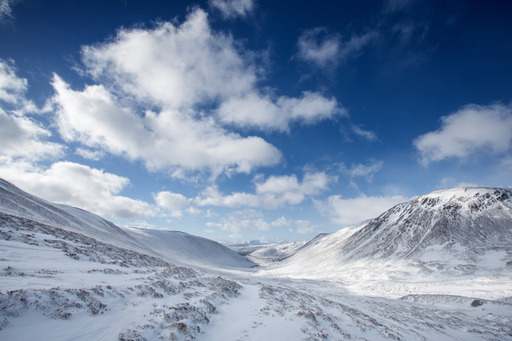}}}
\hfill\subfloat{{\includegraphics[width=0.10\textwidth,valign=c]{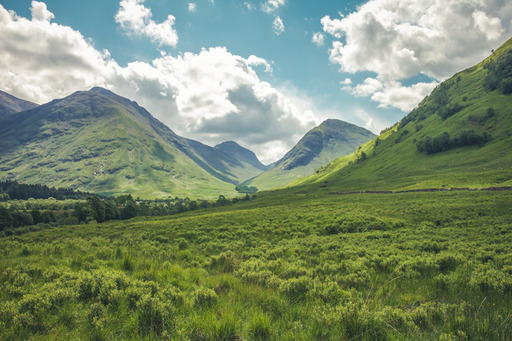}}}
\hfill\subfloat{{\includegraphics[width=0.10\textwidth,valign=c]{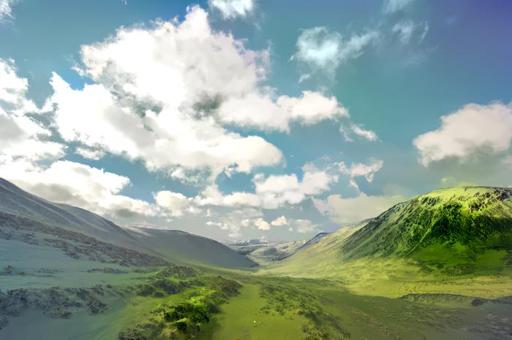}}}
\hfill\subfloat{{\includegraphics[width=0.10\textwidth,valign=c]{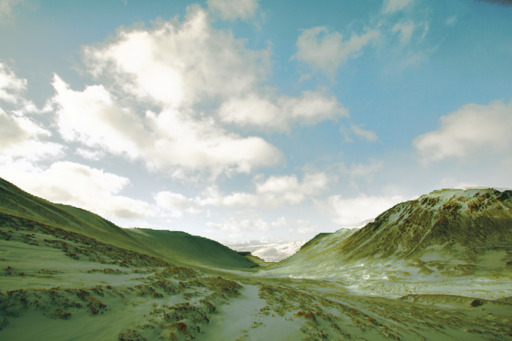}}}\vspace{-2.5mm}

\subfloat{{\includegraphics[width=0.10\textwidth,valign=c]{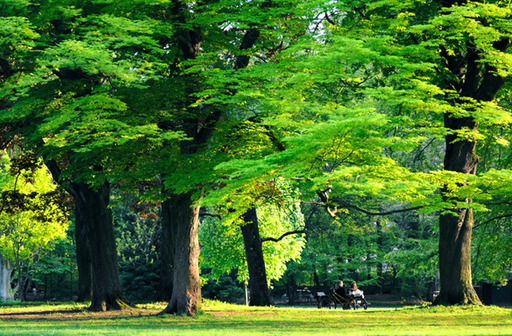}}}
\hfill\subfloat{{\includegraphics[width=0.10\textwidth,valign=c]{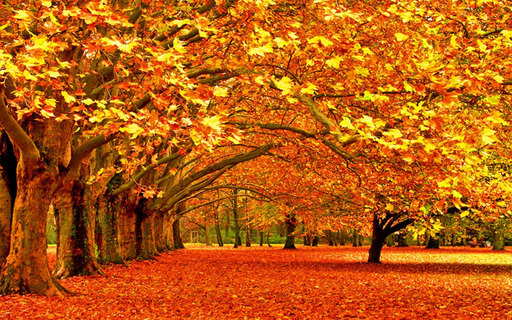}}}
\hfill\subfloat{{\includegraphics[width=0.10\textwidth,valign=c]{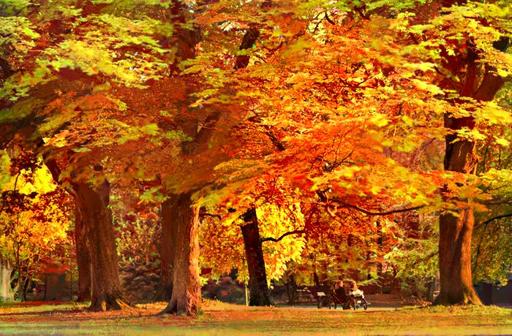}}}
\hfill\subfloat{{\includegraphics[width=0.10\textwidth,valign=c]{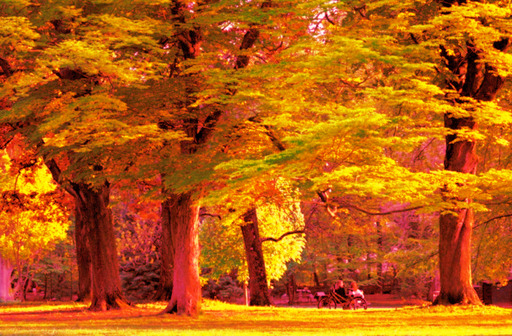}}}\vspace{-2.5mm}

\subfloat{{\includegraphics[width=0.10\textwidth,valign=c]{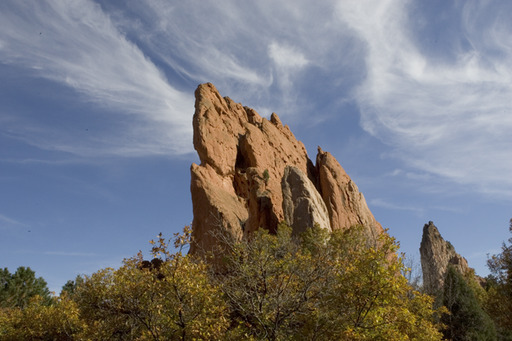}}}
\hfill\subfloat{{\includegraphics[width=0.10\textwidth,valign=c]{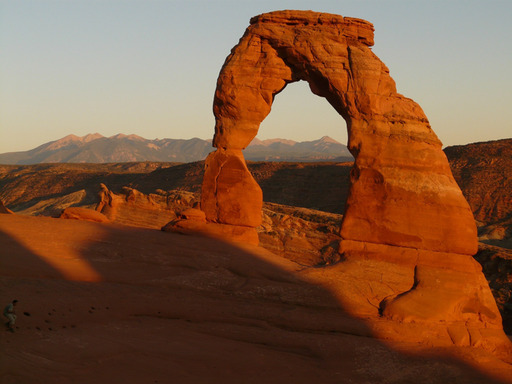}}}
\hfill\subfloat{{\includegraphics[width=0.10\textwidth,valign=c]{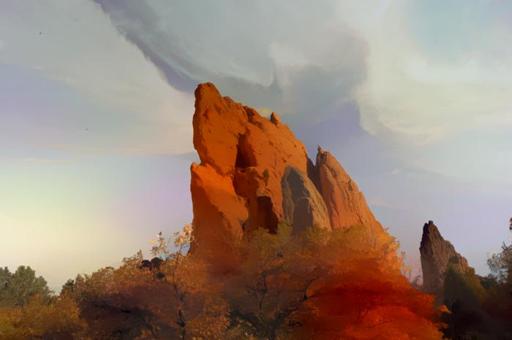}}}
\hfill\subfloat{{\includegraphics[width=0.10\textwidth,valign=c]{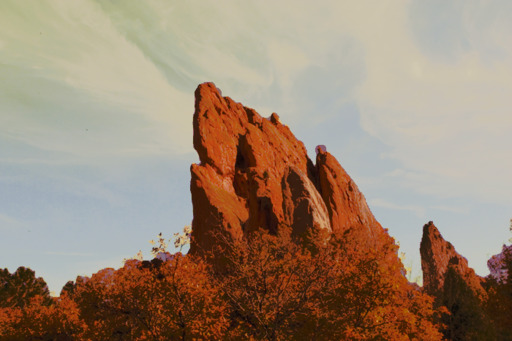}}}\vspace{-2.5mm}

\subfloat{{\includegraphics[width=0.10\textwidth,valign=c]{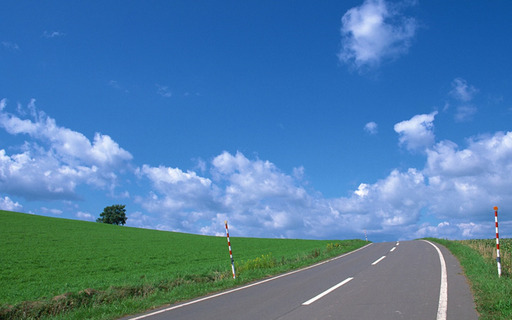}}}
\hfill\subfloat{{\includegraphics[width=0.10\textwidth,valign=c]{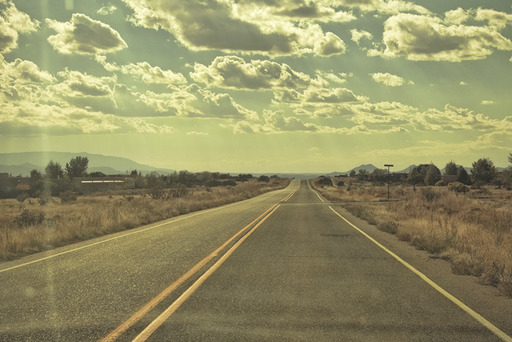}}}
\hfill\subfloat{{\includegraphics[width=0.10\textwidth,valign=c]{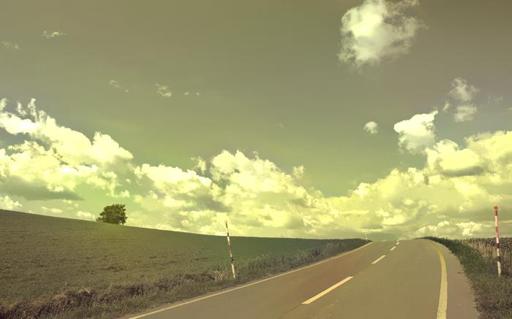}}}
\hfill\subfloat{{\includegraphics[width=0.10\textwidth,valign=c]{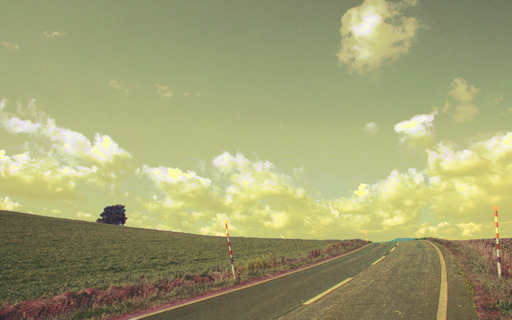}}}\vspace{-2.5mm}

\subfloat{{\includegraphics[width=0.10\textwidth,valign=c]{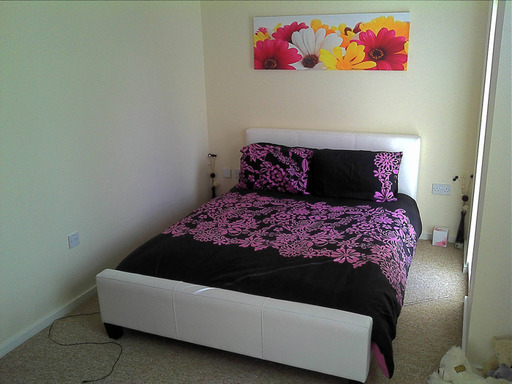}}}
\hfill\subfloat{{\includegraphics[width=0.10\textwidth,valign=c]{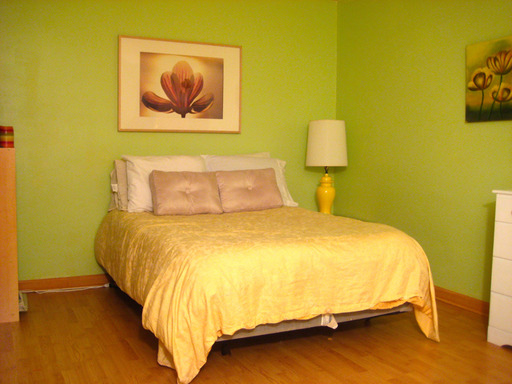}}}
\hfill\subfloat{{\includegraphics[width=0.10\textwidth,valign=c]{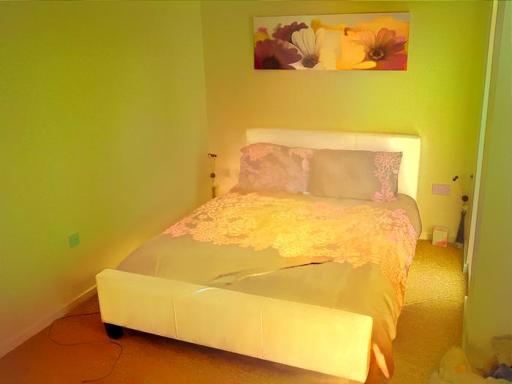}}}
\hfill\subfloat{{\includegraphics[width=0.10\textwidth,valign=c]{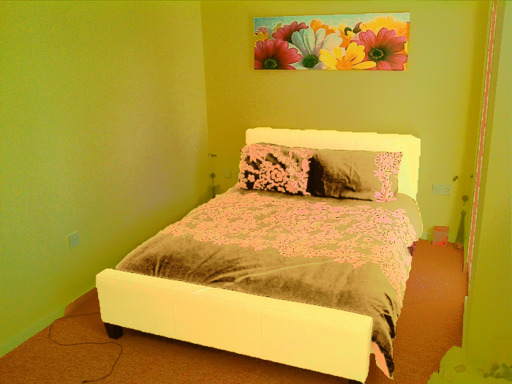}}}\vspace{-2.5mm}

\subfloat{{\includegraphics[width=0.10\textwidth,valign=c]{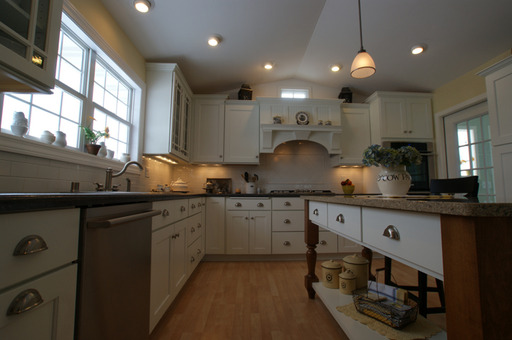}}}
\hfill\subfloat{{\includegraphics[width=0.10\textwidth,valign=c]{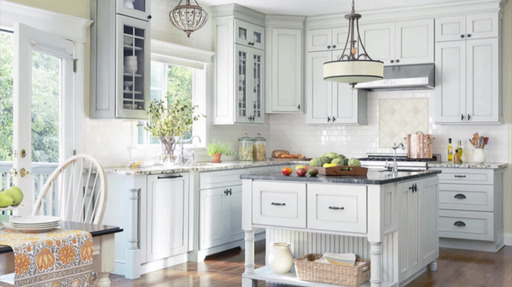}}}
\hfill\subfloat{{\includegraphics[width=0.10\textwidth,valign=c]{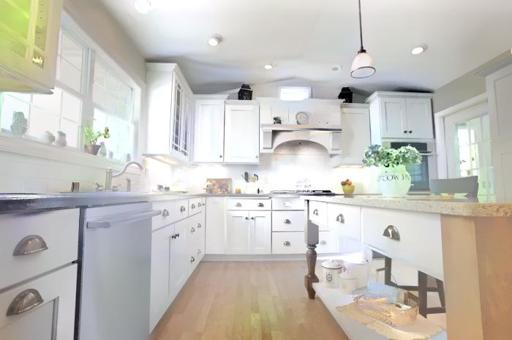}}}
\hfill\subfloat{{\includegraphics[width=0.10\textwidth,valign=c]{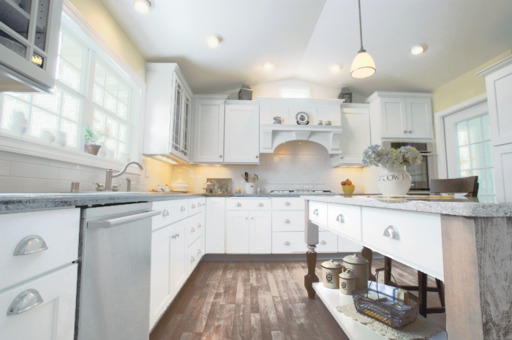}}}\vspace{-2.5mm}

\subfloat{{\includegraphics[width=0.10\textwidth,valign=c]{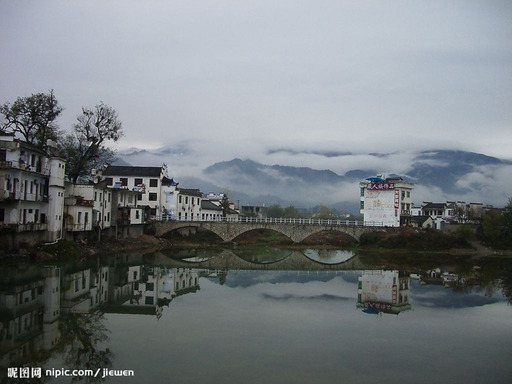}}}
\hfill\subfloat{{\includegraphics[width=0.10\textwidth,valign=c]{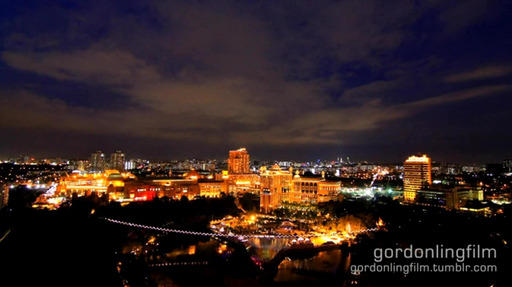}}}
\hfill\subfloat{{\includegraphics[width=0.10\textwidth,valign=c]{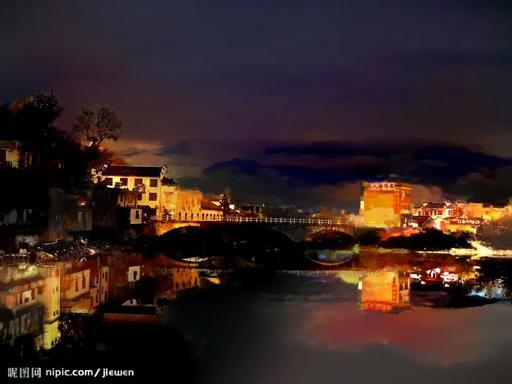}}}
\hfill\subfloat{{\includegraphics[width=0.10\textwidth,valign=c]{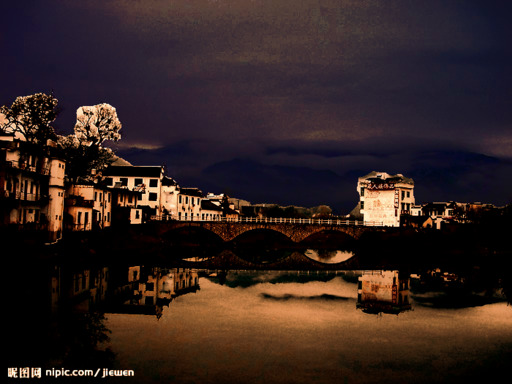}}}\vspace{-2.5mm}

\subfloat{{\includegraphics[width=0.10\textwidth,valign=c]{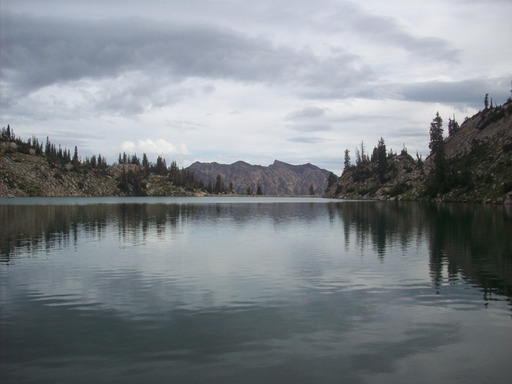}}}
\hfill\subfloat{{\includegraphics[width=0.10\textwidth,valign=c]{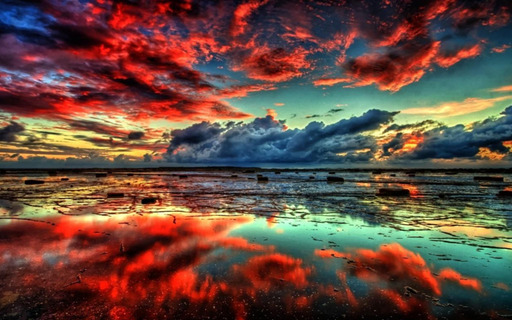}}}
\hfill\subfloat{{\includegraphics[width=0.10\textwidth,valign=c]{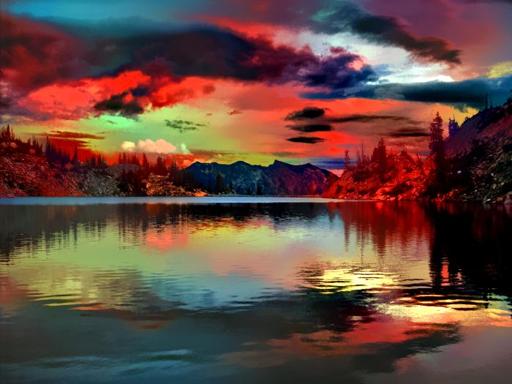}}}
\hfill\subfloat{{\includegraphics[width=0.10\textwidth,valign=c]{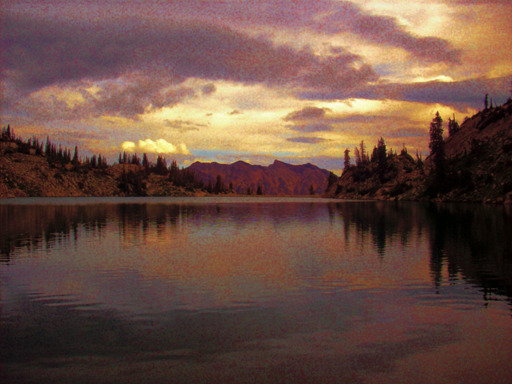}}}\vspace{-2.5mm}

\subfloat{{\includegraphics[width=0.10\textwidth,valign=c]{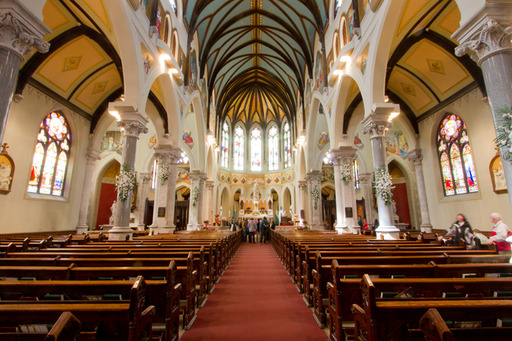}}}
\hfill\subfloat{{\includegraphics[width=0.10\textwidth,valign=c]{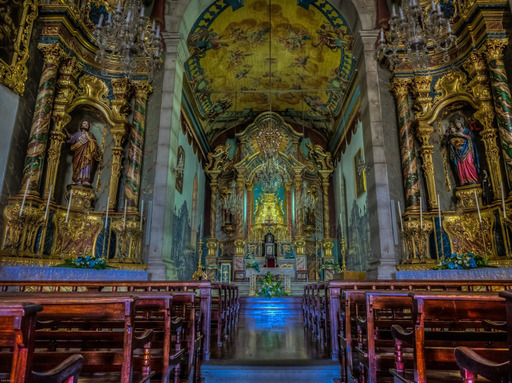}}}
\hfill\subfloat{{\includegraphics[width=0.10\textwidth,valign=c]{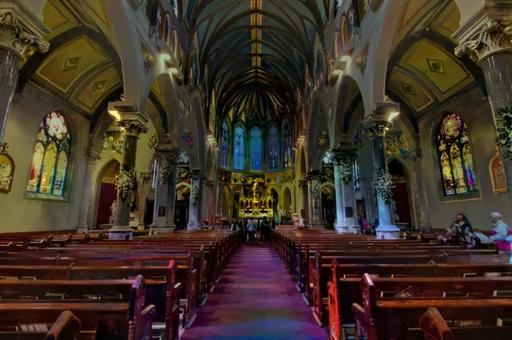}}}
\hfill\subfloat{{\includegraphics[width=0.10\textwidth,valign=c]{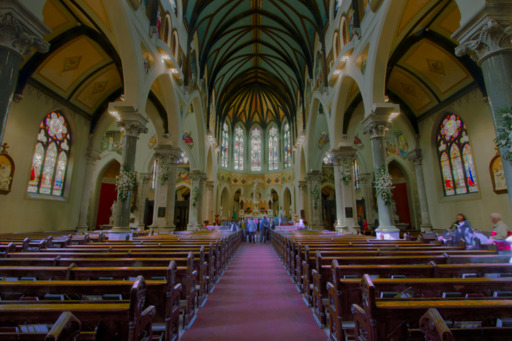}}}\vspace{-2.5mm}

\subfloat{{\includegraphics[width=0.10\textwidth,valign=c]{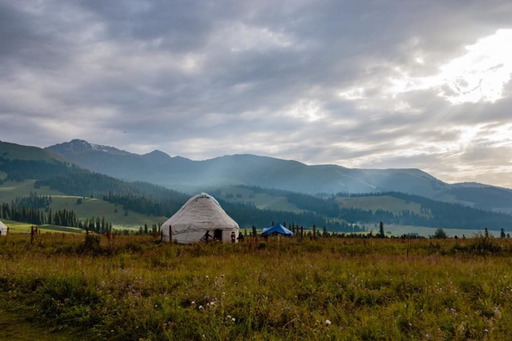}}}
\hfill\subfloat{{\includegraphics[width=0.10\textwidth,valign=c]{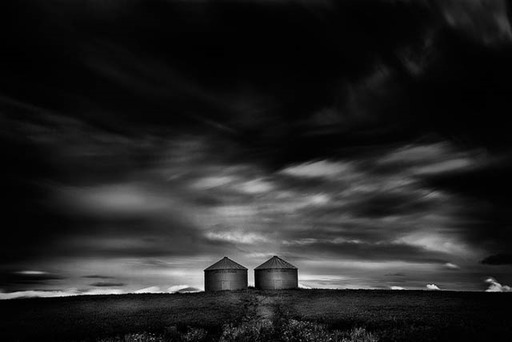}}}
\hfill\subfloat{{\includegraphics[width=0.10\textwidth,valign=c]{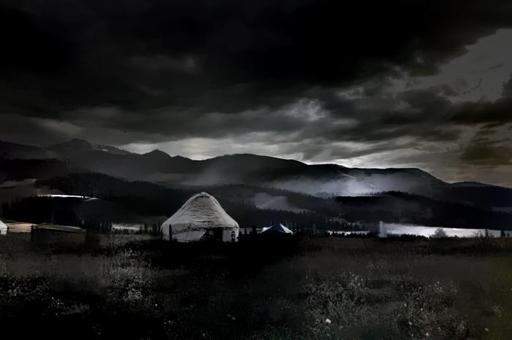}}}
\hfill\subfloat{{\includegraphics[width=0.10\textwidth,valign=c]{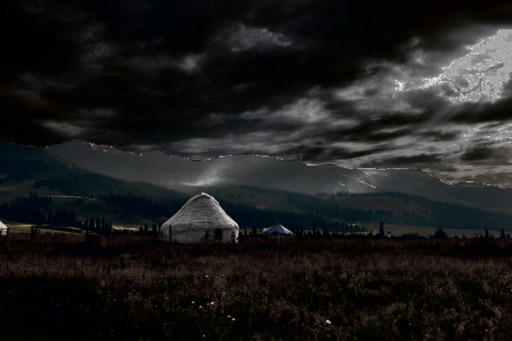}}}\vspace{-2.5mm}

\subfloat{{\includegraphics[width=0.10\textwidth,valign=c]{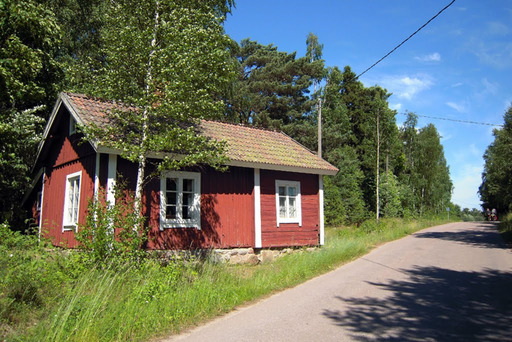}}}
\hfill\subfloat{{\includegraphics[width=0.10\textwidth,valign=c]{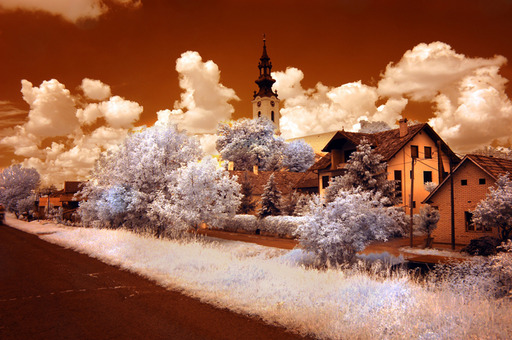}}}
\hfill\subfloat{{\includegraphics[width=0.10\textwidth,valign=c]{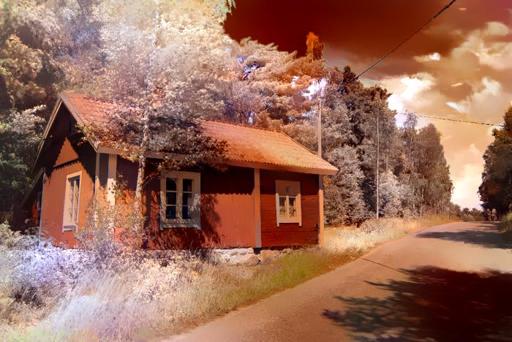}}}
\hfill\subfloat{{\includegraphics[width=0.10\textwidth,valign=c]{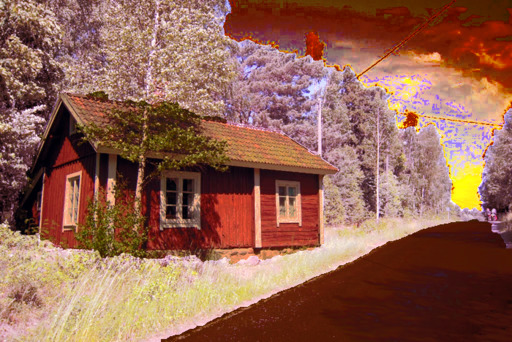}}}\vspace{-2.5mm}

\subfloat{{\includegraphics[width=0.10\textwidth,valign=c]{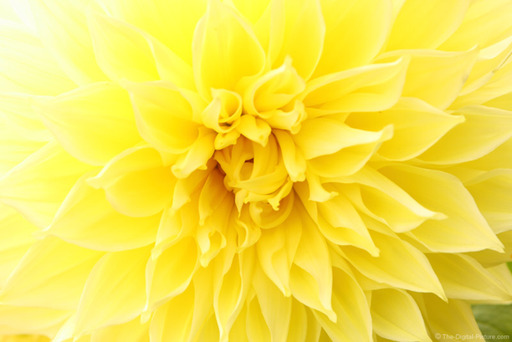}}}
\hfill\subfloat{{\includegraphics[width=0.10\textwidth,valign=c]{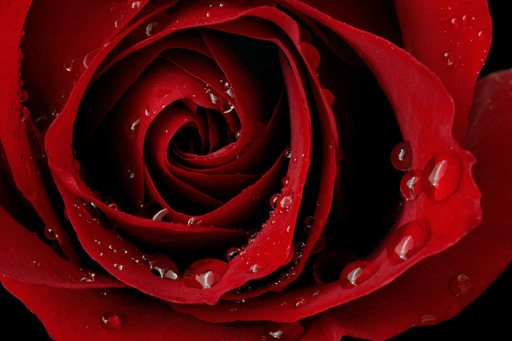}}}
\hfill\subfloat{{\includegraphics[width=0.10\textwidth,valign=c]{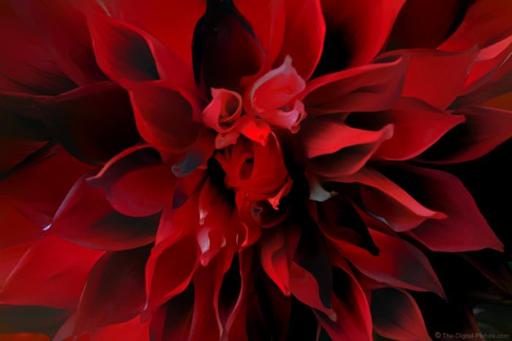}}}
\hfill\subfloat{{\includegraphics[width=0.10\textwidth,valign=c]{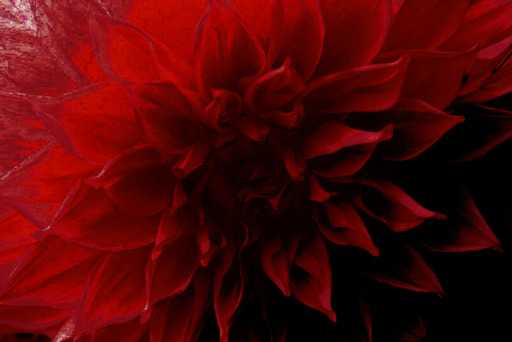}}}\vspace{-2.5mm}

\subfloat{{\includegraphics[width=0.10\textwidth,valign=c]{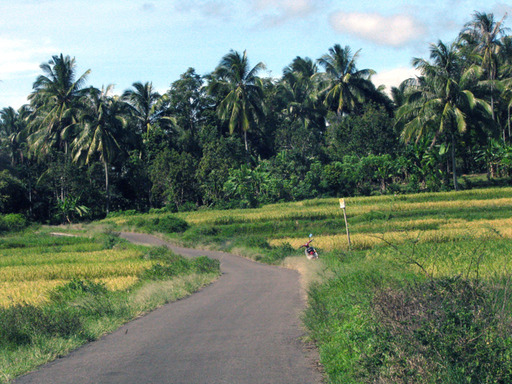}}}
\hfill\subfloat{{\includegraphics[width=0.10\textwidth,valign=c]{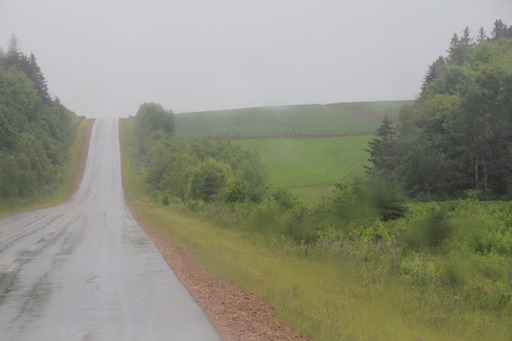}}}
\hfill\subfloat{{\includegraphics[width=0.10\textwidth,valign=c]{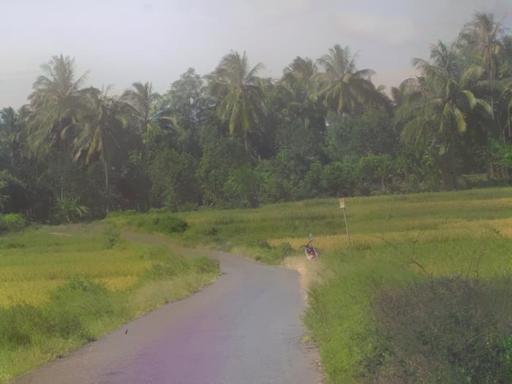}}}
\hfill\subfloat{{\includegraphics[width=0.10\textwidth,valign=c]{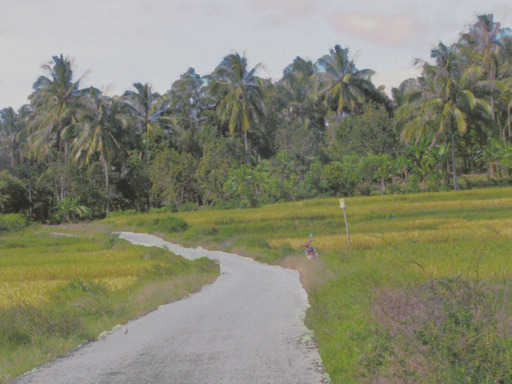}}}\vspace{-2.5mm}

\subfloat[\centering Input]{{\includegraphics[width=0.10\textwidth,valign=c]{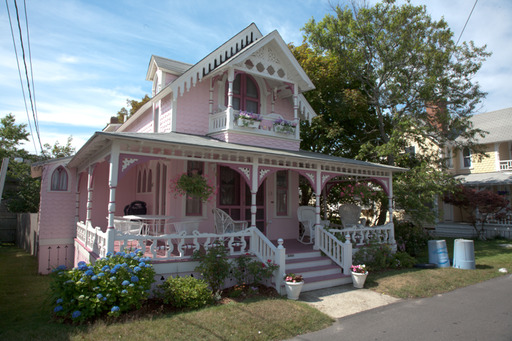}}}
\hfill\subfloat[\centering Style]{{\includegraphics[width=0.10\textwidth,valign=c]{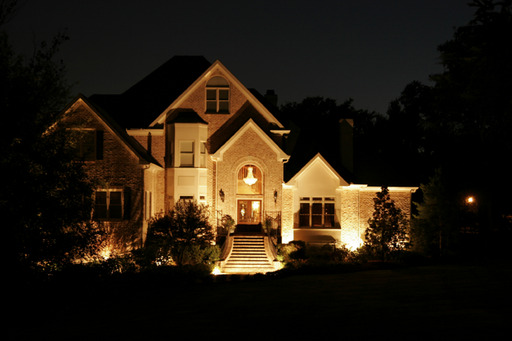}}}
\hfill\subfloat[\centering Luan]{{\includegraphics[width=0.10\textwidth,valign=c]{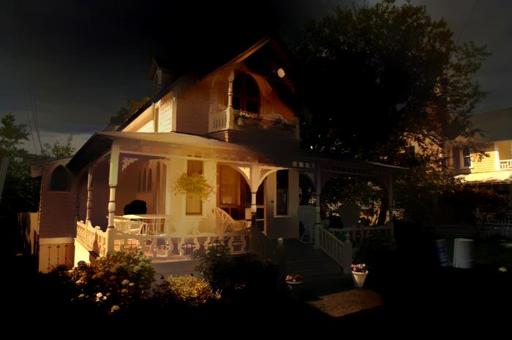}}}
\hfill\subfloat[\centering Redistributor]{{\includegraphics[width=0.10\textwidth,valign=c]{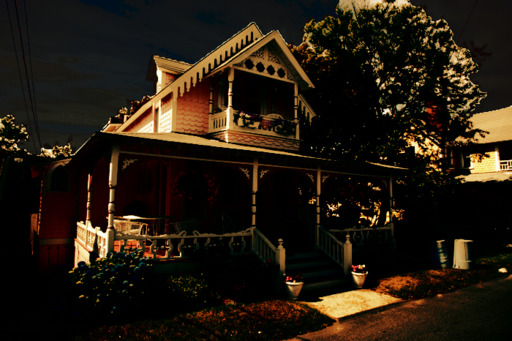}}} \vspace{0.5mm}

    \caption{Comparison of Redistributor with Deep Photo Style Transfer~\cite{DPST}. While the NST method of Luan can capture drastic style changes, Redistributor keeps the fidelity of the input, creating images with new color scheme but realistic content by retaining all original details. 
    }%
    \label{fig:nst}%
\end{figure}

\section{Comparisons}
\label{sec:comparisons}
\noindent
We present two comparisons of Redistributor to existing methods.
First, we consider a~color correction task where we compare it to model-based methods using a~dataset and evaluation criteria from \cite{ReinhardSurvey}. Second, we compare it to results obtained by a~deep-learning approach to photorealistic style transfer \cite{DPST}, where they provide and make use of masks for the input and reference images, in order to match semantically corresponding areas of the respective images (e.g.\@ sky, buildings, water). Due to its flexible implementation, Redistributor can conveniently take advantage of these masks as well.  

\subsection{Color Correction Comparison}
\label{subsec:comparison-correction}

\noindent
For this comparison we use a~dataset\footnote{Available at \url{https://sites.google.com/site/reviewcolormapping/}.} for ground-truth-based evaluation, first presented in \cite{ReinhardSurvey}. The idea is to capture an image of a~scene with the desired illumination in order to create a~ground truth, as well as a~series of test images of the scene from the same viewpoint but with changes in imaging modalities, like shutter speed, illuminant, or white balance. The task is to transform these test images into something close to the ground truth by using a~reference image of the same scene with the correct illumination but from a~different viewpoint. The dataset is organized into three sets of images corresponding to different changes in imaging modalities. In the first one the shutter speed is varied from $1/10s$ to $1/3200s$, in the second one both shutter speed and illuminant is varied, in the third one the color temperature of the camera is varied from $4500K$ to $2500K$. In \Cref{fig:colormapping} each row shows an example from one the three sets.
The result of the transformation is then compared to the ground truth by taking the color difference in CIE $\Delta E00$ unit~\cite{DeltaEpaper} for each pixel, which is then summarized via boxplots in \Cref{fig:colormapping-perf}. When employing simple masks Redistributor outperforms the other methods across the board and still produces competitive results without masking.

\subsection{Photorealistic Style Transfer Comparison}
\label{subsec:comparison-nst}

\noindent In this section we present a~comparison to a~deep learning approach~\cite{DPST}, which aims for photorealistic style transfer, by attempting to suppress the distortion of the image content which is often produced by NST methods. By contrast, our method, as well as the other model-based methods considered in the previous subsection, guarantee that the structure of the image is kept, but may struggle to convincingly match the style of the reference image. We demonstrate that Redistributor is capable of yielding visually satisfying results on the images considered in~\cite{DPST}. 
In particular it is notably better at maintaining fine structures of the images, e.g.\@ the tree in the image in row $5$ and flower petals in the image in row $16$. As demonstrated in rows $8$ and $11$ our method avoids introducing unnatural green resp.\@ blue spot illuminations. Redistributor does however fall short of fully matching the reference style in row $10$, where there is both a~large difference in colors and content, i.e.\@ image to be transformed contains mountains and trees, but the reference image does not.

\section{Conclusion}
\label{sec:conclusion}
\noindent
We introduced Redistributor, a~method and Python package for estimating and transforming empirical data distributions, and provided a~range of applications in image and signal processing as well as machine learning.

Subsequently, we demonstrated its utility in various applications. In the context of image processing, it offers an powerful approach to color correction and photorealistic style transfer. It effectively changes an image's color scheme using a~reference image, yielding visually pleasing results while preserving the original image's content. Unlike neural style transfer methods, Redistributor guarantees content fidelity and offers greater flexibility, explainability, and computational efficiency.

We also showed an application in image data augmentation, enabling the creation of variations of images for the purpose of expanding datasets for neural network training. Moreover, due to its implementation, Redistributor can be integrated as a~preprocessing step into machine learning pipelines.

Redistributor outperformed other model-based methods in color correction tasks, demonstrating its effectiveness in a~ground-truth-based evaluation. Additionally, we showcased its ability to achieve photorealistic style transfer with superior content preservation compared to deep learning methods.

In conclusion, Redistributor is a~versatile and efficient tool for transforming data distributions, offering a~wide range of applications in image processing, signal processing, machine learning, and beyond. Its strong theoretical foundation, along with its proven performance in various use cases, makes it a~valuable addition to the toolbox of researchers and practitioners in these fields.

\section*{Acknowledgement}
\noindent Photos in \Cref{fig:badlight} by \href{https://stanci.com/}{S.\@ Markovičová}. Images in \Cref{fig:teaser} and \Cref{fig:matchcolor} by \href{https://unsplash.com/@oulashin}{S.\@ Oulashin}), \href{https://unsplash.com/@cristofer}{C.\@ Maximilian}, \href{https://unsplash.com/@jplenio}{J.\@ Plenio}, \href{https://unsplash.com/@vincentiu}{V.\@ Solomon}, \href{https://unsplash.com/@sanderweeteling}{S.\@ Weeteling}, \href{https://unsplash.com/@cristina_gottardi}{C.\@ Gottardi}, \href{https://unsplash.com/@jamie452}{J.\@ Street}, and \href{https://unsplash.com/@thevantagepoint718}{L.\@ Pieters}. Photo in~\Cref{fig:mosaic} by \href{https://unsplash.com/@sour_moha}{S.\@ Moha}. Images in \Cref{fig:augmentation} by \href{https://unsplash.com/@ehizelesamuel1}{E.\@S.\@ Agbonyeme}, \href{https://unsplash.com/@lancereis}{L.\@ Reis}, \href{https://unsplash.com/@svsokolov}{S.\@ Sokolov}, \href{https://unsplash.com/@anamnesis33}{Andrey K}, \href{https://unsplash.com/@kate_gliz}{K.\@ Hliznitsova}, and \href{https://unsplash.com/@exploringzhongguo}{Taha}. Free image editing software \href{https://photopea.com}{Photopea.com} by \href{https://twitter.com/ivankutskir}{I.\@ Kutskir} was used to preprocess and finalize multiple of the~presented figures. We thank J.\@ Schl{\"u}ter for the initial inspiration, and J.\@ Berner for many valuable discussions.

\section*{Author Contributions}
\noindent
\textbf{HP} - conceptualization, investigation, methodology, implementation, experimentation, visualization, writing. \textbf{DE} - conceptualization, investigation, methodology, formal analysis, experimentation, writing. \textbf{MD} - formal analysis, writing, supervision. \textbf{KDJ} - investigation, experimentation, writing.

\section*{Funding}
\noindent
This work has been funded in part by Research Network Data Science\,@\,Uni Vienna and in part by ETH~Foundations of Data Science (ETH-FDS).

\section*{Data Availability}
\noindent
The datasets utilized in this study are available at \href{https://sites.google.com/site/reviewcolormapping/}{https://sites.google.com/site/reviewcolormapping}, \href{https://github.com/luanfujun/deep-photo-styletransfer}{https://github.com/luanfujun/deep-photo-styletransfer}, and \href{http://vis-www.cs.umass.edu/lfw/}{http://vis-www.cs.umass.edu/lfw/}. 

\section*{Code Availability}
\noindent
Source code for the Redistributor package as well as the code for reproducing the figures used in this study are available at \href{https://github.com/paloha/redistributor}{https://github.com/paloha/redistributor}. 

\section*{Conflict of Interest}
\noindent
The authors have no competing interests to declare that are relevant to the content of this manuscript.


\bibliographystyle{apalike}
\bibliography{references}

\end{document}